\documentclass[10pt,twoside]{article}

\usepackage[T1]{fontenc}
\usepackage[utf8]{inputenc}
\usepackage[british]{babel}
\usepackage{lmodern}
\usepackage{microtype}
\usepackage[a4paper,margin=27mm,headheight=14pt]{geometry}
\usepackage{amsmath,amssymb,amsfonts,amsthm,mathtools,bm,mathrsfs}
\usepackage{booktabs,array,tabularx,graphicx,multirow}
\usepackage{float}
\usepackage{tikz}
\usepackage{pgfplots}
\pgfplotsset{compat=1.18}
\usetikzlibrary{arrows.meta,positioning,fit,calc}
\usepackage[shortlabels]{enumitem}
\usepackage{xcolor}
\usepackage{natbib}
\usepackage{xurl}
\usepackage{hyperref}
\usepackage[nameinlink,capitalize,noabbrev]{cleveref}
\usepackage{fancyhdr}

\definecolor{actarefblue}{RGB}{0,0,180}
\hypersetup{
  colorlinks=true,
  linkcolor=actarefblue,
  citecolor=actarefblue,
  urlcolor=actarefblue,
  filecolor=actarefblue,
  pdftitle={Common Covariance Geometry and Certification for Brownian Kernel Ladders},
  pdfauthor={Mahdi Mohammadigohari}
}

\theoremstyle{plain}
\newtheorem{theorem}{Theorem}[section]
\newtheorem{proposition}[theorem]{Proposition}
\newtheorem{lemma}[theorem]{Lemma}
\newtheorem{corollary}[theorem]{Corollary}
\theoremstyle{definition}
\newtheorem{definition}[theorem]{Definition}

\newtheorem{examplex}[theorem]{Example}
\theoremstyle{remark}
\newtheorem{remark}[theorem]{Remark}

\makeatletter
\@ifundefined{theHtheorem}
  {}
  {}
\@ifundefined{theHproposition}{}{}
\@ifundefined{theHlemma}{}{}
\@ifundefined{theHcorollary}{}{}
\@ifundefined{theHdefinition}{}{}
\@ifundefined{theHassumption}{}{}
\@ifundefined{theHexamplex}{}{}
\@ifundefined{theHremark}{}{}
\makeatother

\crefname{theorem}{theorem}{theorems}
\Crefname{theorem}{Theorem}{Theorems}
\crefname{proposition}{proposition}{propositions}
\Crefname{proposition}{Proposition}{Propositions}
\crefname{lemma}{lemma}{lemmas}
\Crefname{lemma}{Lemma}{Lemmas}
\crefname{corollary}{corollary}{corollaries}
\Crefname{corollary}{Corollary}{Corollaries}
\crefname{definition}{definition}{definitions}
\Crefname{definition}{Definition}{Definitions}
\crefname{assumption}{assumption}{assumptions}
\Crefname{assumption}{Assumption}{Assumptions}
\crefname{remark}{remark}{remarks}
\Crefname{remark}{Remark}{Remarks}
\crefname{examplex}{example}{examples}
\Crefname{examplex}{Example}{Examples}
\crefname{appendix}{appendix}{appendices}
\Crefname{appendix}{Appendix}{Appendices}

\newenvironment{keywords}
  {\par\smallskip\noindent\textbf{Keywords: }\itshape}
  {\par\medskip}

\newcommand{\R}{\mathbb{R}}
\newcommand{\N}{\mathbb{N}}
\newcommand{\E}{\mathbb{E}}
\newcommand{\Pp}{\mathbb{P}}
\newcommand{\cX}{\mathcal{X}}
\newcommand{\cH}{\mathcal{H}}

\newcommand{\cA}{\mathcal{A}}

\newcommand{\cC}{\mathcal{C}}

\newcommand{\sK}{\mathscr{K}}
\newcommand{\sS}{\mathscr{S}}

\newcommand{\bx}{\bm{x}}
\newcommand{\by}{\bm{y}}
\newcommand{\be}{\bm{e}}

\newcommand{\bv}{\bm{v}}
\newcommand{\bone}{\bm{1}}

\newcommand{\diag}{\operatorname{diag}}
\newcommand{\tr}{\operatorname{tr}}
\newcommand{\ran}{\operatorname{ran}}

\newcommand{\cl}{\operatorname{cl}}
\newcommand{\absconv}{\operatorname{absconv}}

\newcommand{\op}{\operatorname{op}}

\newcommand{\argmax}{\operatorname*{arg\,max}}
\newcommand{\ip}[2]{\left\langle #1,#2\right\rangle}
\newcommand{\norm}[1]{\left\lVert #1\right\rVert}
\newcommand{\abs}[1]{\ensuremath{\left\lvert #1\right\rvert}}
\newcommand{\set}[1]{\left\{#1\right\}}

\newcommand{\psd}{\mathbb{S}_{+}}
\newcommand{\kB}{k^{(\mathrm B)}}
\newcommand{\HB}{H_{\mathrm B}}
\newcommand{\Chat}{\widehat{\mathcal C}}
\newcommand{\Ghat}{\widehat{\mathfrak G}}
\newcommand{\Rhat}{\widehat{\mathfrak R}}
\newcommand{\eps}{\varepsilon}

\allowdisplaybreaks
\setlist[itemize]{leftmargin=*,topsep=3pt,itemsep=2pt}
\setlist[enumerate]{leftmargin=*,topsep=3pt,itemsep=2pt}
\newcommand{\DevMeanConvGain}{+0.191}

\newcommand{\ConfMeanConvGain}{+0.163}

\newcommand{\ConfPairedWins}{7}

\title{Common Covariance Geometry and Certification for Brownian Kernel Ladders}

\author{Mahdi Mohammadigohari\\
\small Faculty of Engineering, Free University of Bozen--Bolzano\\
\small Via Bruno Buozzi 1, 39100 Bolzano, Italy\\
\small \texttt{Mahdi.Mohammadigohari@gmail.com}}
\date{}

\begin{document}
\maketitle

\begin{abstract}
A representation-adaptive kernel class produces, on a fixed sample, a union of reproducing-kernel Hilbert-space ellipsoids rather than one ellipsoid. We introduce the minimum-trace common covariance that dominates the unrestricted empirical union generated by Brownian kernel ladders and develop its statistical, approximation-theoretic, and computational consequences. The covariance value admits exact formulations through absolutely two-summing operators and covariance-dominated multipliers, and it yields a universal Gaussian-complexity bound. A closed last-layer Dirac-trace reduction and a signed Brownian threshold representation convert the generic covariance problem into threshold, graph-coarea, and effective-resistance geometry. These tools give deterministic depth laws, conditional Gaussian reverses, random-design and perturbation transfers, and an exact empirical Kolmogorov-width formula whose leading covariance eigenspaces approximate the complete adaptive ball simultaneously. Finite contact, active semidefinite programs, verified separation, and a convex resistance-design relaxation provide complementary lower and upper certificates. A finite covariance-indexed Brownian path on frozen representations illustrates the distinction between successful covariance certification and predictive selection: all reported path certificates succeed, whereas the locked predictive study misses one predeclared aggregate criterion. The paper thereby identifies one finite-dimensional covariance object linking unrestricted kernel adaptation, Gaussian geometry, common subspaces, and certifiable computation.
\end{abstract}



\tableofcontents
\medskip

\section{Introduction and Main Results}\label{sec:introduction}

\subsection{The unrestricted adaptive-kernel problem}

A fixed reproducing kernel Hilbert space (RKHS) gives one Gram matrix and one ellipsoid of attainable evaluations on a finite sample.  A learned metric, kernel mixture, recursive feature map, or representation-dependent kernel is different: the covariance may be selected together with the predictor.  The corresponding empirical class is therefore a union of ellipsoids.  The main problem of this paper is to replace that union by one positive semidefinite covariance in a way that remains sharp enough for statistical analysis, simultaneous approximation, and numerical certification.

Brownian kernel ladders (BKLs) provide a structured instance of this problem.  The foundational construction recursively averages Brownian pullback kernels to produce a depth-indexed hierarchy of integral RKHSs \cite{mohammadigohari2026bkl}; integration of reproducing kernels is the immediate functional-analytic precursor \cite{hotz2012representation}.  Related Brownian projection and recursive Hilbert-space constructions include learned shallow Brownian features, neural Hilbert ladders, and reproducing-kernel chains \cite{follain2025enhanced,chen2024neural,heeringa2025chains}.  The present paper studies the complementary unrestricted empirical-union problem: the canonical ladder representation may vary inside the Gaussian supremum.  This differs from the controlled-family statistical results in the foundational paper, where one realised ladder or a finite dictionary is fixed independently of the estimation sample.  Variation Brownian kernel ladders (VBKLs) provide a neighbouring path-atomic formulation in which nonlinear recursive paths are generated first and signed-measure superposition is introduced only at the outermost level \cite{mohammadigohari2026vbkl}.  That placement of superposition connects VBKLs to neural variation-space and reproducing-kernel Banach-space viewpoints \cite{bach2017convex,bartolucci2023rkbs}.  The common-covariance results below concern the canonical integral-RKHS BKL family; they are not asserted to transfer automatically to the full VBKL variation hull.

Fix a sample $\bx=(x_1,\ldots,x_n)$.  Let $\sK_L(\bx)\subseteq\psd^n$ be the closed family of depth-$L$ empirical BKL Gram matrices.  The central quantity is
\begin{equation}\label{eq:intro-tau}
\tau_L(\bx)
:=
\min\bigl\{\tr N:N\succeq K\ \text{for every }K\in\sK_L(\bx)\bigr\}.
\end{equation}
The same minimiser has three roles.  Its trace controls the Gaussian support of the full adaptive ball, its spectrum yields a target-independent empirical approximation space, and its active constraints support finite lower and upper certification schemes.  The paper develops these roles in a single theorem programme.

\subsection{Main contributions}

\paragraph{R1: empirical covariance geometry and exact factorisation.}
\Cref{thm:trace-body} identifies the empirical BKL ball with the closure of a union of RKHS trace ellipsoids.  The minimum in \eqref{eq:intro-tau} exists and separates the largest single-kernel trace from the additional inflation required for simultaneous domination; see \Cref{def:profile,prop:profile-existence}.  The exact identity in \Cref{thm:pi2} shows that $\tau_L(\bx)^{1/2}$ is the absolutely two-summing norm of the adaptive evaluation operator.  This is not merely a generic factorisation upper bound: it is the value of the robust minimum-trace covariance problem at finite resolution.

\paragraph{R2: robust multipliers and Gaussian sharpness.}
\Cref{thm:robust-minimax} gives an exact covariance-dominated multiplier formulation of $\tau_L$.  Replacing the worst centred multiplier law by a standard Gaussian law incurs the defect quantified in \Cref{thm:cotype-reverse,cor:gaussian-defect}.  A bounded defect means that the common-covariance certificate is tight up to constants; it does not mean that the unrestricted class has a fast statistical rate.  Orthogonal and high-dimensional spherical samples exhibit precisely this distinction.

\paragraph{R3: Brownian Dirac, threshold, and graph reductions.}
The BKL-specific step is the closed last-layer Dirac-trace reduction of \Cref{thm:dirac-reduction}.  It removes the final mixing measure and reduces the support function to Brownian matrices indexed by previous-layer traces.  The signed layer-cake identity in \Cref{prop:threshold-certificate} then converts these matrices into integrated positive and negative threshold vectors.  Recursive threshold profiles and the graph-coarea majorant in \Cref{thm:recursive-threshold,thm:graph-majorant} connect covariance domination to set systems, ordered and hierarchical traces, and effective resistance.

\paragraph{R4: regimes, common subspaces, and statistical consequences.}
The deterministic and random results identify several phases of unrestricted adaptation.  Orthogonal samples satisfy an explicit depth law \Cref{thm:orthogonal}; well-conditioned and spherical designs have bounded Gaussian defect but retain the slow sign-rich full-class rate \Cref{thm:well-conditioned-richness,cor:spherical-random-richness}; repeated, categorical, hierarchical, ordered, and perturbative models admit sharper conclusions under explicit occupancy, threshold, conditioning, or margin hypotheses.  Independently, \Cref{thm:kolmogorov-width} gives an exact empirical Kolmogorov-width formula.  The leading eigenspaces of any common covariance provide one sample-dependent subspace that approximates every function in the adaptive ball, and the hierarchical results yield matching spectral compression mechanisms.

\paragraph{R5: finite certification and a computational illustration.}
\Cref{thm:sparse-contact} proves that finitely many Brownian contact matrices determine the exact empirical optimum, although the result is existential.  Active semidefinite programs give certified lower bounds, verified global separation gives upper bounds, and the resistance-design problem in \Cref{thm:resistance-design} supplies a conservative convex upper certificate with a computable optimality gap.  The finite covariance path in \Cref{sec:finite-path-example,sec:numerics} illustrates how the same objects can be computed and audited.  All $45$ selected-path covariance certificates succeed, while the locked predictive study records $7/15$ paired wins or ties against a predeclared requirement of $8/15$.  This negative gate is retained because covariance domination and predictive superiority are logically different claims.

\subsection{Relation to previous work}

Classical RKHS theory supplies the evaluation geometry used throughout \cite{aronszajn1950theory}.  Brownian and distance-induced kernels belong to the theory of negative type and positive-definite functions \cite{schoenberg1938metric,berg1984harmonic}; distance covariance and the equivalence of distance- and RKHS-based statistics give related modern formulations \cite{lyons2013distance,sejdinovic2013equivalence}.  Integral mixtures of kernels provide the direct analytic lineage for canonical BKL layers \cite{hotz2012representation}.  At the level of ambient function spaces, an RKHS cannot in general contain every continuous function on a compact metric space, and intermediate RKHS embeddings are tied to factorization and type--cotype properties \cite{steinwart2024continuous,scholpple2026spaces}.

Recursive function-space constructions differ according to where kernel averaging, nonlinear composition, and linear superposition are introduced.  Learned Brownian projections, neural Hilbert ladders, and reproducing-kernel chains are close Hilbertian comparators \cite{follain2025enhanced,chen2024neural,heeringa2025chains}.  Convex neural variation spaces and reproducing-kernel Banach spaces provide complementary non-Hilbertian viewpoints \cite{bach2017convex,bartolucci2023rkbs}.  The foundational BKL and path-atomic VBKL constructions \cite{mohammadigohari2026bkl,mohammadigohari2026vbkl} are distinguished here because the common covariance is formed for the unrestricted empirical union of canonical integral-RKHS ladders.

Kernel adaptation has also been studied through hyperkernels, semidefinite kernel learning, and multiple-kernel learning \cite{ong2005hyperkernels,lanckriet2004kernel,micchelli2005learning,sonnenburg2006mkl,rakotomamonjy2008simplemkl,gonen2011mkl}.  Statistical analyses of learned kernels include support-vector and generalization bounds for data-dependent kernel choice \cite{srebro2006learnedkernels,cortes2010kernels}, while representer questions for learned regularizers are treated in \cite{argyriou2009representer}.  Deep-kernel and finite learnable-kernel models provide parametrised alternatives \cite{wilson2016deep,wilson2016stochastic,ji2024finite}.  These works optimize finite or parametrised families; the present object is the complete recursive empirical BKL union.

The operator-geometric backbone uses absolutely summing operators and Pietsch factorisation \cite{pietsch1980operator,diestel1995absolutely}, together with Gaussian-process and convex-body geometry \cite{ledoux1991probability,pisier1999volume}.  The discrepancy-theoretic use of factorization norms in \cite{matousek2020factorization} is a related finite-dimensional precedent.  The contribution here is the exact identification of the BKL common-covariance value with an adaptive evaluation operator and with a covariance-dominated multiplier problem.

The simultaneous approximation results belong to the classical theory of widths \cite{pinkus1985nwidths} and are related to reduced-basis and greedy constructions \cite{prudhomme2002reduced,binev2011greedy,devore2013greedy}; a broad current account of sampling and approximation complexity is given in \cite{kriegullrich2026sampling}.  Data-dependent kernel subspaces obtained through Nystr{\"o}m and leverage-score methods provide complementary computational constructions \cite{williams2001nystrom,dellavecchia2024nystrom,chatalic2025leverage}.  Our quantifier is different: one sample-dependent covariance eigenspace controls the complete adaptive union rather than one fixed kernel or one target.

The certification layer draws on convex and semidefinite optimization \cite{boyd2004convex}, classical optimal-design and equivalence principles \cite{kiefer1959optimum,kieferwolfowitz1960equivalence}, modern computational optimal design \cite{huan2024oed}, and semi-infinite exchange methods \cite{blankenship1976infinitely,hettich1993semiinfinite}.  The sparse update and computable gap are related to Frank--Wolfe methodology \cite{jaggi2013frankwolfe}.  Graph upper bounds use effective resistance and spectral sparsification \cite{ghosh2008resistance,spielmansrivastava2011sparsification}, while the random-design transfers use nonasymptotic singular-value and matrix-concentration tools \cite{mendelsonpajor2006singular,vershynin2012nonasymptotic,tropp2015matrix}.

The novelty claimed here is therefore not the invention of these generic theories.  It is their exact empirical realization for the unrestricted BKL family and the Brownian Dirac, threshold, graph, phase-law, common-subspace, and certification consequences built on that realization.

\subsection{Organization}

\Cref{sec:setup} formulates the empirical BKL covariance family and the minimum-trace profile, including a finite motivating path.  \Cref{sec:2summing,sec:minimax} develop the operator and multiplier identities.  Brownian Dirac, threshold, and graph geometry appear in \Cref{sec:brownian}.  Deterministic, random, and perturbative regimes are collected in \Cref{sec:regimes,sec:richness}.  \Cref{sec:approximation} treats common approximation spaces and learning consequences.  \Cref{sec:certification} develops the certification algorithms, and \Cref{sec:numerics} gives the finite computational study.  Detailed proofs are integrated in the appendices.

\section{Empirical BKL Covariance Geometry}\label{sec:setup}

\subsection{Notation and logical conventions}

The sample is $\bx=(x_1,\ldots,x_n)$ and all unrestricted covariance matrices act on $\R^n$.  A matrix inequality is a Loewner inequality.  Symbols indexed by the depth $L$ belong to the unrestricted recursive BKL family; symbols indexed by a path parameter $s$ belong to the finite covariance path.  The principal quantities are summarized in \Cref{tab:focm-notation}.

\begin{table}[t]
\centering
\caption{Principal finite-sample objects.}
\label{tab:focm-notation}
\small
\begin{tabularx}{\textwidth}{p{0.21\textwidth}p{0.24\textwidth}X}
\toprule
Role & Objects & Meaning \\
\midrule
Adaptive family & $\sK_L(\bx)$, $\cA_L(\bx)$, $h_{L,\bx}$ & closed Gram family, union of evaluation ellipsoids, and its support function \\
Common covariance & $\tau_L$, $M_L$, $S_L$, $\Lambda_L$, $\Gamma_L^{\mathrm{dom}}$ & minimum trace, largest single-kernel trace, diagonal scale, normalized profile, and domination inflation \\
Gaussian geometry & $W_{G,1}$, $W_{G,2}$, $\chi_L$ & first and second Gaussian support moments and certificate defect \\
Approximation & $d_{m,L}$, $\omega_{m,L}$ & empirical Kolmogorov width and rank-$m$ covariance profile \\
Certification & $\tau(\mathcal F)$, $\operatorname{sep}_L$, $\Psi_{L,\mathcal E}$, $g_{\mathcal E}$ & active lower value, separation, resistance relaxation, and computable gap \\
Finite path & $\Sigma_s$, $K_s$, $\tau_{\mathcal S}$, $\widehat N$ & covariance path, path Gram matrices, finite envelope, and repaired covariance \\
\bottomrule
\end{tabularx}
\end{table}

We distinguish exact identities and universal upper bounds from conditional reverses, existential finite reductions, conservative relaxations, and implemented finite certificates.  In particular, $\Lambda_L$ is not a pure model-selection penalty, $\chi_L$ measures certificate tightness rather than a fast learning rate, local witness optimization is not verified global separation, and the common eigenspaces are initially empirical and transductive.

\subsection{Canonical Brownian Kernel Ladders}

The scalar two-sided Brownian kernel is
\begin{equation}\label{eq:brownian-kernel}
\kB(s,t):=\frac12\bigl(|s|+|t|-|s-t|\bigr),\qquad s,t\in\R.
\end{equation}

Let $\cX\subset\R^d$ be compact and put
\begin{align}
V_{\cX}:=\operatorname{span}(\cX)\subseteq\R^d.
\notag
\end{align}
If $V_{\cX}=\{0\}$, every BKL function and every empirical Gram matrix is zero.  Hence all substantive covariance values, Gaussian support moments, complexities, and approximation widths vanish; ratio-type quantities retain the explicit zero-case conventions stated below.  Every subsequent result is then immediate.  Henceforth assume $V_{\cX}\ne\{0\}$.
The first level is the RKHS of the restricted linear kernel
\begin{align}
k^{(0)}(x,x'):=\langle x,x'\rangle_{\R^d},
\qquad x,x'\in\cX.
\notag
\end{align}
Equivalently,
\begin{align}
\cH^{(1)}
&=\set{x\mapsto\langle w,x\rangle:w\in V_{\cX}},
\notag\\
\norm{x\mapsto\langle w,x\rangle}_{\cH^{(1)}}
&=\norm w_2,
\qquad w\in V_{\cX}.
\notag
\end{align}
For an arbitrary $\widetilde w\in\R^d$, the same function has the quotient characterization
\begin{align}
\norm{x\mapsto\langle\widetilde w,x\rangle}_{\cH^{(1)}}
=
\inf\set{\norm v_2:\langle v,x\rangle=\langle\widetilde w,x\rangle\text{ for all }x\in\cX}
=
\norm{P_{V_{\cX}}\widetilde w}_2,
\notag
\end{align}
where $P_{V_{\cX}}$ is the orthogonal projection onto $V_{\cX}$.  Thus the norm is independent of the chosen ambient coefficient vector.
Given an RKHS $H$ of functions on $\cX$ and a Borel probability measure $\mu$ on its unit sphere $S_1(H)$, define the integral Brownian kernel
\begin{equation}\label{eq:integral-kernel}
k[H,\mu](x,x')
:=\int_{S_1(H)}\kB(u(x),u(x'))\,d\mu(u).
\end{equation}
A canonical depth-$L$ ladder is a sequence of RKHSs
\begin{align}
\cH^{(1)},\cH^{(2)}_{\boldsymbol\mu},\ldots,
\cH^{(L)}_{\boldsymbol\mu},
\qquad
\cH^{(\ell+1)}_{\boldsymbol\mu}
=H_{k[\cH^{(\ell)}_{\boldsymbol\mu},\mu_\ell]},
\notag
\end{align}
where $\boldsymbol\mu=(\mu_1,\ldots,\mu_{L-1})$.  The full top-layer space and its canonical complexity are
\begin{align}
\mathfrak H_L
&:=\bigcup_{\boldsymbol\mu}\cH^{(L)}_{\boldsymbol\mu},\label{eq:full-space}\\
\Chat_L(f)
&:=\inf_{\boldsymbol\mu:\,f\in\cH^{(L)}_{\boldsymbol\mu}}
\norm{f}_{\cH^{(L)}_{\boldsymbol\mu}}.\label{eq:canonical-complexity}
\end{align}
We write
\begin{align}
B_L(r):=\set{f\in\mathfrak H_L:\Chat_L(f)\le r}.
\notag
\end{align}
The functional $\Chat_L$ is absolutely homogeneous.  The zero function belongs to every top-layer RKHS with norm zero, so $\Chat_L(0)=0$.  For $c\ne0$, homogeneity of every top-layer RKHS norm gives
$\Chat_L(cf)\le |c|\Chat_L(f)$, and applying this inequality to $f=c^{-1}(cf)$ gives the reverse inequality.  Hence
\begin{equation}\label{eq:ball-scaling}
\Chat_L(cf)=|c|\Chat_L(f),
\qquad
B_L(r)=rB_L(1)
\quad(r>0).
\end{equation}
For notational convenience, we also use $B_1(1)$ for the unit ball of the linear RKHS.

The Brownian diagonal identity $\kB(t,t)=|t|$ gives the recursive evaluation scale
\begin{equation}\label{eq:rho-def}
\rho_L(x):=\norm{x}_2^{\,2^{-(L-1)}},
\qquad L\ge1.
\end{equation}
The standard RKHS evaluation inequality and induction yield
\begin{equation}\label{eq:pointwise-scale}
|f(x)|\le \Chat_L(f)\rho_L(x),
\qquad
f\in\mathfrak H_L.
\end{equation}
The same induction through the Brownian kernel metric gives the recursive H\"older estimate
\begin{equation}\label{eq:holder-scale}
|f(x)-f(x')|
\le
\Chat_L(f)\norm{x-x'}_2^{\,2^{-(L-1)}},
\qquad
f\in\mathfrak H_L.
\end{equation}
For a top-layer kernel $k$ arising from any depth-$L$ ladder,
\begin{equation}\label{eq:diagonal-scale}
k(x,x)\le \rho_L(x)^2
=\norm{x}_2^{\,2^{-(L-2)}}.
\end{equation}
These estimates are proved in the BKL framework by recursively applying the RKHS pointwise bound and Brownian homogeneity~\cite{mohammadigohari2026bkl}.  A proof is given in Appendix~\ref{app:trace}.

\subsection{Empirical Kernel and Trace Families}

Fix a sample
\begin{align}
\bx=(x_1,\ldots,x_n)\in\cX^n.
\notag
\end{align}
For each canonical depth-$L$ ladder, let
\begin{align}
K_{\boldsymbol\mu}(\bx)
:=\bigl[k_{\boldsymbol\mu}^{(L-1)}(x_i,x_j)\bigr]_{i,j=1}^n
\in\psd^n
\notag
\end{align}
be its top-layer Gram matrix.  Define the empirical covariance family
\begin{equation}\label{eq:kernel-family}
\sK_L(\bx)
:=\cl\set{K_{\boldsymbol\mu}(\bx):\boldsymbol\mu\text{ is a canonical depth-$L$ ladder}}.
\end{equation}
This family is compact.  Every member is positive semidefinite, and \eqref{eq:diagonal-scale} gives $K_{ii}\le \rho_L(x_i)^2$.  The positive-semidefinite Cauchy--Schwarz inequality gives
\begin{equation*}
|K_{ij}|^2\le K_{ii}K_{jj}\le \rho_L(x_i)^2\rho_L(x_j)^2.
\end{equation*}
Thus the preclosure in \eqref{eq:kernel-family} is bounded in the finite-dimensional space of symmetric matrices, and its closure is compact.

For $K\succeq0$, define its unit RKHS trace ellipsoid
\begin{equation}\label{eq:ellipsoid}
E(K)
:=K^{1/2}B_2^n
=\set{a\in\ran K:a^\top K^\dagger a\le1},
\end{equation}
where $K^\dagger$ is the Moore--Penrose pseudoinverse.  Let
\begin{equation}\label{eq:trace-body}
\cA_L(\bx)
:=\cl\set{(f(x_1),\ldots,f(x_n)):f\in B_L(1)}.
\end{equation}

\begin{theorem}[Empirical covariance-body representation]\label{thm:trace-body}
For every $L\ge2$ and every sample $\bx$,
\begin{equation}\label{eq:trace-union}
\cA_L(\bx)
=\cl\bigcup_{K\in\sK_L(\bx)}E(K).
\end{equation}
Consequently, the support function
\begin{align}
h_{L,\bx}(z):=\sup_{a\in\cA_L(\bx)}z^\top a
\notag
\end{align}
satisfies
\begin{equation}\label{eq:support-covariance}
h_{L,\bx}(z)^2
=\sup_{K\in\sK_L(\bx)}z^\top Kz,
\qquad z\in\R^n.
\end{equation}
Moreover,
\begin{equation}\label{eq:empirical-width-support}
\Ghat_{\bx}(B_L(r))
:=\E_G\sup_{f\in B_L(r)}\frac1n\sum_{i=1}^nG_i f(x_i)
=\frac{r}{n}\E_G h_{L,\bx}(G),
\end{equation}
where $G\sim\mathcal N(0,I_n)$.
\end{theorem}

The closure in \eqref{eq:trace-union} is essential only because the infimum in \eqref{eq:canonical-complexity} need not be attained.  The theorem reduces the complete function-space union to a robust finite-dimensional covariance problem.

\subsection{The Minimum-Trace Common Covariance}\label{sec:profile}

Set
\begin{equation}\label{eq:d-and-S}
d_{L,i}:=\rho_L(x_i)^2
=\norm{x_i}_2^{\,2^{-(L-2)}},
\qquad
S_L(\bx):=\sum_{i=1}^n d_{L,i}.
\end{equation}
Coordinates with $d_{L,i}=0$ can be deleted, because \eqref{eq:pointwise-scale} forces every $f\in\mathfrak H_L$ to vanish at such $x_i$.

\begin{definition}[Common-covariance profiles]\label{def:profile}
The depth-$L$ covariance-envelope value, maximum single-kernel trace in the closed empirical family, normalized common-covariance profile, and common-domination inflation are
\begin{align}
\tau_L(\bx)
&:=\min_{N\succeq0}
\set{\tr N:N\succeq K\ \text{for every }K\in\sK_L(\bx)},
\label{eq:tau-def}\\
M_L(\bx)
&:=\max_{K\in\sK_L(\bx)}\tr K,
\label{eq:M-def}\\
\Lambda_L(\bx)
&:=
\begin{cases}
\bigl(\dfrac{\tau_L(\bx)}{S_L(\bx)}\bigr)^{1/2},&S_L(\bx)>0,\\[2mm]
0,&S_L(\bx)=0,
\end{cases}
\label{eq:lambda-def}\\
\Gamma_L^{\mathrm{dom}}(\bx)
&:=
\begin{cases}
\bigl(\dfrac{\tau_L(\bx)}{M_L(\bx)}\bigr)^{1/2},&M_L(\bx)>0,\\[2mm]
1,&M_L(\bx)=0.
\end{cases}
\label{eq:gamma-dom-def}
\end{align}
\end{definition}

The matrix $N$ is a single covariance whose ellipsoid contains the unit trace ellipsoid of every admissible ladder.  To see this, if $K\preceq N$, then
$h_{E(K)}(z)=\sqrt{z^\top Kz}\le\sqrt{z^\top Nz}=h_{E(N)}(z)$ for every $z$; domination of support functions of closed convex sets implies $E(K)\subseteq E(N)$.  The trace objective is tailored to Gaussian average radius:
\begin{align}
\E_G\sqrt{G^\top NG}\le\sqrt{\tr N}.
\notag
\end{align}

\begin{proposition}[Existence and scale decomposition]\label{prop:profile-existence}
The minimum in \eqref{eq:tau-def} and the maximum in \eqref{eq:M-def} are attained.  Moreover,
\begin{align}
0&\le M_L(\bx)\le S_L(\bx),
\notag\\
M_L(\bx)&\le\tau_L(\bx)\le nM_L(\bx)\le nS_L(\bx),
\label{eq:profile-universal}\\
0&\le\Lambda_L(\bx)\le\sqrt n,
\qquad
1\le\Gamma_L^{\mathrm{dom}}(\bx)\le\sqrt n.
\notag
\end{align}
If $S_L(\bx)>0$, then
\begin{equation}\label{eq:lambda-gamma-decomposition}
\Lambda_L(\bx)
=
\Gamma_L^{\mathrm{dom}}(\bx)
\bigl(\frac{M_L(\bx)}{S_L(\bx)}\bigr)^{1/2}.
\end{equation}
\end{proposition}

\begin{proof}
Compactness of $\sK_L(\bx)$ and continuity of the trace imply that the maximum in \eqref{eq:M-def} is attained.  Equation \eqref{eq:diagonal-scale} gives
\begin{align}
\tr K
=\sum_{i=1}^nK_{ii}
\le\sum_{i=1}^nd_{L,i}
=S_L(\bx),
\notag
\end{align}
so $M_L(\bx)\le S_L(\bx)$.

Every feasible common covariance $N$ dominates each $K\in\sK_L(\bx)$, and therefore
\begin{align}
\tr N\ge\tr K.
\notag
\end{align}
Taking the maximum over $K$ and then the minimum over $N$ proves
$M_L(\bx)\le\tau_L(\bx)$.

For the reverse comparison, let $M:=M_L(\bx)$.  Every $K\in\sK_L(\bx)$ is positive semidefinite, so
\begin{align}
\lambda_{\max}(K)
\le\tr K
\le M.
\notag
\end{align}
Hence $K\preceq MI_n$ for every $K$, which shows that $MI_n$ is feasible and
\begin{align}
\tau_L(\bx)\le\tr(MI_n)=nM_L(\bx).
\notag
\end{align}
This proves the middle chain in \eqref{eq:profile-universal}, including the earlier universal bound by $nS_L(\bx)$.

It remains to prove attainment of the minimum.  Let $(N_j)$ be a minimizing sequence.  Because $M_L(\bx)I_n$ is feasible, after discarding finitely many terms we may assume
\begin{align}
\tr N_j\le nM_L(\bx)+1.
\notag
\end{align}
The set of positive semidefinite matrices with trace at most $nM_L(\bx)+1$ is closed.  If $\lambda_1(N),\ldots,\lambda_n(N)$ are the nonnegative eigenvalues of such a matrix, then
\begin{align}
\norm N_{\mathrm F}^2
=\sum_i\lambda_i(N)^2
\le\bigl(\sum_i\lambda_i(N)\bigr)^2
=\bigl(\tr N\bigr)^2,
\notag
\end{align}
so the set is bounded and therefore compact.  Pass to a convergent subsequence $N_{j_k}\to N_*$.  For each fixed $K\in\sK_L(\bx)$, closedness of the positive-semidefinite cone gives
$N_*-K\succeq0$.  Thus $N_*$ is feasible, and continuity of the trace gives $\tr N_*=\tau_L(\bx)$.

If $S_L(\bx)=0$, then \eqref{eq:diagonal-scale} and positive semidefiniteness force every empirical kernel to be zero.  Hence $M_L(\bx)=\tau_L(\bx)=\Lambda_L(\bx)=0$, while $\Gamma_L^{\mathrm{dom}}(\bx)=1$ by convention.  Suppose $S_L(\bx)>0$.  If $M_L(\bx)=0$, the preceding argument again gives $\tau_L(\bx)=0$, and \eqref{eq:lambda-gamma-decomposition} holds.  If $M_L(\bx)>0$, divide the inequalities
$M_L\le\tau_L\le nM_L$ to obtain
$1\le\Gamma_L^{\mathrm{dom}}\le\sqrt n$.  The bound for $\Lambda_L$ follows from $\tau_L\le nS_L$, and \eqref{eq:lambda-gamma-decomposition} follows by multiplying the two defining ratios.
\end{proof}

\begin{theorem}[Profile-dependent Gaussian complexity]\label{thm:profile-complexity}
For every $L\ge2$, every $r>0$, and every sample $\bx$,
\begin{align}
\Ghat_{\bx}(B_L(r))
&\le \frac{r}{n}\sqrt{\tau_L(\bx)}\label{eq:profile-bound-tau}\\
&=\frac{r\Lambda_L(\bx)}{n}
\bigl(\sum_{i=1}^n\norm{x_i}_2^{\,2^{-(L-2)}}\bigr)^{1/2}\label{eq:profile-bound-sum}\\
&\le\frac{r\Lambda_L(\bx)}{\sqrt n}
\bigl(\max_{1\le i\le n}\norm{x_i}_2\bigr)^{2^{-(L-1)}}.
\label{eq:profile-bound-radius}
\end{align}
\end{theorem}

\begin{proof}
Let $N_*$ be an optimizer in \eqref{eq:tau-def}.  By \Cref{thm:trace-body}, for every realization of $G$,
\begin{align}
h_{L,\bx}(G)^2
=\sup_{K\in\sK_L(\bx)}G^\top KG
\le G^\top N_*G,
\notag
\end{align}
because $N_*-K\succeq0$ for every admissible $K$.  Both sides are nonnegative, so taking square roots preserves the inequality.  Jensen's inequality for the concave square-root map gives
\begin{align}
\E_Gh_{L,\bx}(G)
\le\E_G\sqrt{G^\top N_*G}
\le\sqrt{\E_G[G^\top N_*G]}.
\notag
\end{align}
Since $\E[GG^\top]=I_n$,
\begin{align}
\E_G[G^\top N_*G]
=\tr\!\bigl(N_*\E[GG^\top]\bigr)
=\tr N_*
=\tau_L(\bx).
\notag
\end{align}
Substitution into \eqref{eq:empirical-width-support} proves \eqref{eq:profile-bound-tau}.  If $S_L(\bx)=0$, then \Cref{prop:profile-existence} gives $\tau_L(\bx)=\Lambda_L(\bx)=0$, and all three displayed bounds are zero.  If $S_L(\bx)>0$, use \eqref{eq:lambda-def} to obtain \eqref{eq:profile-bound-sum}.  Finally,
$S_L(\bx)\le n\max_i d_{L,i}$ and $d_{L,i}^{1/2}=\norm{x_i}_2^{2^{-(L-1)}}$, which gives \eqref{eq:profile-bound-radius}.
\end{proof}

\begin{remark}[What the normalized profiles measure]\label{rem:profile-meaning}
The profile $\Lambda_L$ normalizes common-covariance scale by the universal diagonal envelope $S_L$; it is not a pure selection cost.  For one fixed kernel $K$, the identity $\Lambda_L^2=\tr K/S_L$ may give a value strictly below one.  Equation~\eqref{eq:lambda-gamma-decomposition} separates the single-kernel trace ratio $M_L/S_L$ from the common-domination inflation $\Gamma_L^{\mathrm{dom}}$.  For one fixed nonzero kernel, $\Gamma_L^{\mathrm{dom}}=1$.  Values above one record the extra trace required to dominate the full covariance family.
\end{remark}

\subsection{A finite covariance-path model}\label{sec:finite-path-example}

A finite depth-two model makes the preceding geometry explicit.  Let $z=\phi(x)\in\R^d$ be a frozen feature vector and, for $\Sigma\succeq0$, define
\begin{equation}\label{cove:eq:covariance-brownian-kernel}
k_{\Sigma}(z,z')
:=
\frac12\bigl(\norm{z}_{\Sigma}+\norm{z'}_{\Sigma}-\norm{z-z'}_{\Sigma}\bigr),
\qquad
\norm{u}_{\Sigma}:=(u^{\top}\Sigma u)^{1/2}.
\end{equation}
For a positive mean-one diagonal Fisher metric $\Sigma_{\mathrm F}$, put
\begin{equation}\label{cove:eq:covariance-path}
\Sigma_s=(1-s)I+s\Sigma_{\mathrm F},
\qquad
s\in\mathcal S:=\bigl\{0,\tfrac14,\tfrac12,\tfrac34,1\bigr\}.
\end{equation}
The endpoints are the identity and full-Fisher Brownian geometries.  The complete path is evaluated from two bases because
\begin{align}
\norm{u}_{\Sigma_s}^2
&=(1-s)\norm{u}_2^2+s\norm{u}_{\Sigma_{\mathrm F}}^2,\notag\\
\norm{u-v}_{\Sigma_s}^2
&=(1-s)\norm{u-v}_2^2+s\norm{u-v}_{\Sigma_{\mathrm F}}^2.
\label{cove:eq:path-affine-squares}
\end{align}

Let $\Theta$ be uniform on $\mathbb S^{d-1}$ and set
\begin{equation}\label{cove:eq:spherical-constant}
c_d:=\E\abs{\Theta_1}
=\frac{\Gamma(d/2)}{\sqrt\pi\,\Gamma((d+1)/2)}.
\end{equation}
\begin{theorem}[Exact spherical Brownian mixture]\label{cove:thm:spherical-mixture}
For every $\Sigma\succeq0$ and $z,z'\in\R^d$,
\begin{equation}\label{cove:eq:spherical-mixture}
k_{\Sigma}(z,z')
=
\frac1{c_d}\E_{\Theta}\kB\bigl(\ip{\Theta}{\Sigma^{1/2}z},\ip{\Theta}{\Sigma^{1/2}z'}\bigr).
\end{equation}
In particular, $k_{\Sigma}$ is positive semidefinite and is an exact depth-two BKL-type kernel.
\end{theorem}

On an anchor sample of size $n_{\mathrm a}$, let $K_s$ be the Gram matrix of $k_{\Sigma_s}$ and define
\begin{equation}\label{cove:eq:finite-envelope-program}
\tau_{\mathcal S}
:=
\min_{\alpha\in\R_+^{\mathcal S}}
\tr\Bigl(\sum_{s\in\mathcal S}\alpha_sK_s\Bigr)
\quad\text{subject to}\quad
\sum_{s\in\mathcal S}\alpha_sK_s\succeq K_t
\ \text{for every }t\in\mathcal S.
\end{equation}
If $\widehat N_{\mathrm{raw}}$ is a numerical mixture, set
\begin{align}
\widehat\lambda_{\min}
&:=\min_{t\in\mathcal S}\lambda_{\min}(\widehat N_{\mathrm{raw}}-K_t),
\label{cove:eq:raw-residual}\\
\rho&:=\max\{0,\mu-\widehat\lambda_{\min}\},
\qquad
\widehat N:=\widehat N_{\mathrm{raw}}+\rho I.
\label{cove:eq:identity-repair}
\end{align}
\begin{proposition}[Exact path evaluation and numerical repair]\label{cove:prop:path-basis}
The identities in \eqref{cove:eq:path-affine-squares} determine every path Gram matrix without new $d$-dimensional pairwise products.  Moreover,
\begin{equation}\label{cove:eq:repaired-domination}
\widehat N-K_t\succeq\mu I
\qquad(t\in\mathcal S).
\end{equation}
\end{proposition}
The finite path is used in \Cref{sec:numerics}.  It is not a certificate for the unrestricted family $\sK_L(\bx)$.

\section{Operator Factorization and Gaussian Geometry}\label{sec:2summing}

The profile has an exact operator-ideal interpretation.  Let
\begin{align}
V_{L,\bx}:=\operatorname{span}\cA_L(\bx),
\notag
\end{align}
and equip it with the norm whose closed unit ball is
\begin{align}
\cl\absconv\cA_L(\bx).
\notag
\end{align}
Call the resulting finite-dimensional normed space $X_{L,\bx}$.  Convexification does not change the support function.  Define the dual evaluation operator
\begin{equation}\label{eq:T-operator}
T_{L,\bx}:\ell_2^n\longrightarrow X_{L,\bx}^*,
\qquad
T_{L,\bx}z:=\sum_{i=1}^n z_i\delta_{x_i}.
\end{equation}
Then
\begin{equation}\label{eq:T-norm}
\norm{T_{L,\bx}z}_{X_{L,\bx}^*}=h_{L,\bx}(z).
\end{equation}

Recall that, for an operator $T:\ell_2^n\to Y$, the absolutely $2$-summing norm can be written as
\begin{equation}\label{eq:pi2-def}
\pi_2(T)^2
=\sup\set{
\sum_{r=1}^m\norm{Tz_r}_Y^2:
\sum_{r=1}^m z_rz_r^\top\preceq I_n,\ m\in\N,\ m\ge1
}.
\end{equation}

\begin{theorem}[Exact $2$-summing identity]\label{thm:pi2}
For every $L\ge2$ and every sample $\bx$,
\begin{equation}\label{eq:pi2-tau}
\ \pi_2(T_{L,\bx})^2=\tau_L(\bx).\
\end{equation}
Consequently, if $S_L(\bx)>0$,
\begin{equation}\label{eq:lambda-pi2}
\Lambda_L(\bx)
=\frac{\pi_2(T_{L,\bx})}{\bigl(\sum_{i=1}^n\rho_L(x_i)^2\bigr)^{1/2}}.
\end{equation}
If $S_L(\bx)=0$, then $T_{L,\bx}=0$, $\pi_2(T_{L,\bx})=0$, and $\Lambda_L(\bx)=0$.
\end{theorem}

The identity is stronger than a generic factorization upper bound: it states that the robust minimum-trace covariance problem is exactly the finite-resolution $2$-summing norm of adaptive evaluation.  Its proof proceeds first for a finite family and then by compact passage; see Appendix~\ref{app:pi2}.

\subsection{Finite-Family Primal and Dual Programs}

For $K_1,\ldots,K_M\in\psd^n$, define
\begin{equation}\label{eq:finite-primal}
\tau(K_1,\ldots,K_M)
:=\min_{N\in\mathbb S^n}\set{\tr N:N-K_j\succeq0,\ j=1,\ldots,M}.
\end{equation}
Strict feasibility gives the dual program
\begin{equation}\label{eq:finite-dual}
\tau(K_1,\ldots,K_M)
=\max_{Z_1,\ldots,Z_M}
\set{
\sum_{j=1}^M\ip{Z_j}{K_j}:
Z_j\succeq0,\ \sum_{j=1}^M Z_j=I_n
}.
\end{equation}
Here $\ip{A}{B}=\tr(A^\top B)$.  The dual decomposes the identity into positive semidefinite directional weights; each direction is allowed to place its mass on a different kernel.  This is the finite-dimensional mechanism behind representation selection.

For the compact full family,
\begin{equation}\label{eq:finite-subsets}
\tau_L(\bx)
=\sup_{\mathcal F\subset\sK_L(\bx),\ |\mathcal F|<\infty}
\tau(\mathcal F).
\end{equation}
Thus finite SDPs approximate the robust profile monotonically from below, while a valid upper certificate requires verifying domination over the complete family.

\subsection{Robust Multiplier Geometry and Gaussian Defect}\label{sec:minimax}

The common-covariance profile is an upper certificate for the standard Gaussian width.  This section identifies the exact minimax problem solved by that certificate and then isolates the abstract geometry required to reverse the bound for Gaussian multipliers.  The covariance-dominated multiplier identity is distribution free, while empirical Gaussian cotype quantifies the gap for the standard Gaussian law.  BKL-specific threshold reverses are developed after the Brownian reduction in \Cref{sec:richness}.

For brevity, write
\begin{equation}\label{eq:gaussian-width-moments}
W_{G,1}(L,\bx)
:=\E_G h_{L,\bx}(G),
\qquad
W_{G,2}(L,\bx)
:=\bigl(\E_G h_{L,\bx}(G)^2\bigr)^{1/2},
\end{equation}
where $G\sim\mathcal N(0,I_n)$.  By \eqref{eq:empirical-width-support},
\begin{equation}\label{eq:complexity-WG1}
\Ghat_{\bx}(B_L(r))
=\frac{r}{n}W_{G,1}(L,\bx).
\end{equation}

\subsubsection{A covariance-dominated multiplier minimax identity}

Let $\mathfrak P_{\preceq I_n}$ denote the class of Borel probability measures $\nu$ on $\R^n$ satisfying
\begin{equation}\label{eq:subisotropic-laws}
\int_{\R^n}\norm{z}_2^2\,d\nu(z)<\infty,
\qquad
\int_{\R^n}z\,d\nu(z)=0,
\qquad
\int_{\R^n}zz^\top\,d\nu(z)\preceq I_n.
\end{equation}

\begin{theorem}[Robust covariance-dominated multiplier minimax]\label{thm:robust-minimax}
For every depth $L\ge2$ and every sample $\bx$,
\begin{equation}\label{eq:robust-minimax}
\tau_L(\bx)
=
\sup_{\nu\in\mathfrak P_{\preceq I_n}}
\int_{\R^n} h_{L,\bx}(z)^2\,d\nu(z).
\end{equation}
Consequently, for every $r>0$,
\begin{equation}\label{eq:robust-minimax-scaled}
\frac{r}{n}\sqrt{\tau_L(\bx)}
=
\sup_{\nu\in\mathfrak P_{\preceq I_n}}
\frac{r}{n}
\bigl(
\int_{\R^n}h_{L,\bx}(z)^2\,d\nu(z)
\bigr)^{1/2}.
\end{equation}
\end{theorem}

\begin{remark}[Meaning of the minimax identity]\label{rem:multiplier-minimax-scope}
\Cref{thm:robust-minimax} is an exact minimax characterization of the covariance certificate over multiplier laws whose second-moment matrix is bounded by the identity.  It is not, by itself, a minimax excess-risk lower bound for a supervised learning problem.  Its role is to identify precisely what information is used by the common-covariance relaxation: among all multiplier distributions constrained only through their covariance, the value $\sqrt{\tau_L(\bx)}$ is optimal.
\end{remark}

A reverse inequality for the standard Gaussian law cannot follow from $\tau_L$ alone for an arbitrary covariance family.

\begin{examplex}[Why covariance domination alone does not imply a Gaussian reverse]\label{ex:coordinate-kernels}
Let
\begin{align}
\mathscr K=\set{e_ie_i^\top:i\in[n]}.
\notag
\end{align}
Then the associated support function is $h(z)=\norm{z}_\infty$, and the minimum-trace common covariance value is $\tau(\mathscr K)=n$.  Nevertheless,
\begin{equation}\label{eq:coordinate-gaussian-width}
\E\norm{G}_\infty
\le
\sqrt{2\log(2n)}.
\end{equation}
Hence
\begin{align}
\frac{\sqrt{\tau(\mathscr K)}}{\E\norm{G}_\infty}
\ge
\sqrt{\frac{n}{2\log(2n)}}.
\notag
\end{align}
Thus a constant-factor or polylogarithmic Gaussian reverse requires additional structure beyond the common-covariance value.  For BKLs, that additional structure is supplied by the recursive Brownian trace geometry.
\end{examplex}

\subsubsection{Gaussian cotype and the exact Gaussian defect}

Let
\begin{equation}\label{eq:empirical-dual-range}
Y_{L,\bx}
:=T_{L,\bx}(\ell_2^n)
\subseteq X_{L,\bx}^*
\end{equation}
with the norm inherited from $X_{L,\bx}^*$.  If $Y_{L,\bx}\ne\set{0}$, define its Gaussian cotype-$2$ constant $C_{2,L}^{\mathrm g}(\bx)$ as the least $C\in[1,\infty)$ for which
\begin{equation}\label{eq:gaussian-cotype-def}
\bigl(\sum_{j=1}^m\norm{y_j}_{X_{L,\bx}^*}^2\bigr)^{1/2}
\le
C
\bigl(
\E_\gamma
\norm{\sum_{j=1}^m\gamma_jy_j}_{X_{L,\bx}^*}^2
\bigr)^{1/2}
\end{equation}
for every $m\ge1$ and every $y_1,\ldots,y_m\in Y_{L,\bx}$, where $\gamma_1,\ldots,\gamma_m$ are independent standard Gaussian variables.  If $Y_{L,\bx}=\set{0}$, set $C_{2,L}^{\mathrm g}(\bx):=1$.  The constant is finite because $Y_{L,\bx}$ is finite dimensional and all norms on a finite-dimensional vector space are equivalent.

\begin{theorem}[Gaussian reverse through empirical cotype]\label{thm:cotype-reverse}
Put
\begin{align}
\kappa_{\mathrm G}:=\sqrt{1+\frac\pi2}.
\notag
\end{align}
Then
\begin{equation}\label{eq:cotype-rms-reverse}
\sqrt{\tau_L(\bx)}
\le
C_{2,L}^{\mathrm g}(\bx)\,
W_{G,2}(L,\bx),
\end{equation}
\begin{equation}\label{eq:gaussian-moment-comparison}
W_{G,2}(L,\bx)
\le
\kappa_{\mathrm G}W_{G,1}(L,\bx),
\end{equation}
and therefore
\begin{equation}\label{eq:cotype-complexity-sandwich}
\frac{r\sqrt{\tau_L(\bx)}}
{\kappa_{\mathrm G}C_{2,L}^{\mathrm g}(\bx)n}
\le
\Ghat_{\bx}(B_L(r))
\le
\frac{r\sqrt{\tau_L(\bx)}}{n}.
\end{equation}
\end{theorem}

Define the empirical Gaussian defect by
\begin{equation}\label{eq:gaussian-defect}
\chi_L(\bx)
:=
\begin{cases}
\dfrac{\sqrt{\tau_L(\bx)}}{W_{G,2}(L,\bx)},&\tau_L(\bx)>0,\\[2mm]
1,&\tau_L(\bx)=0.
\end{cases}
\end{equation}

\begin{corollary}[Exact multiplier-to-Gaussian defect]\label{cor:gaussian-defect}
For every sample,
\begin{equation}\label{eq:defect-range}
1\le\chi_L(\bx)\le C_{2,L}^{\mathrm g}(\bx).
\end{equation}
If $\tau_L(\bx)>0$, then
\begin{equation}\label{eq:defect-minimax-ratio}
\chi_L(\bx)^2
=
\frac{
\displaystyle
\sup_{\nu\in\mathfrak P_{\preceq I_n}}
\int h_{L,\bx}(z)^2\,d\nu(z)
}{
\displaystyle
\E_Gh_{L,\bx}(G)^2
}.
\end{equation}
Thus $\chi_L(\bx)$ is the exact loss incurred when the worst covariance-dominated multiplier law in \Cref{thm:robust-minimax} is replaced by the standard Gaussian law, at the level of second moments.
\end{corollary}

\section{Brownian Dirac, Threshold, and Graph Geometry}\label{sec:brownian}

\subsection{Reduction to the Closed Last-Layer Dirac Trace Family}

Let $\sS_{L-1}(\bx)$ be the closure of all vectors
\begin{align}
a=(u(x_1),\ldots,u(x_n))
\notag
\end{align}
obtained from unit-sphere elements $u$ of any admissible level-$(L-1)$ RKHS.  Define the Brownian pullback matrix
\begin{equation}\label{eq:B-matrix}
\mathbf B(a)_{ij}:=\kB(a_i,a_j).
\end{equation}

\begin{theorem}[Dirac reduction]\label{thm:dirac-reduction}
For every $z\in\R^n$,
\begin{equation}\label{eq:dirac-support}
h_{L,\bx}(z)^2
=\sup_{a\in\sS_{L-1}(\bx)}z^\top\mathbf B(a)z.
\end{equation}
Moreover, for $N\succeq0$,
\begin{equation}\label{eq:dirac-domination}
N\succeq K\ \text{for all }K\in\sK_L(\bx)
\quad\Longleftrightarrow\quad
N\succeq\mathbf B(a)\ \text{for all }a\in\sS_{L-1}(\bx).
\end{equation}
\end{theorem}

Indeed, every top kernel is an integral of matrices $\mathbf B(a)$, while every Dirac measure at an admissible unit atom is itself allowed.  The final mixing measure therefore disappears from both the support and robust-domination problems.

\subsection{Threshold Covariance Representation}

For $a\in\R^n$ and $t\ge0$, define
\begin{equation}\label{eq:threshold-vectors}
v_t^+(a):=\bigl(\mathbf 1_{\{a_i\ge t\}}\bigr)_{i=1}^n,
\qquad
v_t^-(a):=\bigl(\mathbf 1_{\{a_i\le-t\}}\bigr)_{i=1}^n.
\end{equation}
The Brownian identity gives the matrix-valued layer-cake formula
\begin{equation}\label{eq:threshold-matrix}
\mathbf B(a)
=\int_0^\infty
\bigl[v_t^+(a)v_t^+(a)^\top+v_t^-(a)v_t^-(a)^\top\bigr]dt.
\end{equation}
This representation turns the covariance profile into a continuous aggregation of factorization problems for threshold set systems.  The following certificate will be useful.

\begin{proposition}[Threshold covariance certificate]\label{prop:threshold-certificate}
Suppose $t\mapsto D_t\in\psd^n$ is measurable and integrable, and for almost every $t\ge0$,
\begin{align}
D_t\succeq vv^\top
\quad\text{for every }v\in
\set{v_t^+(a),v_t^-(a):a\in\sS_{L-1}(\bx)}.
\notag
\end{align}
Then
\begin{equation}\label{eq:threshold-certificate-bound}
\tau_L(\bx)\le 2\int_0^\infty\tr D_t\,dt.
\end{equation}
\end{proposition}

This is the point at which classical set-system factorization norms and discrepancy methods can enter~\cite{matousek2020factorization}.  The minimum-trace objective, rather than the maximum diagonal, is the natural quantity for Gaussian average width.

\subsection{Recursive Threshold Factorization}\label{sec:recursive-threshold}

The layer-cake representation contains more information than the marginal pointwise envelope: it records which sample coordinates can occur together in a signed superlevel set of a previous-layer unit trace.  We now make that recursive information quantitative.

For $A\subseteq[n]$, let $\bone_A\in\{0,1\}^n$ denote its indicator vector.  Define the complete signed threshold set system generated at depth $L-1$ by
\begin{equation}\label{eq:threshold-family}
\mathcal F_{L-1}(\bx)
:=\set{
\set{i:a_i\ge t},\ \set{i:a_i\le -t}
:\ a\in\sS_{L-1}(\bx),\ t\ge0
}.
\end{equation}
Although $\sS_{L-1}(\bx)$ may be infinite, $\mathcal F_{L-1}(\bx)$ is a finite set system because its members are subsets of the finite ground set $[n]$.

A collection $\mathscr D\subseteq2^{[n]}$ is called a \emph{decomposition dictionary} for $\mathcal F_{L-1}(\bx)$ if there is an integer $s\ge1$ such that every $A\in\mathcal F_{L-1}(\bx)$ can be written as a disjoint union
\begin{equation}\label{eq:dictionary-decomposition}
A=D_1\mathbin{\dot\cup}\cdots\mathbin{\dot\cup}D_m,
\qquad
m\le s,
\qquad
D_j\in\mathscr D.
\end{equation}
The empty set is represented by the empty union.  Let $s(\mathscr D)$ be the least admissible $s$ in \eqref{eq:dictionary-decomposition}, and define the overlap degree
\begin{equation}\label{eq:dictionary-overlap}
\Delta(\mathscr D)
:=\max_{i\in[n]}\abs{\set{D\in\mathscr D:i\in D}}.
\end{equation}
The previous-layer threshold decomposition complexity is
\begin{equation}\label{eq:D-profile}
\mathfrak D_{L-1}(\bx)
:=\min_{\mathscr D}
 s(\mathscr D)\Delta(\mathscr D),
\end{equation}
where the minimum ranges over all decomposition dictionaries.  The singleton dictionary $\set{\{1\},\ldots,\{n\}}$ is always admissible, so
\begin{equation}\label{eq:D-profile-universal}
\mathfrak D_{L-1}(\bx)\le n.
\end{equation}
The minimum in \eqref{eq:D-profile} is attained.  Indeed, the ground set $[n]$ has only finitely many subsets, and hence there are only finitely many possible dictionaries $\mathscr D\subseteq2^{[n]}$.

For a finite set system $\mathcal F\subseteq2^{[n]}$, define its minimum-trace threshold domination value
\begin{equation}\label{eq:Theta-def}
\Theta(\mathcal F)
:=\min\set{
\tr H:
H\succeq0,\quad
H\succeq\bone_A\bone_A^\top
\text{ for every }A\in\mathcal F
}.
\end{equation}
This minimum is well defined and attained.  For every $A\subseteq[n]$,
$\bone_A\bone_A^\top\preceq\norm{\bone_A}_2^2 I\preceq nI$, so $nI$ is feasible.  A minimizing sequence may be restricted to a trace-bounded compact subset of $\psd^n$, and the Loewner constraints are closed.  We use the convention $\Theta(\varnothing)=0$.
For $t\ge0$ and $\delta>0$, let $\mathcal F_{L-1}^{+,\mathrm{st}}(t,\delta)$ be the sets $A\subseteq[n]$ for which there exists $a\in\sS_{L-1}(\bx)$ satisfying
\begin{equation}\label{eq:stable-plus}
\set{i:a_i\ge s}=A
\qquad\text{for every }s\in[t,t+\delta],
\end{equation}
and define $\mathcal F_{L-1}^{-,\mathrm{st}}(t,\delta)$ by replacing $a_i\ge s$ with $a_i\le-s$.  Set
\begin{equation}\label{eq:stable-family}
\mathcal F_{L-1}^{\mathrm{st}}(t,\delta)
:=\mathcal F_{L-1}^{+,\mathrm{st}}(t,\delta)
\cup
\mathcal F_{L-1}^{-,\mathrm{st}}(t,\delta)
\end{equation}
and
\begin{equation}\label{eq:R-profile}
\mathfrak R_{L-1}(\bx)
:=\frac{1}{S_L(\bx)}
\sup_{t\ge0,\,\delta>0}
\delta\,
\Theta\!\bigl(\mathcal F_{L-1}^{\mathrm{st}}(t,\delta)\bigr),
\end{equation}
whenever $S_L(\bx)>0$.  If $S_L(\bx)=0$, set $\mathfrak R_{L-1}(\bx):=0$.  The supremum is finite in the nonzero case.  If a nonempty set is stable on $[t,t+\delta]$, then \eqref{eq:pointwise-scale} gives $\delta\le\max_i d_{L,i}$; if only the empty set is stable, its domination value is zero.  Also $\Theta(\mathcal F)\le n^2$ for every $\mathcal F\subseteq2^{[n]}$, since $nI$ is feasible.

\begin{theorem}[Recursive threshold sandwich]\label{thm:recursive-threshold}
Let $L\ge2$ and let $\bx=(x_1,\ldots,x_n)$ satisfy $S_L(\bx)>0$.  Then
\begin{equation}\label{eq:recursive-sandwich}
\mathfrak R_{L-1}(\bx)
\le
\Lambda_L(\bx)^2
\le
2\mathfrak D_{L-1}(\bx).
\end{equation}
In particular,
\begin{align}
\tau_L(\bx)
&\le 2\mathfrak D_{L-1}(\bx)S_L(\bx),
\label{eq:recursive-tau}\\
\Ghat_{\bx}(B_L(r))
&\le
\frac{r\sqrt{2\mathfrak D_{L-1}(\bx)}}{n}
\bigl(\sum_{i=1}^n
\norm{x_i}_2^{\,2^{-(L-2)}}\bigr)^{1/2}
\label{eq:recursive-G-sum}\\
&\le
r\sqrt{\frac{2\mathfrak D_{L-1}(\bx)}{n}}
\bigl(\max_{1\le i\le n}\norm{x_i}_2\bigr)^{2^{-(L-1)}}.
\label{eq:recursive-G-radius}
\end{align}
Thus the full adaptive BKL ball is unchanged: the statistical price is expressed entirely through the signed threshold geometry already realizable by the preceding layer.
\end{theorem}

\begin{corollary}[Hierarchical threshold geometry]\label{cor:hierarchical-threshold}
Assume that there is a rooted laminar partition tree on $[n]$ of height $H$ such that every set in $\mathcal F_{L-1}(\bx)$ is a disjoint union of at most $q$ tree nodes.  Then
\begin{equation}\label{eq:hierarchical-D}
\mathfrak D_{L-1}(\bx)\le q(H+1),
\qquad
\Lambda_L(\bx)\le\sqrt{2q(H+1)},
\end{equation}
and
\begin{equation}\label{eq:hierarchical-G}
\Ghat_{\bx}(B_L(r))
\le
r\sqrt{\frac{2q(H+1)}{n}}
\bigl(\max_i\norm{x_i}_2\bigr)^{2^{-(L-1)}}.
\end{equation}
Consequently, a balanced hierarchy with $H=O(\log n)$ and bounded $q$ gives the full-class rate $O(\!\sqrt{\log n/n}\!)$, up to the displayed radius factors.
\end{corollary}

\begin{corollary}[Ordered traces with bounded oscillation]\label{cor:ordered-threshold}
Fix an ordering of the sample indices and suppose that every set in $\mathcal F_{L-1}(\bx)$ is a union of at most $q$ intervals in that ordering.  Put $H_n:=\lceil\log_2(2n)\rceil$.  Then
\begin{equation}\label{eq:ordered-D}
\mathfrak D_{L-1}(\bx)\le2qH_n^2,
\qquad
\Lambda_L(\bx)\le2\sqrt q\,H_n,
\end{equation}
and hence
\begin{equation}\label{eq:ordered-G}
\Ghat_{\bx}(B_L(r))
\le
2r\sqrt q\,
\frac{H_n}{\sqrt n}
\bigl(\max_i\norm{x_i}_2\bigr)^{2^{-(L-1)}}.
\end{equation}
Thus bounded empirical oscillation yields an $O(\log n/\sqrt n)$ full-class complexity bound uniformly in depth apart from the explicit radius exponent, which tends to zero as $L\to\infty$.
\end{corollary}

\begin{remark}[Scope of the structured threshold regimes]\label{rem:threshold-scope}
The hypotheses of \Cref{cor:hierarchical-threshold,cor:ordered-threshold} are geometric conditions on the complete signed-threshold family generated by the preceding BKL layer.  They are not asserted to hold automatically for every sample or every depth.  They must be verified from the trace geometry or covariance certificationally.  The orthogonal construction in \Cref{thm:orthogonal} shows that no uniform bounded-$q$ conclusion is possible for the unrestricted class.
\end{remark}

\subsection{A Graph-Coarea/Effective-Resistance Majorant}

Let $G=([n],E,w)$ be a connected weighted graph with strictly positive edge weights $w_{ij}>0$, Laplacian $L_G$, and pseudoinverse $L_G^\dagger$.  Define
\begin{align}
\operatorname{TV}_G(a):=\sum_{\{i,j\}\in E}w_{ij}|a_i-a_j|.
\notag
\end{align}

\begin{theorem}[Brownian graph majorant]\label{thm:graph-majorant}
For every $a\in\R^n$,
\begin{equation}\label{eq:graph-psd}
\mathbf B(a)
\preceq
\frac{2\norm{a}_1}{n}\bone\bone^\top
+2\operatorname{TV}_G(a)L_G^\dagger.
\end{equation}
Consequently, with
\begin{align*}
M_{L-1,1}(\bx)&:=\sup_{a\in\sS_{L-1}(\bx)}\norm{a}_1,\\
V_{L-1,G}(\bx)&:=\sup_{a\in\sS_{L-1}(\bx)}\operatorname{TV}_G(a),
\end{align*}
we have
\begin{equation}\label{eq:graph-tau}
\tau_L(\bx)
\le 2M_{L-1,1}(\bx)
+2V_{L-1,G}(\bx)\tr L_G^\dagger.
\end{equation}
\end{theorem}

The BKL pointwise and H\"older estimates imply
\begin{align}
M_{L-1,1}(\bx)
&\le S_L(\bx),\label{eq:M-bound}\\
V_{L-1,G}(\bx)
&\le\sum_{\{i,j\}\in E}w_{ij}
\norm{x_i-x_j}_2^{\,2^{-(L-2)}}.
\label{eq:V-bound}
\end{align}
If $S_L(\bx)>0$, every connected weighted graph yields the explicit normalized certificate
\begin{equation}\label{eq:graph-lambda}
\Lambda_L(\bx)^2
\le 2+
\frac{2\tr L_G^\dagger}{S_L(\bx)}
\sum_{\{i,j\}\in E}w_{ij}
\norm{x_i-x_j}_2^{\,2^{-(L-2)}}.
\end{equation}
If $S_L(\bx)=0$, then every empirical kernel is zero and $\Lambda_L(\bx)=0$ by \Cref{prop:profile-existence}.
Unlike a scan-statistic argument, \eqref{eq:graph-psd} is a deterministic Loewner-order majorization and preserves the covariance structure throughout the proof.

\section{Deterministic, Random, and Perturbative Regimes}\label{sec:regimes}

\subsection{Fixed, Finite, and Repeated-Location Families}

\begin{proposition}[A fixed ladder]\label{prop:fixed-ladder}
If the empirical family consists of one kernel $K$, then
\begin{align}
\tau_L(\bx)=M_L(\bx)=\tr K\le S_L(\bx),
\qquad
\Lambda_L(\bx)^2=\frac{\tr K}{S_L(\bx)}\le1.
\notag
\end{align}
If $K\ne0$, then $\Gamma_L^{\mathrm{dom}}(\bx)=1$.  Indeed, $N=K$ is feasible, while $N\succeq K$ implies $\tr N\ge\tr K$.  Thus \Cref{thm:profile-complexity} reduces to the usual fixed-RKHS trace bound.
\end{proposition}

\begin{proposition}[A finite kernel family]\label{prop:finite-family}
Let $K_1,\ldots,K_M\in\psd^n$.  If every kernel in the empirical family belongs to the convex hull of $K_1,\ldots,K_M$, then
\begin{equation}\label{eq:M-bound-profile}
\Lambda_L(\bx)\le\sqrt M.
\end{equation}
If $M_L(\bx)>0$, then also
\begin{equation}\label{eq:M-bound-domination}
\Gamma_L^{\mathrm{dom}}(\bx)\le\sqrt M.
\end{equation}
\end{proposition}

\begin{proof}
Let $\Delta_M$ be the probability simplex and define
\begin{equation}\label{eq:finite-family-coefficient-set}
\cC
:=
\set{\lambda\in\Delta_M:
\sum_{j=1}^M\lambda_jK_j\in\sK_L(\bx)}.
\end{equation}
The set $\cC$ is nonempty by the convex-hull hypothesis.  It is compact because $\Delta_M$ is compact, the map
$\lambda\mapsto\sum_j\lambda_jK_j$ is continuous, and $\sK_L(\bx)$ is closed.  Hence, for each $j\in[M]$, the number
\begin{align}
\alpha_j:=\max_{\lambda\in\cC}\lambda_j
\notag
\end{align}
is attained.  Set
\begin{align}
N:=\sum_{j=1}^M\alpha_jK_j.
\notag
\end{align}
Fix $K\in\sK_L(\bx)$.  Choose $\lambda\in\cC$ such that
$K=\sum_j\lambda_jK_j$.  Since $\lambda_j\le\alpha_j$ and every $K_j$ is positive semidefinite,
\begin{align}
N-K
=\sum_{j=1}^M(\alpha_j-\lambda_j)K_j
\succeq0.
\notag
\end{align}
Thus $N$ is feasible for the common-covariance problem.

For each $j$, choose $\lambda^{(j)}\in\cC$ satisfying
$\lambda_j^{(j)}=\alpha_j$.  Positivity of the traces gives
\begin{align}
\alpha_j\tr K_j
&\le
\sum_{\ell=1}^M\lambda_\ell^{(j)}\tr K_\ell\notag\\
&=
\tr\bigl(\sum_{\ell=1}^M\lambda_\ell^{(j)}K_\ell\bigr)
\le M_L(\bx).
\notag
\end{align}
Summing over $j$ yields
\begin{align}
\tr N
=\sum_{j=1}^M\alpha_j\tr K_j
\le M M_L(\bx)
\le M S_L(\bx).
\notag
\end{align}
Since $N$ is feasible, $\tau_L(\bx)\le\tr N$.  Division by $S_L(\bx)$ proves \eqref{eq:M-bound-profile}; when $S_L(\bx)=0$, the conclusion is already contained in the zero convention of \Cref{def:profile}.  If $M_L(\bx)>0$, division by $M_L(\bx)$ proves \eqref{eq:M-bound-domination}.
\end{proof}

\begin{remark}[Common-covariance cost versus a direct finite-union bound]\label{rem:finite-union-caveat}
The factor $\sqrt M$ is a bound for one covariance ellipsoid that dominates all $M$ RKHS ellipsoids.  It need not be the sharp Gaussian complexity of a finite union: a direct maximum-of-$M$ argument can instead give a logarithmic model-selection factor.  The common-covariance profile is designed to treat the unrestricted adaptive family and, simultaneously, to provide one spectral approximation object for all its members.
\end{remark}

\begin{proposition}[Only $k$ distinct sample locations]\label{prop:k-distinct}
If $x_1,\ldots,x_n$ take only $k$ distinct values, then
\begin{equation}\label{eq:k-distinct-profile}
\Lambda_L(\bx)\le\sqrt k
\end{equation}
and therefore
\begin{equation}\label{eq:k-distinct-complexity}
\Ghat_{\bx}(B_L(r))
\le r\bigl(\max_i\norm{x_i}_2\bigr)^{2^{-(L-1)}}\sqrt{\frac{k}{n}}.
\end{equation}
\end{proposition}

This result applies to exact repetitions and, after approximation, motivates effective-support variants for clustered samples.

\subsection{Orthogonal Designs: Worst-Case Growth}

Take $d=n$ and $x_i=e_i\in\R^n$.  Then $d_{L,i}=1$ and $S_L=n$.

\begin{theorem}[Orthogonal growth law]\label{thm:orthogonal}
For every fixed depth $L\ge2$, there exist universal constants $c,C>0$ such that
\begin{align}
c\,n^{1/2-2^{-L}}
&\le\Lambda_L(e_1,\ldots,e_n)
\le C\,n^{1/2-2^{-L}},\label{eq:orthogonal-lambda}\\
c\,n^{-2^{-L}}
&\le\Ghat_{(e_i)_{i=1}^n}(B_L(1))
\le C\,n^{-2^{-L}}.
\label{eq:orthogonal-complexity}
\end{align}
\end{theorem}

The upper bound follows from a recursive $\ell_1$ radius estimate for the empirical trace body.  The lower bound constructs, inside the unchanged BKL unit ball, a complete sign cube of amplitude comparable to $n^{-2^{-L}}$.  The construction uses the Brownian Cameron--Martin profile
\begin{equation}\label{eq:psi-a}
\psi_a(t)=
\begin{cases}
-b,&t\le-a,\\
(b/a)t,&|t|\le a,\\
b,&t\ge a,
\end{cases}
\qquad b=\sqrt{a/2},
\end{equation}
which has Brownian RKHS norm one.  Each additional Brownian layer maps sign amplitude $a$ to at least $\sqrt{a/2}$.

\begin{remark}[Depth and non-Hilbertianity]\label{rem:depth}
As $L\to\infty$, the orthogonal profile approaches its universal maximum order $\sqrt n$.  Thus depth makes the full empirical union progressively less representable by a single low-trace covariance body, even though every individual ladder remains an RKHS.
\end{remark}

\begin{table}[t]
\centering
\caption{Representative regimes.  Gaussian richness means tightness of the common-covariance certificate; it does not by itself mean a fast full-class rate.  Radius and input-scale factors are suppressed only in the last column.}
\label{tab:regimes}
\scriptsize
\begin{tabularx}{\textwidth}{@{}>{\raggedright\arraybackslash}p{0.16\textwidth}>{\raggedright\arraybackslash}p{0.19\textwidth}>{\raggedright\arraybackslash}p{0.18\textwidth}>{\raggedright\arraybackslash}p{0.15\textwidth}>{\raggedright\arraybackslash}X@{}}
\toprule
Regime & Assumptions & Profile / defect & Actual complexity scale & Status \\
\midrule
One fixed ladder & one empirical kernel & $\Gamma_L^{\mathrm{dom}}=1$, $\Lambda_L^2=\operatorname{tr}K/S_L$ & at most $n^{-1/2}$ & deterministic upper certificate \\
Finite controlled family & $M$ kernels fixed independently of multipliers & common covariance: $\Gamma_L^{\mathrm{dom}}\le\sqrt M$ & common-envelope $\sqrt{M/n}$; direct union can give $\sqrt{\log M/n}$ & deterministic upper; not claimed sharp \\
$k$ repeated locations & exactly $k$ distinct input values & $\Lambda_L\le\sqrt k$ & $\sqrt{k/n}$ & deterministic upper \\
Hierarchical thresholds & decomposition $(q,H)$; reverse also needs stable persistence & $\Lambda_L\le\sqrt{2q(H+1)}$; $\chi_L\lesssim\sqrt{q(H+1)/\kappa}$ under persistence $\kappa$ & near $n^{-1/2}$ for bounded $q$ and logarithmic $H$ & upper always under decomposition; reverse conditional \\
Ordered thresholds & at most $q$ intervals; reverse also needs stability or integrated witness & $\Lambda_L\lesssim\sqrt q\log n$; $\chi_L\lesssim\sqrt q\log n$ under stated witness conditions & $\sqrt q\log(n)/\sqrt n$ & deterministic or high-probability conditional results \\
Orthogonal / well-conditioned & orthogonal design, or bounded Gram condition number & $\chi_L=O(1)$ and $\Lambda_L^2\asymp n^{1-2^{1-L}}$ & $n^{-2^{-L}}$ & deterministic two-sided; certificate tight but rate slow \\
Spherical random design & $d\gtrsim n+\log(1/\delta)$ & $\chi_L=O(1)$ with probability $1-\delta$ & $n^{-2^{-L}}$ & high-probability two-sided sign-rich phase \\
Balanced categorical support & occupancy control on $K$ Gaussian-rich latent states & $\chi_L=O(1)$ & $K^{1/2-2^{-L}}/\sqrt n$ & high probability \\
Noisy repetitions & explicit matched perturbation and threshold-margin conditions & defect inherited quantitatively & reference rate plus stability error & deterministic conditional transfer \\
\bottomrule
\end{tabularx}
\end{table}

\subsection{Gaussian Richness, Random Designs, and Stability}\label{sec:richness}

\subsubsection{A Brownian stable-threshold Gaussian reverse}

For a finite set family $\mathcal F\subseteq2^{[n]}$, define its Gaussian set width by
\begin{equation}\label{eq:set-gaussian-width}
\omega_{\mathrm G}(\mathcal F)
:=
\E_G
\max_{A\in\mathcal F}
\abs{\sum_{i\in A}G_i},
\end{equation}
with the convention $\omega_{\mathrm G}(\varnothing)=0$.  If $S_L(\bx)>0$, define the stable-threshold Gaussian richness
\begin{equation}\label{eq:W-profile}
\mathfrak W_{L-1}(\bx)
:=
\frac1{S_L(\bx)}
\sup_{t\ge0,\,\delta>0}
\delta\,
\omega_{\mathrm G}
\bigl(
\mathcal F_{L-1}^{\mathrm{st}}(t,\delta)
\bigr)^2.
\end{equation}
If $S_L(\bx)=0$, set $\mathfrak W_{L-1}(\bx):=0$.

\begin{theorem}[Stable-threshold Gaussian reverse]\label{thm:threshold-gaussian-reverse}
For every depth $L\ge2$ and every sample $\bx$,
\begin{equation}\label{eq:W-R-D-chain}
\mathfrak W_{L-1}(\bx)
\le
\mathfrak R_{L-1}(\bx)
\le
\Lambda_L(\bx)^2
\le
2\mathfrak D_{L-1}(\bx),
\end{equation}
where the middle quantities are interpreted as zero when $S_L(\bx)=0$.  Moreover,
\begin{equation}\label{eq:W-lower-first-moment}
W_{G,1}(L,\bx)^2
\ge
S_L(\bx)\mathfrak W_{L-1}(\bx),
\end{equation}
and therefore
\begin{equation}\label{eq:threshold-two-sided-complexity}
\frac{r}{n}
\sqrt{S_L(\bx)\mathfrak W_{L-1}(\bx)}
\le
\Ghat_{\bx}(B_L(r))
\le
\frac{r}{n}\sqrt{\tau_L(\bx)}.
\end{equation}
If $\mathfrak D_{L-1}(\bx)>0$, then the lower bound may be expressed directly in terms of the common-covariance certificate:
\begin{equation}\label{eq:threshold-reverse-tau}
\Ghat_{\bx}(B_L(r))
\ge
\frac{r\sqrt{\tau_L(\bx)}}{n}
\sqrt{
\frac{\mathfrak W_{L-1}(\bx)}
{2\mathfrak D_{L-1}(\bx)}
}.
\end{equation}
If $\mathfrak D_{L-1}(\bx)=0$, then $\tau_L(\bx)=0$ and all quantities in \eqref{eq:threshold-two-sided-complexity} vanish.
\end{theorem}

\begin{corollary}[Orthogonal constant-factor tightness]\label{cor:minimax-orthogonal}
For the orthogonal design $x_i=e_i$ and every fixed depth $L\ge2$,
\begin{equation}\label{eq:orthogonal-W-equivalence}
\mathfrak W_{L-1}(e_1,\ldots,e_n)
\asymp
\Lambda_L(e_1,\ldots,e_n)^2
\asymp
n^{1-2^{-(L-1)}}.
\end{equation}
Consequently, the common-covariance Gaussian upper bound is sharp up to universal constants on the orthogonal design.
\end{corollary}

The following condition gives a convenient quantitative reverse in structured geometries.  We say that the previous layer contains a $(\kappa,t,\delta)$ \emph{macroscopic stable threshold} if there is a nonempty set
$A\in\mathcal F_{L-1}^{\mathrm{st}}(t,\delta)$ such that
\begin{equation}\label{eq:macroscopic-stable-threshold}
\delta\,|A|
\ge
\kappa S_L(\bx).
\end{equation}

\begin{corollary}[Stable hierarchical Gaussian optimality]\label{cor:minimax-hierarchical}
Assume the hypotheses of \Cref{cor:hierarchical-threshold}, and assume in addition that a macroscopic stable threshold exists with parameter $\kappa>0$.  Then
\begin{equation}\label{eq:hierarchical-minimax-reverse}
\Ghat_{\bx}(B_L(r))
\ge
\sqrt{\frac{\kappa}{\pi q(H+1)}}
\frac{r\sqrt{\tau_L(\bx)}}{n}.
\end{equation}
Thus, when $q$ and $\kappa^{-1}$ are bounded and $H=O(\log n)$, the common-covariance upper certificate is sharp up to a factor $O(\sqrt{\log n})$.
\end{corollary}

\begin{corollary}[Stable ordered Gaussian optimality]\label{cor:minimax-ordered}
Assume the hypotheses of \Cref{cor:ordered-threshold}, and assume in addition that a macroscopic stable threshold exists with parameter $\kappa>0$.  Put $H_n=\lceil\log_2(2n)\rceil$.  Then
\begin{equation}\label{eq:ordered-minimax-reverse}
\Ghat_{\bx}(B_L(r))
\ge
\sqrt{\frac{\kappa}{2\pi q}}
\frac1{H_n}
\frac{r\sqrt{\tau_L(\bx)}}{n}.
\end{equation}
Thus, when $q$ and $\kappa^{-1}$ are bounded, the common-covariance upper certificate is sharp up to a factor $O(\log n)$.
\end{corollary}

The additional stability requirement in the preceding two corollaries is essential to the proof.  A small decomposition dictionary is an upper-geometric condition; it does not by itself ensure that any large threshold set persists over a threshold interval of nonnegligible length.  The quantity $\mathfrak W_{L-1}$ records precisely the stable Gaussian richness needed for a reverse inequality.

\subsubsection{Random-design Gaussian richness}\label{sec:random-design}

The conditioning events used below are controlled by standard nonasymptotic singular-value and matrix-concentration tools \cite{mendelsonpajor2006singular,vershynin2012nonasymptotic,tropp2015matrix}.  The BKL-specific step is the transfer of those events through the recursive covariance and threshold geometry.

The comparison $\mathfrak W_{L-1}(\bx)\asymp\Lambda_L(\bx)^2$ is a concrete sufficient condition for a bounded Gaussian defect, but it is not logically equivalent to boundedness of $\chi_L(\bx)$: $\mathfrak W_{L-1}$ is one particular stable-threshold certificate, whereas $\chi_L$ is the exact multiplier-to-Gaussian loss.  Indeed, \eqref{eq:W-lower-first-moment} and $\tau_L=S_L\Lambda_L^2$ imply
\begin{equation}\label{eq:defect-from-W}
\chi_L(\bx)^2
\le
\frac{\Lambda_L(\bx)^2}{\mathfrak W_{L-1}(\bx)}
\end{equation}
whenever $\mathfrak W_{L-1}(\bx)>0$.  We now identify two probabilistic mechanisms that make the right-hand side bounded or polylogarithmic.

For covariance families $\mathscr A,\mathscr B\subseteq\psd^n$ and $c>0$, write
\begin{equation}\label{eq:family-domination}
\mathscr A\preceq_c\mathscr B
\quad\Longleftrightarrow\quad
\text{for every }A\in\mathscr A\text{ there exists }B\in\mathscr B\text{ such that }A\preceq cB.
\end{equation}
For a sample $\bx=(x_1,\ldots,x_n)$, define its linear Gram matrix by
\begin{equation}\label{eq:linear-gram-random}
G_{\bx}:=\bigl[x_i^\top x_j\bigr]_{i,j=1}^n.
\end{equation}
Write $\sK_1^\circ(\bx):=\set{G_{\bx}}$.  For $\ell\ge2$, let $\sK_\ell^\circ(\bx)$ denote the family of actual empirical depth-$\ell$ Gram matrices before the closure in \eqref{eq:kernel-family}, so $\sK_\ell(\bx)=\cl\sK_\ell^\circ(\bx)$.

\begin{proposition}[Brownian propagation of covariance distortion]\label{prop:brownian-distortion}
Let $\bx$ and $\by$ be samples of the same cardinality, let $\ell\ge1$, and let $c>0$.  If
\begin{align}
\sK_\ell^\circ(\bx)\preceq_c\sK_\ell^\circ(\by),
\notag
\end{align}
then
\begin{align}
\sK_{\ell+1}^\circ(\bx)
&\preceq_{\sqrt c}
\sK_{\ell+1}^\circ(\by),
\notag\\
\sK_{\ell+1}(\bx)
&\preceq_{\sqrt c}
\sK_{\ell+1}(\by).
\label{eq:brownian-distortion-propagation}
\end{align}
Thus every additional Brownian layer takes the square root of the preceding Loewner distortion.
\end{proposition}

Let $\be_n:=(e_1,\ldots,e_n)$ denote the orthogonal reference design, and set
\begin{equation}\label{eq:theta-random-design}
\vartheta_L:=2^{-(L-1)}.
\end{equation}

\begin{theorem}[Well-conditioned designs are Gaussian rich]\label{thm:well-conditioned-richness}
Let $L\ge2$, and assume that
\begin{equation}\label{eq:well-conditioned-gram}
0<\lambda_- I_n
\preceq
G_{\bx}
\preceq
\lambda_+I_n.
\end{equation}
Put $\kappa_{\bx}:=\lambda_+/\lambda_-$.  Then
\begin{align}
\lambda_-^{\vartheta_L}\tau_L(\be_n)
&\le
\tau_L(\bx)
\le
\lambda_+^{\vartheta_L}\tau_L(\be_n),
\label{eq:well-conditioned-tau}\\
\lambda_-^{\vartheta_L/2}h_{L,\be_n}(z)
&\le
h_{L,\bx}(z)
\le
\lambda_+^{\vartheta_L/2}h_{L,\be_n}(z),
\qquad z\in\R^n.
\label{eq:well-conditioned-support}
\end{align}
There are universal constants $c,C>0$ such that
\begin{equation}\label{eq:well-conditioned-W-Lambda}
c\,\kappa_{\bx}^{-\vartheta_L}n^{1-\vartheta_L}
\le
\mathfrak W_{L-1}(\bx)
\le
\Lambda_L(\bx)^2
\le
C\,\kappa_{\bx}^{\vartheta_L}n^{1-\vartheta_L}.
\end{equation}
Moreover,
\begin{equation}\label{eq:well-conditioned-defect}
\chi_L(\bx)
\le
C\,\kappa_{\bx}^{\vartheta_L/2}.
\end{equation}
In particular, a uniformly bounded condition number gives
\begin{align}
\mathfrak W_{L-1}(\bx)
\asymp
\Lambda_L(\bx)^2
\asymp
n^{1-2^{-(L-1)}},
\qquad
\chi_L(\bx)=O(1).
\notag
\end{align}
The empirical Gaussian complexity satisfies
\begin{equation}\label{eq:well-conditioned-complexity}
c\,r\lambda_-^{\vartheta_L/2}n^{-2^{-L}}
\le
\Ghat_{\bx}(B_L(r))
\le
C\,r\lambda_+^{\vartheta_L/2}n^{-2^{-L}}.
\end{equation}
\end{theorem}

\begin{corollary}[Independent spherical random design]\label{cor:spherical-random-richness}
Let $\delta\in(0,1)$ and let $X_1,\ldots,X_n$ be independent and uniformly distributed on the Euclidean sphere $S^{d-1}$.  There are universal constants $c,C>0$ such that, if
\begin{equation}\label{eq:spherical-dimension-condition}
d
\ge
C\bigl(n+\log\frac2\delta\bigr),
\end{equation}
then, with probability at least $1-\delta$, simultaneously for every depth $L\ge2$,
\begin{align}
\mathfrak W_{L-1}(X_{1:n})
&\asymp
\Lambda_L(X_{1:n})^2
\asymp
n^{1-2^{-(L-1)}},
\label{eq:spherical-W-Lambda}\\
\chi_L(X_{1:n})
&\le C,
\label{eq:spherical-defect}\\
\Ghat_{X_{1:n}}(B_L(r))
&\asymp
r n^{-2^{-L}}.
\label{eq:spherical-complexity}
\end{align}
All implicit constants are universal.  Thus high-dimensional isotropic random designs are Gaussian rich with high probability, although the resulting sharp full-class rate remains the sign-rich rate $n^{-2^{-L}}$.
\end{corollary}

We next consider a complementary model in which the sample is drawn from a finite latent support.  Let
\begin{align}
\boldsymbol\xi=(\xi_1,\ldots,\xi_K),
\qquad
\Pp(X=\xi_j)=p_j>0,
\notag
\end{align}
where the support points are distinct.  For an iid sample $X_1,\ldots,X_n$, let
\begin{equation}\label{eq:occupancy-counts}
N_j:=\abs{\set{i:X_i=\xi_j}},
\qquad
N_-:=\min_{j\in[K]}N_j,
\qquad
N_+:=\max_{j\in[K]}N_j.
\end{equation}

\begin{theorem}[Random-repetition transfer]\label{thm:random-repetition-transfer}
On the event $N_->0$,
\begin{equation}\label{eq:replication-tau}
N_-\tau_L(\boldsymbol\xi)
\le
\tau_L(X_{1:n})
\le
N_+\tau_L(\boldsymbol\xi).
\end{equation}
If, in addition, $S_L(\boldsymbol\xi)>0$, then
\begin{align}
\mathfrak W_{L-1}(X_{1:n})
&\ge
\frac{N_-S_L(\boldsymbol\xi)}{S_L(X_{1:n})}
\mathfrak W_{L-1}(\boldsymbol\xi),
\label{eq:replication-W}\\
\Lambda_L(X_{1:n})^2
&\le
\frac{N_+S_L(\boldsymbol\xi)}{S_L(X_{1:n})}
\Lambda_L(\boldsymbol\xi)^2.
\label{eq:replication-Lambda}
\end{align}
Consequently,
\begin{equation}\label{eq:replication-ratio}
\frac{\mathfrak W_{L-1}(X_{1:n})}{\Lambda_L(X_{1:n})^2}
\ge
\frac{N_-}{N_+}
\frac{\mathfrak W_{L-1}(\boldsymbol\xi)}{\Lambda_L(\boldsymbol\xi)^2}.
\end{equation}
Whenever $\mathfrak W_{L-1}(\boldsymbol\xi)>0$,
\begin{equation}\label{eq:replication-defect}
\chi_L(X_{1:n})^2
\le
\frac{N_+}{N_-}
\frac{\Lambda_L(\boldsymbol\xi)^2}{\mathfrak W_{L-1}(\boldsymbol\xi)}.
\end{equation}
Thus, in the nondegenerate case, repetitions preserve bounded or polylogarithmic Gaussian defect whenever the occupancy counts are balanced.  The theorem is deterministic conditional on the realized counts; probability enters only through the occupancy corollaries below.
\end{theorem}

Set
\begin{align}
p_{\min}:=\min_{j\in[K]}p_j,
\qquad
p_{\max}:=\max_{j\in[K]}p_j.
\notag
\end{align}

\begin{corollary}[Finite-support random-design Gaussian richness]\label{cor:finite-support-random-richness}
Let $\delta\in(0,1)$ and assume that
\begin{equation}\label{eq:occupancy-sample-size}
np_{\min}
\ge
12\log\frac{2K}{\delta}.
\end{equation}
Then, with probability at least $1-\delta$,
\begin{equation}\label{eq:occupancy-ratio}
\frac{N_-}{N_+}
\ge
\frac{p_{\min}}{3p_{\max}}.
\end{equation}
If, in addition, $S_L(\boldsymbol\xi)>0$ and
\begin{equation}\label{eq:base-richness-assumption}
\mathfrak W_{L-1}(\boldsymbol\xi)
\ge
c_0\Lambda_L(\boldsymbol\xi)^2
\end{equation}
for some $c_0>0$, then on the same event
\begin{align}
\mathfrak W_{L-1}(X_{1:n})
&\ge
\frac{c_0p_{\min}}{3p_{\max}}
\Lambda_L(X_{1:n})^2,
\label{eq:finite-support-W-Lambda}\\
\chi_L(X_{1:n})
&\le
\sqrt{\frac{3p_{\max}}{c_0p_{\min}}}.
\label{eq:finite-support-defect}
\end{align}
\end{corollary}

\begin{corollary}[Balanced categorical orthogonal model]\label{cor:categorical-orthogonal}
Let $\delta\in(0,1)$ and let $\boldsymbol\xi=(e_1,\ldots,e_K)\subset\R^K$.  Assume that
\begin{equation}\label{eq:balanced-categorical-probabilities}
\frac{c_-}{K}
\le
p_j
\le
\frac{c_+}{K}
\qquad\text{for every }j\in[K].
\end{equation}
Here $0<c_-\le c_+<\infty$.  There is a constant $C>0$, depending only on $c_-$, such that, if
\begin{equation}\label{eq:categorical-sample-size}
n
\ge
C K\log\frac{2K}{\delta},
\end{equation}
then, with probability at least $1-\delta$, simultaneously for every $L\ge2$,
\begin{align}
\mathfrak W_{L-1}(X_{1:n})
&\asymp_{c_-,c_+}
\Lambda_L(X_{1:n})^2
\asymp_{c_-,c_+}
K^{1-2^{-(L-1)}},
\label{eq:categorical-W-Lambda}\\
\chi_L(X_{1:n})
&=O_{c_-,c_+}(1),
\label{eq:categorical-defect}\\
\Ghat_{X_{1:n}}(B_L(r))
&\asymp_{c_-,c_+}
\frac{rK^{1/2-2^{-L}}}{\sqrt n}.
\label{eq:categorical-complexity}
\end{align}
Thus fixed $K$ gives the usual $n^{-1/2}$ sample-size rate, $K=\operatorname{polylog}(n)$ gives a near-parametric rate up to polylogarithms, and the largest iid support sizes compatible with the occupancy condition approach the orthogonal scale up to logarithmic factors.  Exact $K=n$ orthogonal sampling is obtained by sampling without replacement, not by the iid occupancy statement above.
\end{corollary}

\begin{corollary}[Random hierarchical latent-support model]\label{cor:random-hierarchical-support}
Let $\delta\in(0,1)$ and assume the base support $\boldsymbol\xi$ satisfies the hierarchical threshold hypothesis of \Cref{cor:hierarchical-threshold} with parameters $q,H$, and assume that it has a macroscopic stable threshold with parameter $\kappa_0>0$ in the sense of \eqref{eq:macroscopic-stable-threshold}.  Under \eqref{eq:occupancy-sample-size}, with probability at least $1-\delta$,
\begin{align}
\mathfrak W_{L-1}(X_{1:n})
&\ge
\frac{p_{\min}}{3p_{\max}}
\frac{\kappa_0}{\pi q(H+1)}
\Lambda_L(X_{1:n})^2,
\label{eq:random-hierarchical-W}\\
\chi_L(X_{1:n})
&\le
\sqrt{
\frac{3\pi p_{\max}q(H+1)}{p_{\min}\kappa_0}
}.
\label{eq:random-hierarchical-defect}
\end{align}
In particular, balanced latent supports with bounded $q$, bounded $\kappa_0^{-1}$, and $H=O(\log K)$ have empirical Gaussian defect $O(\bigl(\log K\bigr)^{1/2})$ with high probability.
\end{corollary}

\subsubsection{Perturbation stability and integrated latent-parameter profiles}\label{sec:stability-integrated}

The exact repetition theorem leaves two natural questions.  First, one may observe small perturbations of repeated latent inputs rather than exact repetitions.  Second, for a one-dimensional or ordered latent parameter with a non-atomic law, every individual stable-threshold interval may be short even though the aggregate threshold geometry remains simple.  We address both issues without changing the full adaptive BKL ball.

When two samples occur simultaneously, we write
$\mathcal F_{L-1,\bx}^{\mathrm{st}}(t,\delta)$ for the family in \eqref{eq:stable-family} constructed from $\bx$, and similarly for $\by$.  Put
\begin{equation}\label{eq:matched-perturbation-radius}
\eps(\bx,\by)
:=
\max_{1\le i\le n}\norm{x_i-y_i}_2.
\end{equation}

\begin{theorem}[Matched-sample stability of the empirical BKL geometry]\label{thm:matched-sample-stability}
Let $L\ge2$, let $\bx,\by\in\cX^n$, and put
$\varepsilon:=\eps(\bx,\by)$ and $\vartheta_L=2^{-(L-1)}$.  Then, for every $z\in\R^n$,
\begin{equation}\label{eq:matched-support-stability}
\abs{h_{L,\bx}(z)-h_{L,\by}(z)}
\le
\sqrt n\,\varepsilon^{\vartheta_L}\norm z_2.
\end{equation}
For $q\in\{1,2\}$,
\begin{equation}\label{eq:matched-gaussian-stability}
\abs{W_{G,q}(L,\bx)-W_{G,q}(L,\by)}
\le
n\varepsilon^{\vartheta_L}.
\end{equation}
Moreover,
\begin{align}
\abs{\sqrt{\tau_L(\bx)}-\sqrt{\tau_L(\by)}}
&\le
n\varepsilon^{\vartheta_L},
\label{eq:matched-tau-stability}\\
\abs{S_L(\bx)-S_L(\by)}
&\le
n\varepsilon^{2\vartheta_L}.
\label{eq:matched-S-stability}
\end{align}
If $S_L(\by)>n\varepsilon^{2\vartheta_L}$, then
\begin{equation}\label{eq:matched-profile-stability}
\Lambda_L(\bx)^2
\le
\frac{\bigl(\sqrt{\tau_L(\by)}+n\varepsilon^{\vartheta_L}\bigr)^2}
{S_L(\by)-n\varepsilon^{2\vartheta_L}}.
\end{equation}
The symmetric estimate follows by interchanging $\bx$ and $\by$.

Assume that $W_{G,2}(L,\by)>0$ and that, for some $0<\rho<1$,
\begin{equation}\label{eq:matched-defect-condition}
n\varepsilon^{\vartheta_L}
\le
\rho W_{G,2}(L,\by).
\end{equation}
Then
\begin{equation}\label{eq:matched-defect-stability}
\chi_L(\bx)
\le
\frac{\chi_L(\by)+\rho}{1-\rho}.
\end{equation}
Thus bounded or polylogarithmic Gaussian defect is stable whenever the perturbation is small relative to the Gaussian support scale of the reference sample.
\end{theorem}

\begin{theorem}[Perturbation stability of stable-threshold reverses]\label{thm:threshold-perturbation-stability}
Use the hypotheses and notation of \Cref{thm:matched-sample-stability}, and set
\begin{align}
\eta:=\varepsilon^{2\vartheta_L}.
\notag
\end{align}
Let $t\ge0$, $\delta>0$, and $0\le\zeta<1$ satisfy
\begin{equation}\label{eq:near-optimal-stable-certificate}
\frac{\delta\,\omega_{\mathrm G}\bigl(\mathcal F_{L-1,\by}^{\mathrm{st}}(t,\delta)\bigr)^2}
{S_L(\by)}
\ge
(1-\zeta)\mathfrak W_{L-1}(\by).
\end{equation}
Then
\begin{equation}\label{eq:W-perturbation-stability}
\mathfrak W_{L-1}(\bx)
\ge
(1-\zeta)
\frac{\max\{\delta-2\eta,0\}}{\delta}
\frac{S_L(\by)}{S_L(\bx)}
\mathfrak W_{L-1}(\by).
\end{equation}
Suppose additionally that, for some $0<\rho<1$ and $c>0$,
\begin{align}
n\varepsilon^{\vartheta_L}
&\le
\rho\sqrt{\tau_L(\by)},
\notag\\
n\varepsilon^{2\vartheta_L}
&\le
\rho S_L(\by),
\notag\\
2\varepsilon^{2\vartheta_L}
&\le
\rho\delta,
\notag\\
\mathfrak W_{L-1}(\by)
&\ge
c\Lambda_L(\by)^2.
\notag
\end{align}
Then
\begin{equation}\label{eq:W-Lambda-perturbation-stability}
\mathfrak W_{L-1}(\bx)
\ge
c(1-\zeta)
\frac{(1-\rho)^2}{(1+\rho)^3}
\Lambda_L(\bx)^2.
\end{equation}
The comparison supplied by the exact repetition model therefore persists under perturbations smaller than both its covariance scale and its stable-threshold margin.
\end{theorem}

\begin{corollary}[Noisy random-repetition transfer]\label{cor:noisy-repetition-transfer}
Let $C_1,\ldots,C_n\in[K]$, put $Y_i:=\xi_{C_i}$, and suppose that the observed inputs satisfy
\begin{align}
\max_{1\le i\le n}\norm{X_i-Y_i}_2\le\varepsilon.
\notag
\end{align}
On the event $N_->0$, define
\begin{equation}\label{eq:noisy-repetition-reference-defect}
C_{\mathrm{rep}}^2
:=
\frac{N_+}{N_-}
\frac{\Lambda_L(\boldsymbol\xi)^2}
{\mathfrak W_{L-1}(\boldsymbol\xi)},
\end{equation}
whenever the denominator is positive.  If
\begin{equation}\label{eq:noisy-repetition-perturbation-condition}
n\varepsilon^{\vartheta_L}
\le
\rho W_{G,2}(L,Y_{1:n})
\end{equation}
for some $0<\rho<1$, then
\begin{equation}\label{eq:noisy-repetition-defect}
\chi_L(X_{1:n})
\le
\frac{C_{\mathrm{rep}}+\rho}{1-\rho}.
\end{equation}
In particular, balanced occupancy and a Gaussian-rich latent support give bounded or polylogarithmic defect under every perturbation satisfying \eqref{eq:noisy-repetition-perturbation-condition}.  If the exact repeated sample also has a near-optimal stable certificate whose interval satisfies the margin hypotheses of \Cref{thm:threshold-perturbation-stability}, then the direct comparison
$\mathfrak W_{L-1}(X_{1:n})\gtrsim\Lambda_L(X_{1:n})^2$
persists with the explicit constant in \eqref{eq:W-Lambda-perturbation-stability}.
\end{corollary}

\begin{remark}[Depth dependence of the sufficient perturbation scale]\label{rem:perturbation-scale}
The stability condition is quantitative and becomes restrictive with depth.  For a balanced repeated-support reference model with fixed latent support, $W_{G,2}(L,Y_{1:n})$ is of order $\sqrt n$ up to constants depending on the support and depth.  The condition
$n\varepsilon^{\vartheta_L}\lesssim\sqrt n$ therefore requires approximately
\begin{align}
\varepsilon
\lesssim
n^{-1/(2\vartheta_L)}
=
n^{-2^{L-2}}.
\notag
\end{align}
Thus the sufficient scales are $n^{-1/2}$ for $L=2$, $n^{-1}$ for $L=3$, and $n^{-2}$ for $L=4$.  The theorem proves stability under explicit covariance, Gaussian-width, and threshold-margin conditions; it does not claim robustness to generic noise of fixed size.
\end{remark}

The second part of the subsection replaces one stable threshold interval by an integral over all threshold levels.  For $a\in\sS_{L-1}(\bx)$ and $s\ge0$, retain the vectors $v_s^+(a)$ and $v_s^-(a)$ from \eqref{eq:threshold-vectors}.  A \emph{threshold covariance field} is a Borel measurable map
$s\mapsto H_s\in\psd^n$ such that, for every $a\in\sS_{L-1}(\bx)$ and every $\sigma\in\{+,-\}$,
\begin{equation}\label{eq:threshold-field-domination}
H_s
\succeq
v_s^\sigma(a)v_s^\sigma(a)^\top
\qquad
\text{for Lebesgue-almost every }s\ge0.
\end{equation}
When $S_L(\bx)>0$, define the integrated threshold-domination profile
\begin{equation}\label{eq:integrated-threshold-domination}
\overline{\mathfrak T}_{L-1}(\bx)
:=
\frac1{S_L(\bx)}
\inf_{(H_s)}
\int_0^\infty\tr H_s\,ds,
\end{equation}
where the infimum ranges over all threshold covariance fields with finite trace integral.  Such fields always exist: if $p_{\max}:=\max_i d_{L,i}$, then
$H_s=nI_n$ for $0\le s\le p_{\max}$ and $H_s=0$ otherwise is admissible.  If $S_L(\bx)=0$, set $\overline{\mathfrak T}_{L-1}(\bx):=0$.

Define the integrated threshold-mass richness by
\begin{equation}\label{eq:integrated-threshold-mass}
\mathfrak A_{L-1}(\bx)
:=
\frac1{S_L(\bx)}
\sup_{a\in\sS_{L-1}(\bx)}\norm a_1
\end{equation}
when $S_L(\bx)>0$, and set it to zero otherwise.  The scalar layer-cake identity gives
\begin{align}
\norm a_1
=
\int_0^\infty
\bigl(
\norm{v_s^+(a)}_2^2
+
\norm{v_s^-(a)}_2^2
\bigr)ds,
\notag
\end{align}
so $\mathfrak A_{L-1}$ accumulates threshold mass across all levels rather than requiring persistence on one interval.

\begin{theorem}[Integrated threshold sandwich]\label{thm:integrated-threshold-sandwich}
For every depth $L\ge2$ and every sample $\bx$,
\begin{equation}\label{eq:integrated-threshold-chain}
\mathfrak A_{L-1}(\bx)
\le
\frac{W_{G,2}(L,\bx)^2}{S_L(\bx)}
\le
\Lambda_L(\bx)^2
\le
2\overline{\mathfrak T}_{L-1}(\bx)
\le
2\mathfrak D_{L-1}(\bx),
\end{equation}
with the convention that every ratio is zero when $S_L(\bx)=0$.  If $\mathfrak A_{L-1}(\bx)>0$, then
\begin{equation}\label{eq:integrated-defect-bound}
\chi_L(\bx)^2
\le
\frac{2\overline{\mathfrak T}_{L-1}(\bx)}
{\mathfrak A_{L-1}(\bx)}.
\end{equation}
Thus the integrated profile sharpens the global decomposition certificate and remains informative even when no single stable-threshold interval has macroscopic length.
\end{theorem}

We now specialize the integrated profile to low-dimensional ordered trace models.  Let $\mathcal I\subset\R$ be an interval, let $\gamma:\mathcal I\to\cX$, and write $x_i=\gamma(t_i)$ with
$t_1<\cdots<t_n$.  Let $q:[0,\infty)\to\N$ be measurable and put
\begin{equation}\label{eq:integrated-Q}
Q(r):=\int_0^r q(s)\,ds.
\end{equation}
Assume that, for every unit-sphere element $u$ of every admissible level-$(L-1)$ RKHS and every $s>0$, each signed superlevel set
\begin{align}
\set{t\in\mathcal I:u(\gamma(t))\ge s},
\qquad
\set{t\in\mathcal I:u(\gamma(t))\le-s}
\notag
\end{align}
is a union of at most $q(s)$ intervals.  This condition permits $q(s)$ to grow near zero; only its amplitude-weighted integral will enter the bound.

\begin{theorem}[Integrated ordered and laminar trace models]\label{thm:integrated-ordered-laminar}
Under the preceding ordered-trace assumption, put
\begin{align}
p_i:=d_{L,i}=\norm{x_i}_2^{\,2\vartheta_L},
\qquad
H_n:=\bigl\lceil\log_2(2n)\bigr\rceil.
\notag
\end{align}
Then
\begin{equation}\label{eq:ordered-integrated-profile}
\overline{\mathfrak T}_{L-1}(\bx)
\le
2H_n^2
\frac{\sum_{i=1}^n Q(p_i)}{\sum_{i=1}^n p_i}.
\end{equation}
For every fixed unit-sphere element $u_\star$ of an admissible level-$(L-1)$ RKHS satisfying
$\sum_i|u_\star(x_i)|>0$,
\begin{equation}\label{eq:ordered-integrated-defect}
\chi_L(\bx)
\le
2H_n
\bigl(
\frac{\sum_{i=1}^n Q(p_i)}
{\sum_{i=1}^n|u_\star(x_i)|}
\bigr)^{1/2}.
\end{equation}

Suppose instead that, at every threshold $s$, all signed superlevel sets are disjoint unions of at most $q(s)$ nodes of one fixed rooted laminar tree on the sample index set $[n]$, of height $H$.  Then
\begin{align}
\overline{\mathfrak T}_{L-1}(\bx)
&\le
(H+1)
\frac{\sum_{i=1}^n Q(p_i)}{\sum_{i=1}^n p_i},
\label{eq:laminar-integrated-profile}\\
\chi_L(\bx)
&\le
\bigl(
2(H+1)
\frac{\sum_{i=1}^n Q(p_i)}
{\sum_{i=1}^n|u_\star(x_i)|}
\bigr)^{1/2}.
\label{eq:laminar-integrated-defect}
\end{align}
\end{theorem}

\subsubsection{Concrete ordered models at depth two}

The ordered hypotheses concern every previous-layer unit trace and are therefore strong for unrestricted depth.  At depth two they hold in elementary non-vacuous models.

\begin{examplex}[One-dimensional ray at depth two]\label{ex:depth-two-ray}
Let $v\in\R^d\setminus\{0\}$, let $\mathcal I\subset\R$ be compact, and set $\gamma(t)=tv$.  For $L=2$, every previous-layer unit trace has the form
\begin{align}
u(\gamma(t))=t\langle w,v\rangle,
\qquad w\in\operatorname{span}(\cX),
\qquad \norm w_2=1.
\notag
\end{align}
Each positive or negative superlevel set is empty, all of $\mathcal I$, or one interval.  Hence the ordered theorem applies with $q(s)=1$ and $Q(r)=r$.  Choose $w_\star=v/\norm v_2$.  Then
\begin{align}
|u_\star(\gamma(t_i))|
=\norm v_2|t_i|
=\norm{\gamma(t_i)}_2
=p_i,
\notag
\end{align}
so \eqref{eq:ordered-integrated-defect} gives
\begin{align}
\chi_2(\bx)
\le
2\bigl\lceil\log_2(2n)\bigr\rceil
\notag
\end{align}
whenever the sample is not identically zero.
\end{examplex}

\begin{examplex}[Polynomial feature curve at depth two]\label{ex:depth-two-polynomial}
Let $\gamma:\mathcal I\to\R^d$ be a polynomial map of degree at most $p\ge1$.  For a linear functional $u(x)=\langle w,x\rangle$, the scalar function $P_w(t):=u(\gamma(t))$ is a real polynomial of degree at most $p$.  For any $s>0$, the boundary of $\{t:P_w(t)\ge s\}$ is contained in the real zero set of $P_w-s$.  If $P_w-s$ is not identically zero, it has at most $p$ real roots.  If it is identically zero, the superlevel set is the whole interval and has one component.  Between consecutive roots the sign is constant, and multiple roots cannot increase the number of nonnegative components.  Therefore each superlevel set has at most
\begin{align}
q_p:=\bigl\lfloor\frac p2\bigr\rfloor+1
\notag
\end{align}
interval components; the same bound holds for negative superlevel sets by applying it to $-P_w-s$.  Thus the depth-two ordered theorem applies with $q(s)=q_p$ and $Q(r)=q_pr$.  If a unit linear witness $u_\star$ satisfies
\begin{align}
\sum_{i=1}^n|u_\star(x_i)|
\ge
c_\star\sum_{i=1}^n\norm{x_i}_2
\notag
\end{align}
for some $c_\star>0$, then
\begin{align}
\chi_2(\bx)
\le
2\bigl\lceil\log_2(2n)\bigr\rceil
\sqrt{\frac{q_p}{c_\star}}.
\notag
\end{align}
No claim is made that this bounded-oscillation property propagates automatically through the complete unrestricted Brownian recursion for $L>2$.
\end{examplex}

\begin{corollary}[Non-atomic latent-parameter ordered random design]\label{cor:non-atomic-integrated-richness}
Let $\delta\in(0,1)$.  Let $T$ have a non-atomic distribution on $\mathcal I$, let $\gamma:\mathcal I\to\cX$ be Borel measurable, put $X=\gamma(T)$, and assume the ordered or laminar trace hypothesis of \Cref{thm:integrated-ordered-laminar}.  Fix a deterministic admissible unit-sphere witness $u_\star$, and define
\begin{align}
U:=|u_\star(X)|,
\qquad
V:=Q\bigl(\norm X_2^{\,2\vartheta_L}\bigr).
\notag
\end{align}
Assume that $0\le U\le B_U$ and $0\le V\le B_V$ almost surely, and put
\begin{align}
m_U:=\E U>0,
\qquad
m_V:=\E V<\infty.
\notag
\end{align}
The non-atomicity assumption concerns the latent parameter $T$ and guarantees almost-surely distinct sample parameters.  It does not imply that the pushforward law of $X=\gamma(T)$ is non-atomic when $\gamma$ is non-injective.
Define
\begin{align}
\epsilon_U
&:=
B_U\sqrt{\frac{\log(4/\delta)}{2n}},
\notag\\
\epsilon_V
&:=
B_V\sqrt{\frac{\log(4/\delta)}{2n}}.
\notag
\end{align}
With probability at least $1-\delta$, if $m_U>\epsilon_U$, then
\begin{equation}\label{eq:non-atomic-general-defect}
\chi_L(X_{1:n})
\le
\bigl(
2C_{\mathrm{geom}}
\frac{m_V+\epsilon_V}{m_U-\epsilon_U}
\bigr)^{1/2},
\end{equation}
where
\begin{align}
C_{\mathrm{geom}}
:=
\begin{cases}
2\bigl\lceil\log_2(2n)\bigr\rceil^2,
&\text{for ordered interval traces},\\
H+1,
&\text{for a laminar tree of height }H.
\end{cases}
\notag
\end{align}
If additionally $\epsilon_U\le m_U/2$ and $\epsilon_V\le m_V/2$, then
\begin{equation}\label{eq:non-atomic-simplified-defect}
\chi_L(X_{1:n})
\le
\sqrt{6C_{\mathrm{geom}}\frac{m_V}{m_U}}.
\end{equation}
In particular, when $m_V/m_U=O(1)$, uniformly bounded threshold oscillation gives $\chi_L=O(\log n)$ for ordered models with a non-atomic latent parameter and $\chi_L=O(\sqrt H)$ for laminar models.  More generally, $q(s)$ may diverge as $s\downarrow0$ whenever
$\E Q(\norm X_2^{2\vartheta_L})<\infty$ and the witness mean $m_U$ remains nondegenerate.
\end{corollary}

The preceding results provide explicit sufficient stability guarantees for matched noisy clusters and replace one shrinking stable interval by an integrated threshold profile for structured low-dimensional latent-parameter models.  The remaining probabilistic issue concerns the intrinsic amplitude-weighted oscillation of the complete unrestricted recursive Brownian trace family, without an ordered finite-turn or laminar hypothesis.

\section{Common Approximation Spaces and Learning Consequences}\label{sec:approximation}

The covariance profile contains more approximation information than its trace alone.  A common covariance matrix supplies one sample-dependent trace subspace that is valid simultaneously for all admissible ladders, and its eigenvalue decay measures the dimension required to approximate the complete adaptive BKL trace body.  This question belongs to the classical width and reduced-basis tradition \cite{pinkus1985nwidths,prudhomme2002reduced,binev2011greedy,devore2013greedy,kriegullrich2026sampling}.  Nystr{\"o}m and leverage-score methods provide complementary data-dependent kernel subspaces \cite{williams2001nystrom,dellavecchia2024nystrom,chatalic2025leverage}; the distinction here is that one space must control the entire adaptive BKL union.  This section makes that statement exact.

For a linear subspace $V\subseteq\R^n$, let $P_V$ denote the Euclidean orthogonal projection onto $V$, and write
\begin{align}
\norm{a}_n:=\frac{1}{\sqrt n}\norm{a}_2,
\qquad
\operatorname{dist}_n(a,V):=\inf_{v\in V}\norm{a-v}_n.
\notag
\end{align}

\begin{definition}[Empirical Kolmogorov width and rank-$m$ covariance profile]\label{def:kolmogorov-width}
For $m\in\{0,\ldots,n\}$ and $r>0$, define
\begin{equation}\label{eq:kolmogorov-width-def}
d_{m,L}(\bx;r)
:=
\inf_{\substack{V\subseteq\R^n\\ \dim V\le m}}
\sup_{f\in B_L(r)}
\operatorname{dist}_n\!\bigl((f(x_i))_{i=1}^n,V\bigr).
\end{equation}
For $m\in\{0,\ldots,n-1\}$, define
\begin{equation}\label{eq:omega-def}
\omega_{m,L}(\bx)^2
:=
\inf_{\substack{V\subseteq\R^n\\ \dim V\le m}}
\sup_{K\in\sK_L(\bx)}
\lambda_{\max}\!\bigl(P_{V^\perp}KP_{V^\perp}\bigr).
\end{equation}
\end{definition}

The order of the quantifiers in \eqref{eq:kolmogorov-width-def} is important: one subspace $V$ must approximate the traces of \emph{all} functions in $B_L(r)$.  The space may depend on the sample and on $L$, but not on the target function and not on which ladder realizes it.

\begin{theorem}[Exact empirical widths and spectral common-covariance approximation]\label{thm:kolmogorov-width}
Let $L\ge2$, $r>0$, and $m\in\{0,\ldots,n-1\}$.  Then
\begin{equation}\label{eq:width-exact}
d_{m,L}(\bx;r)^2
 =\frac{r^2}{n}\,\omega_{m,L}(\bx)^2
 =\frac{r^2}{n}
 \inf_{\dim V\le m}
 \sup_{K\in\sK_L(\bx)}
 \lambda_{\max}(P_{V^\perp}KP_{V^\perp}).
\end{equation}
Moreover, let $N_*$ be any optimizer in \eqref{eq:tau-def}, write its eigenvalues as
\begin{align}
\lambda_1(N_*)\ge\cdots\ge\lambda_n(N_*)\ge0,
\notag
\end{align}
and let $V_m^*$ be the span of eigenvectors corresponding to the first $m$ eigenvalues.  Then
\begin{align}
 d_{m,L}(\bx;r)
 &\le \frac{r}{\sqrt n}\sqrt{\lambda_{m+1}(N_*)}
 \label{eq:width-spectral-envelope}\\
 &\le r\sqrt{\frac{\tau_L(\bx)}{n(m+1)}}
 \label{eq:width-trace-envelope}\\
 &=\frac{r\Lambda_L(\bx)}{\sqrt{m+1}}
 \bigl(\frac{S_L(\bx)}{n}\bigr)^{1/2}
 \label{eq:width-profile-envelope}\\
 &\le
 \frac{r\Lambda_L(\bx)}{\sqrt{m+1}}
 \bigl(\max_i\norm{x_i}_2\bigr)^{2^{-(L-1)}}.
 \label{eq:width-radius-envelope}
\end{align}
Finally, $d_{n,L}(\bx;r)=0$.
\end{theorem}

\begin{remark}[Uniform linear approximation versus target-dependent approximation]\label{rem:uniform-linear-width}
\Cref{thm:kolmogorov-width} is a uniform linear approximation theorem.  It has the quantifier order
\begin{align}
\exists V_m\quad\forall f\in B_L(r),
\notag
\end{align}
whereas an $m$-term or Monte-Carlo approximation theorem generally has the order
\begin{align}
\forall f\quad\exists\text{ an $m$-term approximant depending on }f.
\notag
\end{align}
The former is therefore a stronger geometric requirement.  Its advantage is that a single feature space can be reused for every target in the full adaptive ball; its limitation is that the result is initially empirical, on the prescribed sample.
\end{remark}

The recursive threshold proof produces a particular feasible covariance whose full spectrum can be used instead of retaining only its trace.  For a decomposition dictionary $\mathscr D$, use the notation from the proof of \Cref{thm:recursive-threshold}: put
\begin{equation}\label{eq:approx-HD}
H_{\mathscr D}
:=s(\mathscr D)\sum_{D\in\mathscr D}\bone_D\bone_D^\top,
\end{equation}
set $p_i:=d_{L,i}$, $I_t:=\{i:p_i\ge t\}$, $P_t:=\diag(\bone_{I_t})$, and define
\begin{equation}\label{eq:approx-ND}
N_{\mathscr D}
:=2\int_0^\infty P_tH_{\mathscr D}P_t\,dt.
\end{equation}
The proof of \Cref{thm:recursive-threshold} shows that $N_{\mathscr D}\succeq K$ for every $K\in\sK_L(\bx)$.

\begin{corollary}[Recursive spectral compression]\label{cor:recursive-spectral-width}
For every decomposition dictionary $\mathscr D$ and every $m\in\{0,\ldots,n-1\}$,
\begin{equation}\label{eq:recursive-spectral-width}
d_{m,L}(\bx;r)
 \le
 \frac{r}{\sqrt n}
 \sqrt{\lambda_{m+1}(N_{\mathscr D})}.
\end{equation}
In particular, using only the trace of $N_{\mathscr D}$ recovers
\begin{equation}\label{eq:recursive-width-D}
 d_{m,L}(\bx;r)
 \le
 r\bigl(\max_i\norm{x_i}_2\bigr)^{2^{-(L-1)}}
 \sqrt{\frac{2\mathfrak D_{L-1}(\bx)}{m+1}}.
\end{equation}
Thus the recursive threshold theorem controls both statistical width and approximation width, while \eqref{eq:recursive-spectral-width} can be strictly sharper when the recursive covariance has spectral decay.
\end{corollary}

\begin{theorem}[Balanced hierarchical traces admit a common Haar trace subspace]\label{thm:haar-width}
Assume $n=2^H$ and $\norm{x_i}_2=R$ for every $i$.  Suppose that a complete balanced binary partition tree on $[n]$ has the property that every set in $\mathcal F_{L-1}(\bx)$ is a disjoint union of at most $q$ tree nodes.  Then, for every $m\in\{0,\ldots,n-1\}$,
\begin{equation}\label{eq:haar-width}
d_{m,L}(\bx;r)
 \le
 2rR^{\,2^{-(L-1)}}
 \sqrt{\frac{q}{m+1}}.
\end{equation}
For $m\ge1$, the approximating space may be chosen independently of $f$ as the span of the constant vector and the $m-1$ coarsest discrete Haar wavelets associated with the partition tree, with ties resolved arbitrarily; for $m=0$, the approximating space is $\{0\}$.
\end{theorem}

\begin{remark}[Why the hierarchical width is stronger than the trace estimate]\label{rem:haar-no-log}
The complexity consequence of the same hierarchical assumption contains the tree height through $\Lambda_L\le\sqrt{2q(H+1)}$.  In contrast, \eqref{eq:haar-width} contains no factor $H$ or $\log n$.  The reason is spectral: the trace of the hierarchical covariance sums all resolutions, whereas an $m$-dimensional approximation discards fine Haar modes and depends on the $(m+1)$-st eigenvalue rather than on the sum of all eigenvalues.
\end{remark}

\begin{proposition}[Stable disjoint thresholds give a width lower bound]\label{prop:stable-width-lower}
Fix $m<M\le n$.  Suppose there are pairwise disjoint sets $A_1,\ldots,A_M\subseteq[n]$, each of cardinality $s$, and a number $\delta>0$ such that, for every $j$, there is an $a^{(j)}\in\sS_{L-1}(\bx)$ for which either
\begin{align}
\{i:a_i^{(j)}\ge u\}=A_j
\quad\text{for every }u\in[t_j,t_j+\delta],
\notag
\end{align}
or
\begin{align}
\{i:a_i^{(j)}\le-u\}=A_j
\quad\text{for every }u\in[t_j,t_j+\delta]
\notag
\end{align}
for some $t_j\ge0$.  Then
\begin{equation}\label{eq:stable-width-lower}
d_{m,L}(\bx;r)
 \ge
 r\sqrt{\frac{\delta s}{n}
 \bigl(1-\frac{m}{M}\bigr)}.
\end{equation}
If the sets form an equal partition of $[n]$ and $M=2m$ with $m\ge1$, then
\begin{equation}\label{eq:stable-width-partition}
 d_{m,L}(\bx;r)
 \ge \frac{r}{2}\sqrt{\frac{\delta}{m}}.
\end{equation}
Hence the $m^{-1/2}$ upper rate in \Cref{thm:haar-width} is sharp whenever the hierarchy also contains stable disjoint threshold blocks whose persistence length $\delta$ is comparable to the natural preceding-layer amplitude.
\end{proposition}

\begin{remark}[Spectral phase diagram]\label{rem:spectral-phase}
The exact width formula is not tied to the exponent $1/2$.  If a feasible common covariance $N$ satisfies
\begin{align}
\lambda_j(N)\le Cn\,j^{-2\beta},
\notag
\end{align}
then \eqref{eq:width-spectral-envelope} gives $d_{m,L}(\bx;r)\le r\sqrt C\,(m+1)^{-\beta}$.  Finite rank gives exact recovery once $m$ reaches that rank, and exponential eigenvalue decay gives exponential linear approximation.  Thus the factorization profile supplies a full approximation spectrum, whereas its trace records only one aggregate statistical scale.
\end{remark}

\subsection{Learning Consequences}\label{sec:learning}

\subsubsection{Uniform deviation}

Let $Z=(X,Y)$ and let $\ell:\R\times\mathcal Y\to[0,M_\ell]$ be $L_\ell$-Lipschitz in its first argument.  Write
\begin{align}
R(f)=\E\ell(f(X),Y),
\qquad
R_n(f)=\frac1n\sum_{i=1}^n\ell(f(X_i),Y_i).
\notag
\end{align}
To avoid a separate empirical-process measurability digression, probability statements in this section may be read in outer probability.  Under the separability and relative-compactness properties of BKL radius balls established in the foundational BKL theory~\cite{mohammadigohari2026bkl}, the suprema are measurable.  In particular, on every fixed input sample the predictor supremum is the support function of the compact trace body $\cA_L(X_{1:n})$.

\begin{corollary}[Profile-dependent uniform deviation]\label{cor:generalization}
There exists a universal constant $C>0$ such that, for every $r>0$ and $\delta\in(0,1)$, with probability at least $1-\delta$,
\begin{equation}\label{eq:generalization}
\sup_{f\in B_L(r)}|R(f)-R_n(f)|
\le
C L_\ell\frac{r\sqrt{\tau_L(X_{1:n})}}{n}
+3M_\ell\sqrt{\frac{\log(4/\delta)}{2n}}.
\end{equation}
Equivalently, the first term is bounded by
\begin{equation}\label{eq:generalization-lambda}
C L_\ell\frac{r\Lambda_L(X_{1:n})}{\sqrt n}
\bigl(\max_i\norm{X_i}_2\bigr)^{2^{-(L-1)}}.
\end{equation}
\end{corollary}

The proof combines empirical Rademacher deviation, Lipschitz contraction, Gaussian--Rademacher comparison, and \Cref{thm:profile-complexity}; see Appendix~\ref{app:learning}.  If $\widehat f_r$ minimizes $R_n$ over $B_L(r)$, then
\begin{equation}\label{eq:erm-oracle}
R(\widehat f_r)-\inf_{f\in B_L(r)}R(f)
\le 2\sup_{f\in B_L(r)}|R(f)-R_n(f)|.
\end{equation}
Thus the usual $n^{-1/2}$ exponent is recovered on samples for which $\Lambda_L(X_{1:n})$ remains bounded; the orthogonal lower bound shows why this cannot hold uniformly over the unchanged full class.

\begin{corollary}[Bounded squared-loss consequence]\label{cor:bounded-squared-loss}
Assume $|Y|\le B_Y$ almost surely and define
\begin{align}
B_f:=r\sup_{x\in\cX}\rho_L(x)<\infty.
\notag
\end{align}
For $f\in B_L(r)$, the pointwise estimate gives $|f(X)|\le B_f$.  Equations \eqref{eq:pointwise-scale} and \eqref{eq:holder-scale} make $B_L(r)$ uniformly bounded and equicontinuous on compact $\cX$.  The foundational BKL theory proves that this radius ball is uniformly closed in $C(\cX)$~\cite{mohammadigohari2026bkl}.  Arzel\`a--Ascoli therefore makes $B_L(r)$ compact in the uniform topology.  Since $f\mapsto R_n(f)$ for squared loss is continuous under uniform convergence, a constrained empirical squared-loss minimizer exists.

Let $\Pi_{B_f}:\R\to[-B_f,B_f]$ denote Euclidean clipping and define
\begin{align}
\widetilde\ell(u,y)
:=
\bigl(\Pi_{B_f}(u)-y\bigr)^2.
\notag
\end{align}
For $|y|\le B_Y$, the loss $\widetilde\ell(\cdot,y)$ takes values in $[0,(B_f+B_Y)^2]$ on all of $\R$ and is $2(B_f+B_Y)$-Lipschitz.  Indeed, clipping is nonexpansive, and $v\mapsto(v-y)^2$ has derivative bounded in absolute value by $2(B_f+B_Y)$ on $[-B_f,B_f]$.  Moreover,
\begin{align}
\widetilde\ell(f(x),y)
=
\bigl(f(x)-y\bigr)^2
\notag
\end{align}
for every $f\in B_L(r)$, because $|f(x)|\le B_f$.  Therefore \Cref{cor:generalization} applies with
\begin{align}
M_\ell=(B_f+B_Y)^2,
\qquad
L_\ell=2(B_f+B_Y).
\notag
\end{align}
In particular, every constrained empirical squared-loss minimizer over $B_L(r)$ satisfies the excess-risk conclusion in \eqref{eq:erm-oracle} with these constants.  No assertion is made for unbounded responses or predictions.
\end{corollary}

\section{Certification and Computational Methods}\label{sec:certification}

The common-covariance program lies at the intersection of semidefinite and convex optimization \cite{boyd2004convex}, optimal experimental design \cite{kiefer1959optimum,kieferwolfowitz1960equivalence,huan2024oed}, and semi-infinite exchange methods \cite{blankenship1976infinitely,hettich1993semiinfinite}.  The sparse update and a posteriori gap below are related to Frank--Wolfe methodology \cite{jaggi2013frankwolfe}, while the resistance relaxation builds on effective-resistance optimization and sparsification \cite{ghosh2008resistance,spielmansrivastava2011sparsification}.  The BKL-specific difficulty is the recursive separation problem over the closed Brownian Dirac-trace family.

\subsection{Sparse Brownian Contact Certificates}\label{sec:sparse-contact}

The full profile is defined by infinitely many covariance constraints, but an exact optimizer always has a finite Brownian contact certificate.  Put
\begin{align}
d_n:=\frac{n(n+1)}{2}.
\notag
\end{align}

\begin{theorem}[Sparse Brownian contact theorem]\label{thm:sparse-contact}
Let $N_*$ be any optimizer in \eqref{eq:tau-def}.  There exist an integer $m\le d_n+1$, traces
$a_1,\ldots,a_m\in\sS_{L-1}(\bx)$, and vectors $z_1,\ldots,z_m\in\R^n$ such that
\begin{align}
\sum_{j=1}^m z_jz_j^\top&=I_n,\label{eq:contact-isotropy}\\
\bigl(N_*-\mathbf B(a_j)\bigr)z_j&=0,
\qquad j=1,\ldots,m,\label{eq:contact-equations}\\
\tau_L(\bx)&=\sum_{j=1}^m z_j^\top\mathbf B(a_j)z_j.\label{eq:contact-value}
\end{align}
Moreover, the finite family
\begin{align}
\mathcal F_*:=\set{\mathbf B(a_1),\ldots,\mathbf B(a_m)}
\notag
\end{align}
already determines the exact full profile:
\begin{equation}\label{eq:finite-exact-profile}
\tau(\mathcal F_*)=\tau_L(\bx).
\end{equation}
\end{theorem}

Thus the exact empirical profile is determined by at most $O(n^2)$ Brownian matrices from the closed last-layer Dirac trace family.  The vectors $z_j$ identify the isotropic directions in which those covariance constraints touch the optimal envelope.  When $a_j$ is attained by an actual previous-layer unit atom, $\mathbf B(a_j)$ is a realized Dirac kernel.  In general, $a_j$ belongs only to the closed trace family, and $\mathbf B(a_j)$ can be approximated arbitrarily closely by realized Dirac kernels.  The proof is finite-dimensional and yields an exact finite certificate, but it is existential rather than an algorithm for locating the contact family; see Appendix~\ref{app:optimization}.

\begin{corollary}[Finite attainment of the robust multiplier optimum]\label{cor:finite-multiplier-attainment}
The supremum in \eqref{eq:robust-minimax} is attained by a symmetric probability measure with covariance $I_n$ and support cardinality at most
\begin{align}
2\bigl(\frac{n(n+1)}2+1\bigr).
\notag
\end{align}
\end{corollary}

\subsection{Certified Active-Set Bounds}\label{sec:active-set}

For a finite active family $\mathcal F\subset\sK_L(\bx)$, solving \eqref{eq:finite-primal} gives a certified lower bound.  A full-family upper bound requires a verified separation value.

\begin{proposition}[Flooring and certified active-set brackets]\label{prop:certified-computation}
For $\varepsilon\ge0$, define the floored full-family value
\begin{equation}\label{eq:floored-profile}
\tau_{L,\varepsilon}(\bx)
:=\min\set{\tr N:N\succeq\varepsilon I_n,\ N\succeq K\text{ for every }K\in\sK_L(\bx)}.
\end{equation}
Then
\begin{equation}\label{eq:floor-bias}
\tau_L(\bx)\le\tau_{L,\varepsilon}(\bx)
\le\tau_L(\bx)+n\varepsilon.
\end{equation}
Let $\mathcal F\subset\sK_L(\bx)$ be finite, let $N_{\mathcal F}$ solve its finite primal problem, and let $\varepsilon>0$.  Define
\begin{equation}\label{eq:separation}
\operatorname{sep}_L(N;\bx)
:=\sup_{a\in\sS_{L-1}(\bx)}
\lambda_{\max}\!\bigl(N^{-1/2}\mathbf B(a)N^{-1/2}\bigr),
\qquad N\succ0.
\end{equation}
Then
\begin{equation}\label{eq:certified-bracket}
\tr N_{\mathcal F}=\tau(\mathcal F)
\le\tau_L(\bx)
\le
\operatorname{sep}_L(N_{\mathcal F}+\varepsilon I_n;\bx)
\bigl(\tr N_{\mathcal F}+n\varepsilon\bigr).
\end{equation}
In particular, if a verified oracle gives
$\operatorname{sep}_L(N_{\mathcal F}+\varepsilon I_n;\bx)\le1+\eta$, then
\begin{equation}\label{eq:eta-bracket}
\tau(\mathcal F)
\le\tau_L(\bx)
\le(1+\eta)\bigl(\tau(\mathcal F)+n\varepsilon\bigr).
\end{equation}
For nested active families, the lower bounds $\tau(\mathcal F)$ are monotone nondecreasing.
\end{proposition}

A certified exchange iteration therefore has the following form.
\begin{enumerate}
    \item Start from a finite family $\mathcal F_0$ of admissible Brownian kernels.
    \item Solve the finite SDP and obtain $N_t$ and the lower certificate $\underline\tau_t:=\tr N_t$.
    \item Choose $\varepsilon_t>0$ and evaluate, or rigorously upper-bound,
    $\alpha_t:=\operatorname{sep}_L(N_t+\varepsilon_t I_n;\bx)$.
    \item Report the valid upper certificate
    $\overline\tau_t:=\alpha_t(\tr N_t+n\varepsilon_t)$.
    \item If the bracket is not sufficiently small, generate a witness $a_t$ with a large generalized eigenvalue and set
    $\mathcal F_{t+1}:=\mathcal F_t\cup\set{\mathbf B(a_t)}$.
\end{enumerate}
Every iteration is trustworthy even when the witness search is nonconvex: the active SDP supplies a lower certificate, and only a verified separation upper bound is used as an upper certificate.  By \Cref{thm:sparse-contact}, some active family with at most $d_n+1$ Brownian contact matrices is exact, although a generic local witness generator is not asserted to discover that family in finitely many steps.

\begin{figure}[t]
\small
\textbf{Algorithm 1: Active common-covariance certification.}
\begin{enumerate}[(1)]
\item Choose a finite active Brownian covariance family $\mathcal F_0$ and a numerical floor $\eta>0$.
\item Solve the active SDP to obtain $N_t$ and the certified lower value $\tau(\mathcal F_t)$.
\item Evaluate a witness generator for $\operatorname{sep}_L(N_t;\bx)$.  If separation is globally verified, record the certified upper value $\operatorname{sep}_L(N_t)\tr N_t$.
\item If the verified upper--lower gap is below tolerance, stop.  Otherwise add a violating Brownian matrix to $\mathcal F_t$ and repeat.
\item When global separation is unavailable, retain the active value only as a lower certificate and use the resistance program in \Cref{thm:resistance-design} for a conservative global upper certificate.
\end{enumerate}
\caption{The lower certificate is valid at every iteration.  A local witness can improve the active family but is not itself a global upper certificate.}
\label{alg:active-certification}
\end{figure}

\subsection{Convex Resistance-Design Upper Certificates}\label{sec:resistance-design}

The graph majorant in \Cref{thm:graph-majorant} yields a second optimization route that is fully convex and does not require solving the recursive separation problem.  Assume in this subsection that $n\ge2$ and that the sample locations are pairwise distinct.  Repeated locations require the occupancy-weighted replication formulation of \Cref{thm:random-repetition-transfer}, because unweighted contraction does not preserve the trace objective.  Put
\begin{align}
\beta_L:=2^{-(L-2)},
\qquad
J_n:=\frac1n\bone\bone^\top.
\notag
\end{align}
Let $\mathcal E$ be the edge set of any connected simple graph on $[n]$.  For an edge $e=\{i,j\}$, set
\begin{align}
b_e:=e_i-e_j,
\qquad
c_e:=\norm{x_i-x_j}_2^{\,\beta_L}>0.
\notag
\end{align}
For $w=(w_e)_{e\in\mathcal E}\in\R_+^{\mathcal E}$, define
\begin{equation}\label{eq:weighted-laplacian-design}
L(w):=\sum_{e\in\mathcal E}w_eb_eb_e^\top.
\end{equation}
On the cost simplex
\begin{equation}\label{eq:resistance-simplex}
\mathcal W_{\mathcal E}
:=\set{w\in\R_+^{\mathcal E}:\sum_{e\in\mathcal E}c_ew_e=1},
\end{equation}
define the extended trace-inverse objective
\begin{equation}\label{eq:resistance-objective}
\Phi_{\mathcal E}(w)
:=
\begin{cases}
\tr\bigl(L(w)+J_n\bigr)^{-1}-1,
&L(w)+J_n\succ0,\\
+\infty,&\text{otherwise}.
\end{cases}
\end{equation}
The finite value case is exactly $\Phi_{\mathcal E}(w)=\tr L(w)^\dagger$.

\begin{theorem}[Convex resistance-design certificate and optimality gap]\label{thm:resistance-design}
The program
\begin{equation}\label{eq:resistance-program}
\Psi_{L,\mathcal E}(\bx)
:=\min_{w\in\mathcal W_{\mathcal E}}\Phi_{\mathcal E}(w)
\end{equation}
has a minimizer and is a convex optimization problem.  It gives the certified full-family bounds
\begin{align}
\tau_L(\bx)
&\le 2S_L(\bx)+2\Psi_{L,\mathcal E}(\bx),\label{eq:resistance-tau}\\
\Lambda_L(\bx)^2
&\le 2+\frac{2\Psi_{L,\mathcal E}(\bx)}{S_L(\bx)}
\qquad\text{when }S_L(\bx)>0.\label{eq:resistance-lambda}
\end{align}
For every $w\in\mathcal W_{\mathcal E}$ with connected positive-weight support, define
\begin{equation}\label{eq:biharmonic-edge}
q_e(w):=b_e^\top\bigl(L(w)^\dagger\bigr)^2b_e.
\end{equation}
Then
\begin{align}
\frac{\partial\Phi_{\mathcal E}}{\partial w_e}(w)&=-q_e(w),\label{eq:resistance-gradient}\\
\sum_{e\in\mathcal E}w_eq_e(w)&=\Phi_{\mathcal E}(w).\label{eq:resistance-euler}
\end{align}
Moreover, $w$ is optimal in \eqref{eq:resistance-program} if and only if
\begin{equation}\label{eq:resistance-kkt}
\frac{q_e(w)}{c_e}\le\Phi_{\mathcal E}(w)
\quad\text{for every }e\in\mathcal E,
\qquad
\frac{q_e(w)}{c_e}=\Phi_{\mathcal E}(w)
\quad\text{whenever }w_e>0.
\end{equation}
For every connected feasible $w$, the computable quantity
\begin{equation}\label{eq:resistance-gap}
g_{\mathcal E}(w)
:=\max_{e\in\mathcal E}\frac{q_e(w)}{c_e}
-\Phi_{\mathcal E}(w)
\end{equation}
is nonnegative and certifies
\begin{equation}\label{eq:resistance-gap-bound}
0\le\Phi_{\mathcal E}(w)-\Psi_{L,\mathcal E}(\bx)
\le g_{\mathcal E}(w).
\end{equation}
\end{theorem}

The theorem yields a sparse Frank--Wolfe implementation.  At a connected iterate $w_t$, choose
\begin{align}
e_t\in\argmax_{e\in\mathcal E}\frac{q_e(w_t)}{c_e},
\qquad
s_t:=\frac{1}{c_{e_t}}\mathbf1_{\{e_t\}}
\notag
\end{align}
and minimize $\Phi_{\mathcal E}((1-\gamma)w_t+\gamma s_t)$ over $0\le\gamma<1$; when $n=2$, the endpoint $\gamma=1$ is also admissible.  The old connected support remains present for every $\gamma<1$, the objective does not increase under exact line search, and \eqref{eq:resistance-gap} is an a posteriori stopping certificate.

\begin{corollary}[Closed-form optimal certificate on a fixed tree]\label{cor:tree-design}
Let $T$ be a tree on $[n]$.  Removing an edge $e$ separates the tree into components of cardinalities $s_e$ and $n-s_e$, and put
\begin{equation}\label{eq:tree-ae}
a_e:=\frac{s_e(n-s_e)}{n}.
\end{equation}
Under the normalization $\sum_{e\in T}c_ew_e=1$, the unique optimal positive weights are
\begin{equation}\label{eq:tree-optimal-weights}
w_e^*
=\frac{\sqrt{a_e/c_e}}
{\sum_{f\in T}\sqrt{a_fc_f}},
\end{equation}
and the optimal tree value is
\begin{equation}\label{eq:tree-optimal-value}
\Psi_{L,T}(\bx)
=\bigl(\sum_{e\in T}\sqrt{a_ec_e}\bigr)^2.
\end{equation}
Consequently,
\begin{equation}\label{eq:tree-profile-certificate}
\tau_L(\bx)
\le 2S_L(\bx)
+2\bigl(\sum_{e\in T}\sqrt{a_ec_e}\bigr)^2.
\end{equation}
\end{corollary}

The closed form gives an inexpensive certified initialization for the convex design problem and permits direct comparison of paths, minimum-spanning trees, hierarchical trees, and subsequently optimized sparse graphs.

\subsection{Brownian Separation Computations}\label{sec:brownian-oracle}

The exact active-set method requires maximizing the generalized eigenvalue in \eqref{eq:separation}.  The following proposition isolates the BKL-specific computations used by such a witness generator.

\begin{proposition}[Sorting formula and differential Brownian oracle]\label{prop:brownian-oracle}
Let $N\succ0$.  Then
\begin{equation}\label{eq:oracle-minmax}
\operatorname{sep}_L(N;\bx)
=\sup_{a\in\sS_{L-1}(\bx)}
\sup_{\norm{v}_2=1}
q^\top\mathbf B(a)q,
\qquad
q:=N^{-1/2}v.
\end{equation}
Fix $a,q\in\R^n$.  For each sign $\sigma\in\{+,-\}$, list the indices with $\sigma a_i>0$ as
$i_1^\sigma,\ldots,i_{m_\sigma}^\sigma$ so that
\begin{align}
0<t_1^\sigma\le\cdots\le t_{m_\sigma}^\sigma,
\qquad
t_k^\sigma:=|a_{i_k^\sigma}|,
\qquad
t_0^\sigma:=0,
\notag
\end{align}
and define the suffix sums
\begin{align}
Q_k^\sigma:=\sum_{\ell=k}^{m_\sigma}q_{i_\ell^\sigma}.
\notag
\end{align}
Then
\begin{equation}\label{eq:oracle-sorting}
q^\top\mathbf B(a)q
=\sum_{\sigma\in\{+,-\}}
\sum_{k=1}^{m_\sigma}
\bigl(t_k^\sigma-t_{k-1}^\sigma\bigr)
\bigl(Q_k^\sigma\bigr)^2.
\end{equation}
Thus the Brownian quadratic form is evaluable in $O(n\log n)$ arithmetic operations after sorting.

Assume additionally that every $a_i$ is nonzero and that the coordinates of $a$ are pairwise distinct.  Put $Q:=\sum_{j=1}^nq_j$.  Then the function
$F_q(a):=q^\top\mathbf B(a)q$ is differentiable at $a$, with
\begin{equation}\label{eq:oracle-gradient-a}
\frac{\partial F_q}{\partial a_k}(a)
=q_k\bigl[
Q\,\operatorname{sign}(a_k)
-\sum_{j=1}^nq_j\operatorname{sign}(a_k-a_j)
\bigr],
\end{equation}
where the $j=k$ summand is interpreted as zero.  If $a=a(\theta)$ is differentiable and the largest eigenvalue in \eqref{eq:separation} is simple at $\theta$, with normalized eigenvector $v$, then
\begin{equation}\label{eq:oracle-chain-rule}
\nabla_\theta
\lambda_{\max}\!\bigl(N^{-1/2}\mathbf B(a(\theta))N^{-1/2}\bigr)
=J_a(\theta)^\top\nabla_aF_q(a(\theta)),
\qquad q=N^{-1/2}v.
\end{equation}
\end{proposition}

Equation \eqref{eq:oracle-sorting} gives an exact fast objective evaluation, while \eqref{eq:oracle-chain-rule} permits projected or automatic-differentiation ascent through a differentiable finite-atomic ladder parameterization away from amplitude ties and eigenvalue multiplicities.  At nonsmooth points one may use one-sided derivatives, a consistent generalized-gradient choice, or a smooth approximation, but no differentiability claim beyond the stated hypotheses is needed here.  Such ascent produces valid violating kernels and therefore improves the active-set lower certificate.  Because the recursive trace family is nonconvex, a locally optimized witness is not by itself a verified upper bound; rigorous upper certification must use a global separation bound or the convex resistance/threshold certificates above.

\section{Finite Covariance-Path Illustration}\label{sec:numerics}\label{sec:cove-case-study}

\subsection{Training-only path selection}

Let $(z_i,y_i)_{i=1}^n$ be standardized frozen features and labels.  With class means $\mu_c$, global mean $\mu$, and class sizes $n_c$, define
\begin{align}
B_j&:=\sum_c n_c(\mu_{c,j}-\mu_j)^2,\notag\\
W_j&:=\sum_c\sum_{i:y_i=c}(z_{i,j}-\mu_{c,j})^2,\notag\\
r_j&:=\max\Bigl\{\frac{B_j}{W_j+\eps},\eps\Bigr\},
\qquad
\sigma_j:=\frac{r_j}{d^{-1}\sum_{\ell=1}^d r_\ell},
\label{cove:eq:fisher-diagonal}
\end{align}
and $\Sigma_{\mathrm F}:=\diag(\sigma_1,\ldots,\sigma_d)$.  This supervised diagonal follows the classical between/within-class scatter principle \cite{fisher1936multiple}.  Standardization and the Fisher diagonal are refitted within every training fold.

The candidate grid is
\begin{equation}\label{cove:eq:candidate-grid}
\mathcal P:=\mathcal S\times\mathcal C,
\qquad
\mathcal C:=\{0.03,0.1,0.3,1,3\}.
\end{equation}
Following the predictive-reuse principle of cross-validation \cite{stone1974crossvalidatory}, for $R$ repeats of $F_{\mathrm{cv}}$-fold stratified cross-validation every candidate $\vartheta=(s,C)$ uses the same splits.  If $\widehat y_{i,\vartheta}^{(r)}$ is the out-of-fold prediction, set
\begin{equation}\label{cove:eq:crossfit-accuracy}
q_{i,\vartheta}:=\frac1R\sum_{r=1}^R\mathbf 1\{\widehat y_{i,\vartheta}^{(r)}=y_i\},
\qquad
\widehat A_{\vartheta}:=\frac1n\sum_{i=1}^nq_{i,\vartheta}.
\end{equation}
The empirical maximizer is selected, with exact ties broken by smaller sampling standard error, then smaller $C$, then smaller $s$.

\begin{proposition}[Training-only selection]\label{cove:prop:training-only-selection}
For fixed frozen training features, labels, split seeds, and candidate grids, the selected pair is measurable with respect to the labelled training sample and the split seeds.  It does not depend on the test labels.
\end{proposition}

\subsection{Protocol and predictive results}

All images are mapped to the $512$-dimensional penultimate representation of an ImageNet-pretrained ResNet-18 \cite{he2016resnet}; the feature extractor is not fine-tuned.  Each dataset--seed pair uses $1000$ labelled training examples, $F_{\mathrm{cv}}=3$ folds, $R=5$ repeats, and the complete test split.  The comparison contains ordinary and Fisher-reweighted multinomial logistic, RBF, and Laplace models; identity and full-Fisher Brownian endpoints; and the selected path.  Six development datasets informed the correction of the selector.  STL10, GTSRB, and Food101 were locked before fitting, including their acceptance criteria.

\begin{table}[t]
\centering
\caption{Development mean test accuracy and comparisons over five paired seeds. The conventional column is the best mean among ordinary or Fisher-reweighted RBF and Laplace SVC; the endpoint column is the better mean of identity and full-Fisher Brownian SVC. Deltas are percentage points.}
\label{cove:tab:development-results}
\scriptsize
\setlength{\tabcolsep}{4pt}
\begin{tabular}{lrrrrrr}
\toprule
Dataset & V7 (\%) & Conventional (\%) & $\Delta_{\mathrm{conv}}$ & Endpoint (\%) & $\Delta_{\mathrm{end}}$ & Paired conv. \\
\midrule
CIFAR-10 & 79.848 & 80.022 & -0.174 & 79.924 & -0.076 & 1/5 \\
CIFAR-100 & 45.516 & 45.046 & +0.470 & 45.598 & -0.082 & 5/5 \\
EuroSAT & 91.059 & 90.596 & +0.463 & 91.078 & -0.019 & 5/5 \\
Oxford-IIIT Pet & 88.651 & 88.874 & -0.223 & 88.738 & -0.087 & 0/5 \\
Flowers102 & 80.781 & 80.478 & +0.302 & 80.781 & +0.000 & 5/5 \\
DTD & 59.766 & 59.457 & +0.309 & 59.809 & -0.043 & 4/5 \\
\midrule
Macro average & -- & -- & +0.191 & -- & -0.051 & 20/30 \\
\bottomrule
\end{tabular}
\end{table}

\begin{table}[t]
\centering
\caption{Confirmatory mean test accuracy and comparisons over five paired seeds. The conventional column is the best mean among ordinary or Fisher-reweighted RBF and Laplace SVC; the endpoint column is the better mean of identity and full-Fisher Brownian SVC. Deltas are percentage points.}
\label{cove:tab:confirmatory-results}
\scriptsize
\setlength{\tabcolsep}{4pt}
\begin{tabular}{lrrrrrr}
\toprule
Dataset & V7 (\%) & Conventional (\%) & $\Delta_{\mathrm{conv}}$ & Endpoint (\%) & $\Delta_{\mathrm{end}}$ & Paired conv. \\
\midrule
STL10 & 93.837 & 94.025 & -0.188 & 93.915 & -0.078 & 0/5 \\
GTSRB & 67.219 & 66.670 & +0.549 & 67.211 & +0.008 & 4/5 \\
Food101 & 37.640 & 37.512 & +0.128 & 37.640 & +0.000 & 3/5 \\
\midrule
Macro average & -- & -- & +0.163 & -- & -0.023 & 7/15 \\
\bottomrule
\end{tabular}
\end{table}

The development macro-average difference from the strongest tuned RBF/Laplace control is \DevMeanConvGain{} percentage points.  On the locked suite it is \ConfMeanConvGain{} percentage points.  The predeclared confirmatory requirement was at least $8/15$ paired wins or ties, whereas the observed count is \ConfPairedWins{}/15.  The official confirmatory decision is therefore a failure.  This decision is not altered by the positive macro-average difference.

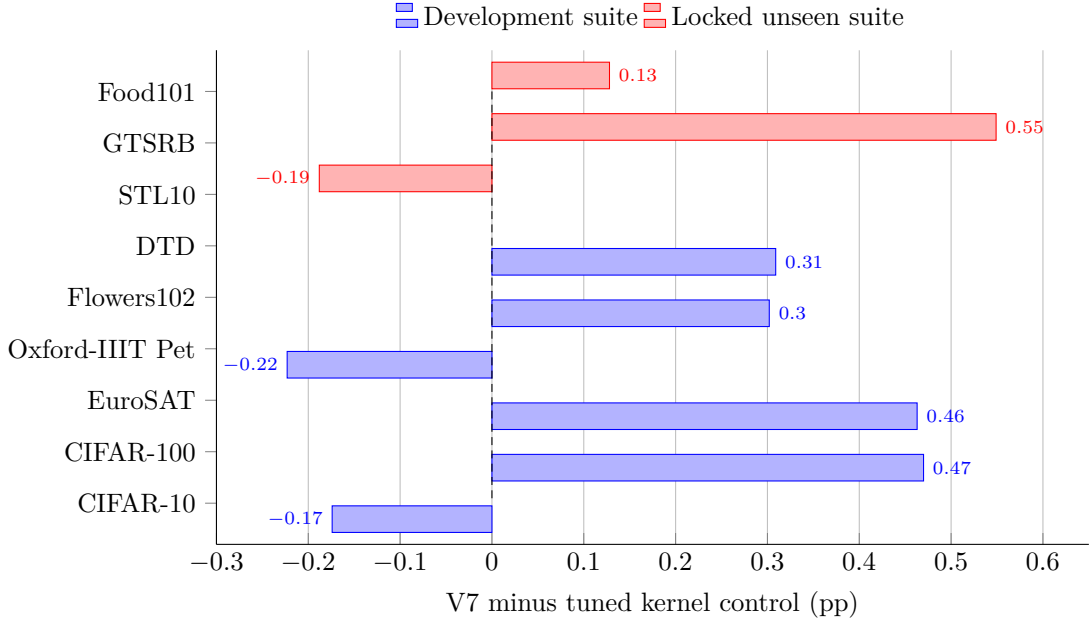
\begin{figure}[t]
\centering
\begin{tikzpicture}
\begin{axis}[
  xbar,
  width=0.84\linewidth,
  height=0.52\linewidth,
  xlabel={V7 minus tuned kernel control (pp)},
  symbolic y coords={Food101,GTSRB,STL10,DTD,Flowers102,Oxford-IIIT Pet,EuroSAT,CIFAR-100,CIFAR-10},
  ytick={Food101,GTSRB,STL10,DTD,Flowers102,Oxford-IIIT Pet,EuroSAT,CIFAR-100,CIFAR-10},
  y dir=reverse,
  axis x line*=bottom,
  axis y line*=left,
  xmajorgrids=true,
  xmin=-0.30,
  xmax=0.65,
  legend style={at={(0.5,1.02)},anchor=south,legend columns=2,draw=none},
  nodes near coords,
  every node near coord/.append style={font=\scriptsize},
]
\addplot coordinates {(-0.174,CIFAR-10) (+0.470,CIFAR-100) (+0.463,EuroSAT) (-0.223,Oxford-IIIT Pet) (+0.302,Flowers102) (+0.309,DTD)};
\addlegendentry{Development suite}
\addplot coordinates {(-0.188,STL10) (+0.549,GTSRB) (+0.128,Food101)};
\addlegendentry{Locked unseen suite}
\draw[densely dashed] (axis cs:0,{Food101}) --
  (axis cs:0,{CIFAR-10});
\end{axis}
\end{tikzpicture}
\caption{Dataset-level mean differences between V7 and the strongest tuned
ordinary or Fisher-reweighted RBF/Laplace SVC. The six development datasets
informed V7; the three red entries form the locked unseen suite.}
\label{cove:fig:dataset-deltas}
\end{figure}

\subsection{Finite-path certificates and interpretation}

For every selected fit, the complete five-kernel path is evaluated on a stratified anchor.  A repaired certificate is accepted only when the minimum domination residual is at least $-10^{-10}$.  The certificate summary is shown in \Cref{cove:tab:certificate-summary}.

\begin{table}[t]
\centering
\caption{Selection and common-covariance certificate diagnostics. Residuals are minimum domination eigenvalues after the explicitly reported identity repair.}
\label{cove:tab:certificate-summary}
\small
\begin{tabular}{lrrrrrr}
\toprule
Suite & V7 records & $s>0$ & $s=1$ & Failures & Min. residual & Max. repair \\
\midrule
Development & 30 & 29 & 14 & 0 & 9.999e-11 & 0.0012 \\
Locked confirmation & 15 & 15 & 8 & 0 & 9.999e-11 & 6.743e-04 \\
\bottomrule
\end{tabular}
\end{table}

All $45$ selected-path certificates succeed, and $44$ of the $45$ selected covariance strengths are strictly positive.  These facts support the finite geometric construction, not a claim of universal predictive superiority.  The certificate covers only the five chosen path kernels on the empirical anchor.  It does not certify the unrestricted recursive BKL family, and the identity repair in \eqref{cove:eq:identity-repair} is not itself a realizable Brownian mixture.

\section{Discussion, Limitations, and Open Problems}\label{sec:outlook}

\paragraph{Exact conclusions.}
At finite sample size, the unrestricted adaptive BKL ball is a union of RKHS ellipsoids, and the minimum-trace common covariance, its absolutely two-summing formulation, and its covariance-dominated multiplier representation are exact.  The common spectrum gives an exact robust formula for empirical widths and a transductive subspace for simultaneous approximation.  Finite contact is exact but existential.

\paragraph{Conditional conclusions.}
Gaussian reverses require more than covariance domination.  Stable thresholds, integrated profiles, Gram conditioning, repetition, occupancy, ordering, hierarchy, and perturbation margins give sufficient reverse estimates in their stated regimes.  A bounded Gaussian defect means that the covariance certificate is tight; it does not imply a parametric full-class rate.  Orthogonal and high-dimensional spherical samples make this distinction explicit.

\paragraph{Computational conclusions.}
Finite active SDPs give lower certificates.  A verified global separation value gives an upper certificate, and resistance design gives a conservative convex upper certificate with computable gap.  Local witness search can strengthen an active family but does not prove global separation.  In the finite path, all repaired certificates succeed even though the locked predictive gate fails.  Certification and prediction therefore require separate evaluation.

\paragraph{Open problems.}
Natural next steps are: population and out-of-sample versions of the common covariance and common feature space; intrinsic oscillation bounds for unrestricted recursive traces beyond depth two; less restrictive perturbation results at large depth; localized profiles that separate global sign richness from target-dependent rates; and globally verified separation relaxations for the recursive family.  Computationally, full or low-rank supervised metrics, regression and operator-learning losses, and reusable common covariances across tasks require new protocols.  The present eigenspaces are sample-dependent, the contact theorem does not locate its contacts, and the resistance certificate may be conservative.

\appendix
\section{Proofs for Empirical Trace Geometry}\label[appendix]{app:trace}

\subsection{Recursive Diagonal Scale and Empirical Covariance Body}

\begin{proof}[Proof of \eqref{eq:pointwise-scale}, \eqref{eq:holder-scale}, and \eqref{eq:diagonal-scale}]
We first verify the diagonal and pointwise estimates used throughout the paper.  At level one, the linear kernel is $k^{(0)}(x,x')=x^\top x'$, and therefore
\begin{align}
k^{(0)}(x,x)^{1/2}
=\norm{x}_2
=\rho_1(x).
\notag
\end{align}
Assume for some $\ell\ge1$ that
\begin{align}
k^{(\ell-1)}(x,x)^{1/2}
\le\rho_\ell(x).
\notag
\end{align}
Every $u$ in the unit sphere of $H_{k^{(\ell-1)}}$ satisfies, by the RKHS evaluation inequality,
\begin{align}
|u(x)|
\le k^{(\ell-1)}(x,x)^{1/2}
\le\rho_\ell(x).
\notag
\end{align}
For the next integral Brownian kernel, the identity $\kB(t,t)=|t|$ gives
\begin{align}
k^{(\ell)}(x,x)
&=\int |u(x)|\,d\mu_\ell(u)\notag\\
&\le\rho_\ell(x).
\notag
\end{align}
Taking square roots and using the definition of the recursive scale yields
\begin{align}
k^{(\ell)}(x,x)^{1/2}
\le\rho_\ell(x)^{1/2}
=\rho_{\ell+1}(x).
\notag
\end{align}
Induction proves \eqref{eq:diagonal-scale}.  For a function $f$ in a realized top-layer RKHS, the evaluation inequality now gives
\begin{align}
|f(x)|
\le\norm{f}_{H_{k^{(L-1)}}}k^{(L-1)}(x,x)^{1/2}
\le\norm{f}_{H_{k^{(L-1)}}}\rho_L(x).
\notag
\end{align}
Taking the infimum over all admissible ladder realizations proves \eqref{eq:pointwise-scale}.

We next prove the recursive H\"older estimate.  For a kernel $k$, write
\begin{align}
d_k(x,x')^2
:=
k(x,x)+k(x',x')-2k(x,x').
\notag
\end{align}
For the linear kernel,
\begin{align}
d_{k^{(0)}}(x,x')
=\norm{x-x'}_2.
\notag
\end{align}
Assume that
\begin{align}
d_{k^{(\ell-1)}}(x,x')
\le
\norm{x-x'}_2^{\,2^{-(\ell-1)}}.
\notag
\end{align}
Every unit-sphere element $u$ of $H_{k^{(\ell-1)}}$ then satisfies
\begin{align}
|u(x)-u(x')|
&=
\abs{\bigl\langle u,k^{(\ell-1)}(\cdot,x)-k^{(\ell-1)}(\cdot,x')\bigr\rangle_{H_{k^{(\ell-1)}}}}\notag\\
&\le
\norm{k^{(\ell-1)}(\cdot,x)-k^{(\ell-1)}(\cdot,x')}_{H_{k^{(\ell-1)}}}\notag\\
&=
d_{k^{(\ell-1)}}(x,x')\notag\\
&\le
\norm{x-x'}_2^{\,2^{-(\ell-1)}}.
\notag
\end{align}
The Brownian metric identity
\begin{align}
\kB(s,s)+\kB(t,t)-2\kB(s,t)=|s-t|
\notag
\end{align}
therefore gives
\begin{align}
d_{k^{(\ell)}}(x,x')^2
&=
\int |u(x)-u(x')|\,d\mu_\ell(u)\notag\\
&\le
\norm{x-x'}_2^{\,2^{-(\ell-1)}}.
\notag
\end{align}
Taking square roots yields
\begin{align}
d_{k^{(\ell)}}(x,x')
\le
\norm{x-x'}_2^{\,2^{-\ell}}.
\notag
\end{align}
Induction gives the kernel metric bound at every level.  Hence, for a function $f$ in a realized depth-$L$ top RKHS,
\begin{align}
|f(x)-f(x')|
&\le
\norm f_{H_{k^{(L-1)}}}
d_{k^{(L-1)}}(x,x')\notag\\
&\le
\norm f_{H_{k^{(L-1)}}}
\norm{x-x'}_2^{\,2^{-(L-1)}}.
\notag
\end{align}
Taking the infimum over all admissible realizations proves \eqref{eq:holder-scale}.
\end{proof}

\begin{proof}[Proof of \Cref{thm:trace-body}]
We first evaluate one RKHS unit ball.
Fix one top-layer RKHS $H$ with kernel $k$ and sample Gram matrix $K$.  Let
\begin{align}
S:H\to\R^n,
\qquad
Sf=(f(x_1),\ldots,f(x_n))
\notag
\end{align}
be its evaluation operator.  Its adjoint is
\begin{align}
S^*z=\sum_{i=1}^n z_i k(\cdot,x_i),
\notag
\end{align}
and hence $SS^*=K$ as an operator on $\R^n$.  Let $M:=\overline{\operatorname{ran}S^*}$.  The orthogonal complement of $M$ consists exactly of the functions that vanish at all sample points: indeed,
$f\perp\operatorname{ran}S^*$ if and only if
$\langle f,k(\cdot,x_i)\rangle_H=f(x_i)=0$ for every $i$.  Replacing $f$ by its orthogonal projection onto $M$ therefore preserves $Sf$ and cannot increase $\norm f_H$.

Take a spectral decomposition $K=U\Sigma U^\top$, and let $u_1,\ldots,u_q$ be the eigenvectors corresponding to the positive eigenvalues $\sigma_1,\ldots,\sigma_q$.  Define
\begin{align}
e_j:=\sigma_j^{-1/2}S^*u_j\in M.
\notag
\end{align}
Then
\begin{align}
\langle e_j,e_k\rangle_H
=\frac{u_j^\top SS^*u_k}{\sqrt{\sigma_j\sigma_k}}
=\frac{u_j^\top Ku_k}{\sqrt{\sigma_j\sigma_k}}
=\delta_{jk},
\notag
\end{align}
The vectors with zero eigenvalue satisfy $S^*u_j=0$, while the positive-eigenvalue vectors $S^*u_j=\sqrt{\sigma_j}e_j$ span $\operatorname{ran}S^*$.  Thus $(e_j)_{j=1}^q$ is an orthonormal basis of $M$, and
$Se_j=\sqrt{\sigma_j}u_j$.  Consequently, if
$f=\sum_{j=1}^q c_je_j\in M$, then
\begin{align}
Sf=\sum_{j=1}^q c_j\sqrt{\sigma_j}u_j=K^{1/2}\sum_{j=1}^q c_ju_j,
\qquad
\norm f_H^2=\sum_{j=1}^q c_j^2.
\notag
\end{align}
It follows in both directions that
\begin{align}
S\set{f:\norm f_H\le1}=K^{1/2}B_2^n=E(K).
\notag
\end{align}
To identify the minimum norm, write $a=\sum_{j=1}^q\alpha_ju_j\in\ran K$.  The unique vector in $M$ satisfying $Sf=a$ is
\begin{align}
f_a=\sum_{j=1}^q\frac{\alpha_j}{\sqrt{\sigma_j}}e_j,
\notag
\end{align}
and every other preimage differs from $f_a$ by a vector in $M^\perp$.  Orthogonality therefore gives
\begin{align}
\inf_{Sf=a}\norm f_H^2
=\norm{f_a}_H^2
=\sum_{j=1}^q\frac{\alpha_j^2}{\sigma_j}
=a^\top K^\dagger a.
\notag
\end{align}
This proves the pseudoinverse description in \eqref{eq:ellipsoid}.

We next pass from realized ladders to the closed covariance family.
Let $\sK_L^0(\bx)$ denote the set of Gram matrices of actual canonical ladders before taking the closure in \eqref{eq:kernel-family}.  If $K\in\sK_L^0(\bx)$ comes from an actual ladder, every function in that ladder's unit RKHS ball has adaptive complexity at most one.  The preceding RKHS unit-ball calculation therefore gives
\begin{align}
\set{f(\bx):f\in B_L(1)}
\supseteq
\bigcup_{K\in\sK_L^0(\bx)}E(K).
\notag
\end{align}
Conversely, if $f\in B_L(1)$ and $\eta>0$, the definition of $\Chat_L$ provides an actual ladder with top RKHS $H_\eta$, Gram matrix $K_\eta\in\sK_L^0(\bx)$, and
$\norm f_{H_\eta}\le1+\eta$.  Applying the same RKHS unit-ball calculation to $H_\eta$ then gives
$f(\bx)/(1+\eta)\in E(K_\eta)$.  Along any sequence $\eta_j\downarrow0$, these vectors converge to $f(\bx)$, so $f(\bx)$ belongs to the closure of the union of the actual-ladder ellipsoids.  Hence
\begin{align}
\cA_L(\bx)=\cl\bigcup_{K\in\sK_L^0(\bx)}E(K).
\notag
\end{align}

It remains to replace $\sK_L^0$ by its closure.  If $K_j\to K$ in the finite-dimensional symmetric-matrix space, continuity of the positive-semidefinite square-root map gives
$K_j^{1/2}\to K^{1/2}$ in operator norm.  Hence every vector $K^{1/2}u$ with $\norm u_2\le1$ is the limit of $K_j^{1/2}u\in E(K_j)$.  Therefore
\begin{align}
\cl\bigcup_{K\in\sK_L^0(\bx)}E(K)
=\cl\bigcup_{K\in\sK_L(\bx)}E(K),
\notag
\end{align}
which proves \eqref{eq:trace-union}.

We then compute the support function.
For every $K\succeq0$,
\begin{align}
\sup_{a\in E(K)}z^\top a
=\sup_{\norm u_2\le1}z^\top K^{1/2}u
=\norm{K^{1/2}z}_2
=\sqrt{z^\top Kz}.
\notag
\end{align}
A continuous linear functional has the same supremum on a set and on its closure.  Applying this observation to \eqref{eq:trace-union} gives
\begin{align}
h_{L,\bx}(z)=\sup_{K\in\sK_L(\bx)}\sqrt{z^\top Kz},
\notag
\end{align}
and squaring proves \eqref{eq:support-covariance}.

Finally, we restore the radius and average over the Gaussian vector.
By \eqref{eq:ball-scaling}, the closed trace body of $B_L(r)$ is $r\cA_L(\bx)$.  Therefore, for every fixed $G$,
\begin{align}
\sup_{f\in B_L(r)}\sum_{i=1}^nG_i f(x_i)=r h_{L,\bx}(G).
\notag
\end{align}
Divide by $n$ and take expectation with respect to $G\sim\mathcal N(0,I_n)$ to obtain \eqref{eq:empirical-width-support}.
\end{proof}

\section{Proofs for Absolutely Two-Summing Duality}\label[appendix]{app:pi2}

We first treat a finite family.

\begin{lemma}[Finite family and $2$-summing norm]\label{lem:finite-pi2}
Let
\begin{align}
h_{\mathcal F}(z)^2:=\max_{1\le j\le M}z^\top K_jz.
\notag
\end{align}
Then
\begin{equation}\label{eq:finite-pi2-equality}
\tau(K_1,\ldots,K_M)
=\sup_{\sum_rz_rz_r^\top\preceq I}
\sum_r h_{\mathcal F}(z_r)^2.
\end{equation}
\end{lemma}

\begin{proof}
We first derive the finite SDP duality used in the argument.  With multipliers $Z_j\succeq0$, the Lagrangian of \eqref{eq:finite-primal} is
\begin{align}
\mathcal L(N,Z_1,\ldots,Z_M)
&=\tr N-\sum_{j=1}^M\ip{Z_j}{N-K_j}\notag\\
&=\ip{I-\sum_{j=1}^MZ_j}{N}+\sum_{j=1}^M\ip{Z_j}{K_j}.
\notag
\end{align}
The infimum over symmetric $N$ is finite exactly when $\sum_{j=1}^MZ_j=I$.  This gives the dual program \eqref{eq:finite-dual}.  Slater's condition holds because $N=tI$ is strictly feasible whenever $t>\max_j\lambda_{\max}(K_j)$, so the finite primal and dual values coincide.

Write $Q$ for the supremum on the right side of \eqref{eq:finite-pi2-equality}.  We first show that the dual value is at most $Q$.  Let $(Z_j)_{j=1}^M$ be dual feasible.  By the spectral theorem, each positive semidefinite matrix has a finite rank-one decomposition
\begin{align}
Z_j=\sum_{r=1}^{r_j}z_{jr}z_{jr}^\top.
\notag
\end{align}
Dual feasibility gives
\begin{align}
\sum_{j=1}^M\sum_{r=1}^{r_j}z_{jr}z_{jr}^\top
=\sum_{j=1}^MZ_j
=I.
\notag
\end{align}
Thus the combined finite sequence $(z_{jr})_{j,r}$ is admissible in the definition of $Q$.  Moreover,
\begin{align*}
\sum_{j=1}^M\ip{Z_j}{K_j}
&=\sum_{j=1}^M\sum_{r=1}^{r_j}\ip{z_{jr}z_{jr}^\top}{K_j}\\
&=\sum_{j=1}^M\sum_{r=1}^{r_j}z_{jr}^\top K_jz_{jr}\\
&\le\sum_{j=1}^M\sum_{r=1}^{r_j}h_{\mathcal F}(z_{jr})^2\\
&\le Q.
\end{align*}
Taking the supremum over dual-feasible collections proves that the dual value is at most $Q$.

We now prove the reverse inequality.  Take a finite sequence $z_1,\ldots,z_m$ satisfying
\begin{equation*}
\sum_{r=1}^m z_rz_r^\top\preceq I.
\end{equation*}
Since the kernel family is finite, for every $r$ there exists an index $j(r)$ such that
\begin{align}
h_{\mathcal F}(z_r)^2=z_r^\top K_{j(r)}z_r.
\notag
\end{align}
Set
\begin{align}
Z_j^0:=\sum_{r:j(r)=j}z_rz_r^\top,
\qquad
S:=I-\sum_{j=1}^MZ_j^0\succeq0.
\notag
\end{align}
Choose any fixed index $j_0\in[M]$ and define
$Z_{j_0}:=Z_{j_0}^0+S$ and $Z_j:=Z_j^0$ for $j\ne j_0$.  Then every $Z_j\succeq0$ and $\sum_jZ_j=I$, so the collection is dual feasible.  Since $S\succeq0$ and $K_{j_0}\succeq0$, $\ip{S}{K_{j_0}}\ge0$.  Hence
\begin{align*}
\sum_{j=1}^M\ip{Z_j}{K_j}
&\ge\sum_{j=1}^M\ip{Z_j^0}{K_j}\\
&=\sum_{r=1}^m z_r^\top K_{j(r)}z_r\\
&=\sum_{r=1}^m h_{\mathcal F}(z_r)^2.
\end{align*}
Taking the supremum over all admissible sequences shows that $Q$ is at most the dual value.  The two inequalities, together with the strong duality established at the beginning of the proof, give \eqref{eq:finite-pi2-equality}.
\end{proof}

\begin{proof}[Proof of \Cref{thm:pi2}]
Put
\begin{align}
Q_L:=\pi_2(T_{L,\bx})^2
=\sup_{\sum_rz_rz_r^\top\preceq I}
\sum_rh_{L,\bx}(z_r)^2,
\notag
\end{align}
where the equality follows from \eqref{eq:T-norm} and \eqref{eq:pi2-def}.  Also set
\begin{align}
t_*:=\sup_{\mathcal F\subset\sK_L(\bx),\ |\mathcal F|<\infty}\tau(\mathcal F).
\notag
\end{align}
We prove $Q_L=t_*$.

First fix a finite subfamily $\mathcal F\subset\sK_L(\bx)$.  Its support function $h_{\mathcal F}$ satisfies $h_{\mathcal F}(z)\le h_{L,\bx}(z)$ for every $z$.  Therefore \Cref{lem:finite-pi2} gives
\begin{align}
\tau(\mathcal F)
=\sup_{\sum_rz_rz_r^\top\preceq I}\sum_rh_{\mathcal F}(z_r)^2
\le Q_L.
\notag
\end{align}
Taking the supremum over finite $\mathcal F$ yields $t_*\le Q_L$.

For the reverse inequality, fix an admissible finite sequence $z_1,\ldots,z_m$ with $m\ge1$ and fix $\eta>0$.  By \eqref{eq:support-covariance}, for each $r$ there is a kernel $K_r\in\sK_L(\bx)$ such that
\begin{align}
z_r^\top K_rz_r
\ge h_{L,\bx}(z_r)^2-\frac{\eta}{m}.
\notag
\end{align}
Let $\mathcal F:=\set{K_1,\ldots,K_m}$.  Then
\begin{align*}
\sum_{r=1}^m h_{L,\bx}(z_r)^2
&\le\eta+\sum_{r=1}^m z_r^\top K_rz_r\\
&\le\eta+\sum_{r=1}^m h_{\mathcal F}(z_r)^2\\
&\le\eta+\tau(\mathcal F)\\
&\le\eta+t_*.
\end{align*}
The third inequality is \Cref{lem:finite-pi2}.  Taking the supremum over admissible sequences and then letting $\eta\downarrow0$ gives $Q_L\le t_*$.  Hence $Q_L=t_*$.

It remains to identify $t_*$ with the full compact-family value $\tau_L(\bx)$.  Every matrix feasible for the full family is feasible for each finite subfamily, so $t_*\le\tau_L(\bx)$.  Fix $\eta>0$ and define
\begin{align}
\mathcal C
:=\set{N\succeq0:\tr N\le t_*+\eta}.
\notag
\end{align}
This set is compact: it is closed, and for $N\succeq0$ one has $\norm{N}_{\mathrm F}\le\tr N$.  For each $K\in\sK_L(\bx)$, let
\begin{align}
\mathcal C_K
:=\set{N\in\mathcal C:N\succeq K}.
\notag
\end{align}
Every $\mathcal C_K$ is closed in $\mathcal C$.  If $\mathcal F\subset\sK_L(\bx)$ is finite, then $\tau(\mathcal F)\le t_*$, so there is a matrix in $\bigcap_{K\in\mathcal F}\mathcal C_K$.  Thus the family $(\mathcal C_K)_{K\in\sK_L(\bx)}$ has the finite-intersection property.  Compactness of $\mathcal C$ gives
\begin{align}
\bigcap_{K\in\sK_L(\bx)}\mathcal C_K
\ne\varnothing.
\notag
\end{align}
Hence $\tau_L(\bx)\le t_*+\eta$.  Letting $\eta\downarrow0$ proves $\tau_L(\bx)=t_*=Q_L$.
\end{proof}

\section{Proofs for Robust Multiplier Geometry and Gaussian Defect}\label[appendix]{app:minimax}

\subsection{Proof of the Robust Multiplier Minimax Identity}

\begin{proof}[Proof of \Cref{thm:robust-minimax}]
Write $h:=h_{L,\bx}$ and $\tau:=\tau_L(\bx)$.

We first establish the upper bound.  Let $N\succeq0$ be a feasible covariance envelope, so that
\begin{align}
N\succeq K
\qquad
\text{for every }K\in\sK_L(\bx).
\notag
\end{align}
By \eqref{eq:support-covariance}, for every $z\in\R^n$,
\begin{align}
h(z)^2
=
\sup_{K\in\sK_L(\bx)}z^\top Kz
\le
z^\top Nz.
\notag
\end{align}
Let $\nu\in\mathfrak P_{\preceq I_n}$ and put
\begin{align}
\Sigma_\nu:=\int zz^\top\,d\nu(z).
\notag
\end{align}
The second-moment assumption in \eqref{eq:subisotropic-laws} guarantees that all entries of $\Sigma_\nu$ are finite.  Integrating the preceding pointwise inequality gives
\begin{align*}
\int h(z)^2\,d\nu(z)
&\le
\int z^\top Nz\,d\nu(z)\\
&=
\int \tr(Nzz^\top)\,d\nu(z)\\
&=
\tr\!\bigl(N\int zz^\top\,d\nu(z)\bigr)\\
&=
\tr(N\Sigma_\nu).
\end{align*}
Because $0\preceq\Sigma_\nu\preceq I_n$ and $N\succeq0$,
\begin{align}
\tr(N\Sigma_\nu)
\le
\tr N.
\notag
\end{align}
Indeed,
\begin{align}
\tr N-\tr(N\Sigma_\nu)
=
\tr\!\bigl(N^{1/2}(I_n-\Sigma_\nu)N^{1/2}\bigr)
\ge0.
\notag
\end{align}
Therefore
\begin{align}
\int h(z)^2\,d\nu(z)
\le
\tr N.
\notag
\end{align}
Taking the infimum over feasible $N$ and then the supremum over $\nu$ yields
\begin{equation}\label{eq:minimax-proof-upper}
\sup_{\nu\in\mathfrak P_{\preceq I_n}}
\int h(z)^2\,d\nu(z)
\le
\tau.
\end{equation}

We next show that every admissible $2$-summing sequence produces a covariance-dominated probability law.
Let $z_1,\ldots,z_m\in\R^n$ satisfy
\begin{align}
\sum_{j=1}^m z_jz_j^\top\preceq I_n.
\notag
\end{align}
Define
\begin{equation}\label{eq:discrete-multiplier-law}
\nu_z
:=
\frac1{2m}
\sum_{j=1}^m
\bigl(
\delta_{\sqrt m z_j}
+
\delta_{-\sqrt m z_j}
\bigr).
\end{equation}
This is a symmetric probability measure.  Its covariance is
\begin{align*}
\int zz^\top\,d\nu_z(z)
&=
\frac1{2m}
\sum_{j=1}^m
\bigl(
(\sqrt m z_j)(\sqrt m z_j)^\top
+
(-\sqrt m z_j)(-\sqrt m z_j)^\top
\bigr)\\
&=
\sum_{j=1}^m z_jz_j^\top\\
&\preceq I_n.
\end{align*}
Thus $\nu_z\in\mathfrak P_{\preceq I_n}$.  Since $h$ is absolutely homogeneous,
\begin{align*}
\int h(z)^2\,d\nu_z(z)
&=
\frac1{2m}
\sum_{j=1}^m
\bigl(
h(\sqrt m z_j)^2+h(-\sqrt m z_j)^2
\bigr)\\
&=
\frac1{2m}
\sum_{j=1}^m
\bigl(
mh(z_j)^2+mh(z_j)^2
\bigr)\\
&=
\sum_{j=1}^m h(z_j)^2.
\end{align*}

We then use the exact $2$-summing identity.
By \eqref{eq:T-norm},
\begin{align}
h(z_j)=\norm{T_{L,\bx}z_j}_{X_{L,\bx}^*}.
\notag
\end{align}
Therefore \Cref{thm:pi2} and \eqref{eq:pi2-def} give
\begin{align}
\tau
=
\sup_{m\ge1}
\sup_{\sum_jz_jz_j^\top\preceq I_n}
\sum_{j=1}^m h(z_j)^2.
\notag
\end{align}
Every value in this supremum is realized by a law of the form \eqref{eq:discrete-multiplier-law}.  Hence
\begin{equation}\label{eq:minimax-proof-lower}
\tau
\le
\sup_{\nu\in\mathfrak P_{\preceq I_n}}
\int h(z)^2\,d\nu(z).
\end{equation}
Combining \eqref{eq:minimax-proof-upper} and \eqref{eq:minimax-proof-lower} proves \eqref{eq:robust-minimax}.  Taking square roots and multiplying by $r/n$ proves \eqref{eq:robust-minimax-scaled}.

\end{proof}

\begin{proof}[Proof of \Cref{ex:coordinate-kernels}]
For every $z\in\R^n$,
\begin{align}
\sup_{K\in\mathscr K}z^\top Kz
=
\max_{1\le i\le n}z_i^2,
\notag
\end{align}
so $h(z)=\norm z_\infty$.  If $N$ dominates every $e_ie_i^\top$, then
\begin{align}
N_{ii}=e_i^\top Ne_i\ge1
\qquad
\text{for every }i.
\notag
\end{align}
Thus $\tr N\ge n$.  The choice $N=I_n$ is feasible, so $\tau(\mathscr K)=n$.

It remains to prove \eqref{eq:coordinate-gaussian-width}.  Let $\lambda>0$.  Since
\begin{align}
e^{\lambda\max_i|G_i|}
\le
\sum_{i=1}^n e^{\lambda|G_i|},
\notag
\end{align}
Jensen's inequality for the logarithm gives
\begin{align*}
\E\max_i|G_i|
&=
\frac1\lambda
\E\log e^{\lambda\max_i|G_i|}\\
&\le
\frac1\lambda
\log\E e^{\lambda\max_i|G_i|}\\
&\le
\frac1\lambda
\log\bigl(
\sum_{i=1}^n\E e^{\lambda|G_i|}
\bigr).
\end{align*}
For a standard Gaussian variable $g$,
\begin{align}
e^{\lambda|g|}
\le
e^{\lambda g}+e^{-\lambda g},
\notag
\end{align}
so
\begin{align}
\E e^{\lambda|g|}
\le
2e^{\lambda^2/2}.
\notag
\end{align}
Consequently,
\begin{align}
\E\max_i|G_i|
\le
\frac{\log(2n)}\lambda+\frac\lambda2.
\notag
\end{align}
Choosing $\lambda=\sqrt{2\log(2n)}$ gives \eqref{eq:coordinate-gaussian-width}.
\end{proof}

\subsection{Proof of the Gaussian Cotype Reverse}

\begin{lemma}[First and second moments of a Gaussian seminorm]\label{lem:gaussian-moment-comparison}
Let $p:\R^n\to[0,\infty)$ be a continuous seminorm and let $G\sim\mathcal N(0,I_n)$.  Then
\begin{equation}\label{eq:seminorm-moment-comparison}
\bigl(\E p(G)^2\bigr)^{1/2}
\le
\sqrt{1+\frac\pi2}\,\E p(G).
\end{equation}
\end{lemma}

\begin{proof}
Set
\begin{align}
\sigma
:=
\sup_{\norm z_2\le1}p(z).
\notag
\end{align}
If $\sigma=0$, then $p=0$ and the claim is immediate.  Assume $\sigma>0$.

We first identify the Lipschitz constant.
For $x,y\in\R^n$, the triangle inequality for $p$ gives
\begin{align}
p(x)-p(y)\le p(x-y),
\qquad
p(y)-p(x)\le p(x-y).
\notag
\end{align}
Hence
\begin{align}
|p(x)-p(y)|
\le
p(x-y)
\le
\sigma\norm{x-y}_2.
\notag
\end{align}
Thus $p$ is $\sigma$-Lipschitz.

We next establish the variance bound.
The Gaussian Poincar\'e inequality states that every locally Lipschitz function $F$ with square-integrable gradient satisfies
\begin{align}
\operatorname{Var}(F(G))
\le
\E\norm{\nabla F(G)}_2^2.
\notag
\end{align}
A Lipschitz function is differentiable almost everywhere, and the gradient norm is bounded by its Lipschitz constant at every differentiability point.  Applying the inequality to $F=p$ therefore gives
\begin{equation}\label{eq:seminorm-poincare}
\operatorname{Var}(p(G))
\le
\sigma^2.
\end{equation}
We use the Gaussian Poincar\'e inequality only in this standard finite-dimensional form; see, for example, \cite{ledoux1991probability}.

We then lower-bound the first moment by the Lipschitz constant.
Define the polar set
\begin{align}
C:=\set{a\in\R^n:a^\top z\le p(z)\text{ for every }z\in\R^n}.
\notag
\end{align}
The set $C$ is closed and symmetric.  Moreover, for $a\in C$ and $\norm z_2\le1$,
\begin{align}
a^\top z\le p(z)\le\sigma,
\notag
\end{align}
so $\norm a_2\le\sigma$.  Hence $C$ is compact.  The finite-dimensional Hahn--Banach theorem gives the support representation
\begin{align}
p(z)=\max_{a\in C}a^\top z
\qquad
\text{for every }z\in\R^n.
\notag
\end{align}
Consequently,
\begin{align*}
\max_{a\in C}\norm a_2
&=
\max_{a\in C}\sup_{\norm z_2\le1}a^\top z\\
&=
\sup_{\norm z_2\le1}\max_{a\in C}a^\top z\\
&=
\sup_{\norm z_2\le1}p(z)\\
&=
\sigma.
\end{align*}
Choose $a\in C$ with $\norm a_2=\sigma$.  Symmetry of $C$ gives $p(z)\ge|a^\top z|$ for every $z$.  Therefore
\begin{equation}\label{eq:seminorm-first-lower}
\E p(G)
\ge
\E|a^\top G|
=
\sqrt{\frac2\pi}\norm a_2
=
\sqrt{\frac2\pi}\sigma.
\end{equation}

Finally, we combine the two moment estimates.
Using \eqref{eq:seminorm-poincare} and \eqref{eq:seminorm-first-lower},
\begin{align*}
\E p(G)^2
&=
\bigl(\E p(G)\bigr)^2
+
\operatorname{Var}(p(G))\\
&\le
\bigl(\E p(G)\bigr)^2+\sigma^2\\
&\le
\bigl(1+\frac\pi2\bigr)
\bigl(\E p(G)\bigr)^2.
\end{align*}
Taking square roots proves \eqref{eq:seminorm-moment-comparison}.
\end{proof}

\begin{proof}[Proof of \Cref{thm:cotype-reverse}]
If $T_{L,\bx}=0$, then \Cref{thm:pi2} gives $\tau_L(\bx)=0$, and all stated inequalities are immediate.  Assume $T_{L,\bx}\ne0$ and abbreviate
\begin{align}
T:=T_{L,\bx},
\qquad
h:=h_{L,\bx},
\qquad
C:=C_{2,L}^{\mathrm g}(\bx).
\notag
\end{align}

We begin with an arbitrary admissible $2$-summing sequence.
Let $z_1,\ldots,z_m\in\R^n$ satisfy
\begin{align}
\sum_{j=1}^m z_jz_j^\top\preceq I_n.
\notag
\end{align}
Define $A:\R^m\to\R^n$ by
\begin{align}
Ae_j=z_j,
\qquad j=1,\ldots,m.
\notag
\end{align}
Then
\begin{align}
AA^\top
=
\sum_{j=1}^m z_jz_j^\top
\preceq I_n.
\notag
\end{align}
Apply \eqref{eq:gaussian-cotype-def} to $y_j:=Tz_j$.  By \eqref{eq:T-norm},
\begin{align}
\bigl(\sum_{j=1}^m h(z_j)^2\bigr)^{1/2}
&=
\bigl(\sum_{j=1}^m\norm{Tz_j}_{X_{L,\bx}^*}^2\bigr)^{1/2}\notag\\
&\le
C
\bigl(
\E_\gamma
\norm{\sum_{j=1}^m\gamma_jTz_j}_{X_{L,\bx}^*}^2
\bigr)^{1/2}\notag\\
&=
C
\bigl(
\E_\gamma h(A\gamma)^2
\bigr)^{1/2}.
\label{eq:cotype-sequence-bound}
\end{align}

We next compare a contracted Gaussian with a standard Gaussian.
Set
\begin{align}
B:=(I_n-AA^\top)^{1/2}.
\notag
\end{align}
Let $\gamma'\sim\mathcal N(0,I_n)$ be independent of $\gamma$.  The vector
\begin{align}
A\gamma+B\gamma'
\notag
\end{align}
is centered Gaussian with covariance
\begin{align}
AA^\top+BB^\top
=
AA^\top+I_n-AA^\top
=I_n.
\notag
\end{align}
Therefore
\begin{align}
A\gamma+B\gamma'
\sim\mathcal N(0,I_n).
\notag
\end{align}
The map $z\mapsto h(z)^2$ is convex because $h$ is a seminorm and the scalar map $u\mapsto u^2$ is convex and nondecreasing on $[0,\infty)$.  Conditional on $\gamma$,
\begin{align}
\E_{\gamma'}[A\gamma+B\gamma'\mid\gamma]
=A\gamma.
\notag
\end{align}
Jensen's inequality therefore gives
\begin{align}
h(A\gamma)^2
\le
\E_{\gamma'}h(A\gamma+B\gamma')^2.
\notag
\end{align}
Taking expectation over $\gamma$ yields
\begin{equation}\label{eq:contracted-gaussian}
\E_\gamma h(A\gamma)^2
\le
\E_Gh(G)^2
=
W_{G,2}(L,\bx)^2.
\end{equation}

We then take the $2$-summing supremum.
Substituting \eqref{eq:contracted-gaussian} into \eqref{eq:cotype-sequence-bound} gives
\begin{align}
\bigl(\sum_{j=1}^m h(z_j)^2\bigr)^{1/2}
\le
C W_{G,2}(L,\bx).
\notag
\end{align}
Taking the supremum over all admissible sequences and using \Cref{thm:pi2} proves
\begin{align}
\sqrt{\tau_L(\bx)}
\le
C_{2,L}^{\mathrm g}(\bx)W_{G,2}(L,\bx),
\notag
\end{align}
which is \eqref{eq:cotype-rms-reverse}.

Finally, we pass from the second Gaussian moment to the first.
Apply \Cref{lem:gaussian-moment-comparison} to the seminorm $h$.  This gives \eqref{eq:gaussian-moment-comparison}.  Combining it with \eqref{eq:cotype-rms-reverse} yields
\begin{align}
W_{G,1}(L,\bx)
\ge
\frac{\sqrt{\tau_L(\bx)}}
{\kappa_{\mathrm G}C_{2,L}^{\mathrm g}(\bx)}.
\notag
\end{align}
Multiplying by $r/n$ and using \eqref{eq:complexity-WG1} gives the lower bound in \eqref{eq:cotype-complexity-sandwich}.  The upper bound is \Cref{thm:profile-complexity}.
\end{proof}

\begin{proof}[Proof of \Cref{cor:gaussian-defect}]
If $\tau_L(\bx)=0$, then $h_{L,\bx}=0$ by \eqref{eq:support-covariance}, and the convention $\chi_L(\bx)=1$ gives \eqref{eq:defect-range}.  Assume $\tau_L(\bx)>0$.

The pointwise bound
\begin{align}
h_{L,\bx}(G)^2
\le
G^\top N_*G
\notag
\end{align}
for an optimal common covariance $N_*$ gives
\begin{align}
W_{G,2}(L,\bx)^2
\le
\E[G^\top N_*G]
=
\tr N_*
=
\tau_L(\bx).
\notag
\end{align}
Thus $\chi_L(\bx)\ge1$.  The upper bound in \eqref{eq:defect-range} is \eqref{eq:cotype-rms-reverse}.  Finally, substitute \eqref{eq:robust-minimax} into the numerator of \eqref{eq:gaussian-defect}; this gives \eqref{eq:defect-minimax-ratio}.
\end{proof}

\section{Proofs for Brownian Recursive Structure}\label[appendix]{app:brownian}

\begin{proof}[Proof of \Cref{thm:dirac-reduction}]
Let $\sS_{L-1}^0(\bx)$ denote the traces of actual unit-sphere atoms before taking the closure that defines $\sS_{L-1}(\bx)$.  For a fixed previous-level RKHS and a probability measure $\mu$ on its unit sphere, the empirical top kernel is
\begin{align}
K_\mu=\int\mathbf B(a(u))\,d\mu(u),
\qquad
a(u)=(u(x_1),\ldots,u(x_n))\in\sS_{L-1}^0(\bx).
\notag
\end{align}
Therefore, for every $z\in\R^n$,
\begin{align}
z^\top K_\mu z
=\int z^\top\mathbf B(a(u))z\,d\mu(u)
\le\sup_{a\in\sS_{L-1}^0(\bx)}z^\top\mathbf B(a)z.
\notag
\end{align}
Conversely, every actual unit atom $u$ generates the admissible Dirac measure $\delta_u$, whose Gram matrix is $\mathbf B(a(u))$.  Thus the supremum over all top kernels equals the supremum over $a\in\sS_{L-1}^0(\bx)$.  The map $a\mapsto\mathbf B(a)$ is continuous entrywise, so taking closures does not change this supremum and proves \eqref{eq:dirac-support}.

For domination, if $N\succeq\mathbf B(a)$ for every $a\in\sS_{L-1}(\bx)$, then integration gives $N\succeq K_\mu$ for every probability mixture and closure gives domination of every $K\in\sK_L(\bx)$.  Conversely, if $N$ dominates every $K\in\sK_L(\bx)$, it dominates every actual Dirac matrix; continuity and closedness of the positive-semidefinite cone then extend the domination to all $a\in\sS_{L-1}(\bx)$.  This proves \eqref{eq:dirac-domination}.
\end{proof}

\begin{proof}[Proof of \Cref{prop:threshold-certificate}]
We first verify the signed layer-cake identity underlying \eqref{eq:threshold-matrix}.  For $s,t\in\R$,
\begin{align}
\kB(s,t)
=\int_0^\infty
\bigl[
\mathbf1_{\{s\ge u\}}\mathbf1_{\{t\ge u\}}
+\mathbf1_{\{s\le-u\}}\mathbf1_{\{t\le-u\}}
\bigr]du.
\notag
\end{align}
If $s$ and $t$ have opposite signs, both products vanish for every $u>0$, while $|s-t|=|s|+|t|$, so both sides are zero.  If $s,t\ge0$, only the first product can be nonzero, and it equals one exactly for $0\le u\le\min\{s,t\}$.  The integral is therefore $\min\{s,t\}=(s+t-|s-t|)/2=\kB(s,t)$.  If $s,t\le0$, only the second product can be nonzero, and its integral is $\min\{|s|,|t|\}=\kB(s,t)$.  Applying this scalar identity to $(a_i,a_j)$ for every pair of indices and collecting the indicators gives \eqref{eq:threshold-matrix}.

By assumption, for almost every $t$ and every admissible $a$,
\begin{align}
v_t^+(a)v_t^+(a)^\top+v_t^-(a)v_t^-(a)^\top\preceq2D_t.
\notag
\end{align}
Integrating and using \eqref{eq:threshold-matrix} gives
\begin{align}
\mathbf B(a)\preceq2\int_0^\infty D_t\,dt.
\notag
\end{align}
The right side is therefore feasible in the Dirac formulation \eqref{eq:dirac-domination}; taking traces proves the claim.
\end{proof}

\subsection{Proof of the Recursive Threshold-Factorization Theorem}\label{app:recursive-threshold}

\begin{proof}[Proof of \Cref{thm:recursive-threshold}]
We prove the upper and lower inequalities in \eqref{eq:recursive-sandwich} separately.  Throughout the proof, set
\begin{equation}\label{eq:recursive-pi}
p_i:=d_{L,i}=\rho_{L-1}(x_i)
=\norm{x_i}_2^{\,2^{-(L-2)}},
\qquad i\in[n].
\end{equation}
By the previous-layer pointwise estimate, every $a=(a_1,\ldots,a_n)\in\sS_{L-1}(\bx)$ satisfies
\begin{equation}\label{eq:recursive-pointwise-a}
|a_i|\le p_i,
\qquad i\in[n].
\end{equation}
Also, by definition,
\begin{equation}\label{eq:recursive-S-layercake}
S_L(\bx)=\sum_{i=1}^n p_i.
\end{equation}

We first construct one covariance for every threshold set from a decomposition dictionary.
Fix an arbitrary decomposition dictionary $\mathscr D$ for $\mathcal F_{L-1}(\bx)$ and abbreviate
\begin{align}
s:=s(\mathscr D),
\qquad
\Delta:=\Delta(\mathscr D).
\notag
\end{align}
Define
\begin{equation}\label{eq:HD-def}
H_{\mathscr D}
:=s\sum_{D\in\mathscr D}\bone_D\bone_D^\top.
\end{equation}
Each summand in \eqref{eq:HD-def} is positive semidefinite, hence $H_{\mathscr D}\succeq0$.

Let $A\in\mathcal F_{L-1}(\bx)$.  By the dictionary property, there are pairwise disjoint $D_1,\ldots,D_m\in\mathscr D$ with $m\le s$ and
\begin{equation}\label{eq:A-dict-vector}
\bone_A=\sum_{j=1}^m\bone_{D_j}.
\end{equation}
If $A=\varnothing$, then $\bone_A\bone_A^\top=0\preceq H_{\mathscr D}$ because $H_{\mathscr D}\succeq0$.  Assume $A\ne\varnothing$.  For arbitrary $z\in\R^n$, \eqref{eq:A-dict-vector} gives
\begin{align}
z^\top\bone_A\bone_A^\top z
&=(z^\top\bone_A)^2
\nonumber\\
&=\bigl(\sum_{j=1}^m z^\top\bone_{D_j}\bigr)^2
\nonumber\\
&\le m\sum_{j=1}^m(z^\top\bone_{D_j})^2
\label{eq:dict-CS}\\
&\le s\sum_{D\in\mathscr D}(z^\top\bone_D)^2
\nonumber\\
&=z^\top H_{\mathscr D}z.
\label{eq:dict-quad}
\end{align}
The inequality in \eqref{eq:dict-CS} is Cauchy--Schwarz. The next step uses $m\le s$ and enlarges the nonnegative sum from the selected sets $D_j$ to all $D\in\mathscr D$. Since \eqref{eq:dict-quad} holds for every $z\in\R^n$,
\begin{equation}\label{eq:dict-rank-one-domination}
\bone_A\bone_A^\top\preceq H_{\mathscr D}
\qquad
\text{for every }A\in\mathcal F_{L-1}(\bx).
\end{equation}

We next localize the dictionary covariance to the coordinates that can survive level $t$.
For $t\ge0$, define
\begin{equation}\label{eq:I-t-def}
I_t:=\set{i\in[n]:p_i\ge t},
\qquad
P_t:=\diag(\bone_{I_t}).
\end{equation}
The matrix $P_t$ is an orthogonal coordinate projection.  Define
\begin{equation}\label{eq:Ht-def}
H_t:=P_tH_{\mathscr D}P_t.
\end{equation}

Fix $a\in\sS_{L-1}(\bx)$.  If $i\in A_t^+(a):=\set{i:a_i\ge t}$, then $t\le a_i\le|a_i|\le p_i$ by \eqref{eq:recursive-pointwise-a}; hence $i\in I_t$.  Therefore
\begin{equation}\label{eq:plus-subset-It}
A_t^+(a)\subseteq I_t.
\end{equation}
Likewise, if $i\in A_t^-(a):=\set{i:a_i\le-t}$, then $t\le-a_i\le|a_i|\le p_i$, and hence
\begin{equation}\label{eq:minus-subset-It}
A_t^-(a)\subseteq I_t.
\end{equation}

Let $A$ denote either $A_t^+(a)$ or $A_t^-(a)$.  By \eqref{eq:plus-subset-It}--\eqref{eq:minus-subset-It}, $P_t\bone_A=\bone_A$.  Applying the congruence map $M\mapsto P_tMP_t$ to the positive semidefinite matrix $H_{\mathscr D}-\bone_A\bone_A^\top$ from \eqref{eq:dict-rank-one-domination} yields
\begin{equation}\label{eq:localized-domination}
\bone_A\bone_A^\top
=P_t\bone_A\bone_A^\top P_t
\preceq
P_tH_{\mathscr D}P_t
=H_t.
\end{equation}
Consequently,
\begin{equation}\label{eq:two-threshold-domination}
v_t^+(a)v_t^+(a)^\top
+v_t^-(a)v_t^-(a)^\top
\preceq2H_t.
\end{equation}

We then integrate the localized covariances.
Define
\begin{equation}\label{eq:N-D-def}
N_{\mathscr D}:=2\int_0^\infty H_t\,dt.
\end{equation}
The integral is finite.  Indeed, if $t>\max_i p_i$, then $I_t=\varnothing$, so $P_t=0$ and $H_t=0$.  Thus the integral in \eqref{eq:N-D-def} is over the finite interval $[0,\max_i p_i]$.

Using the Brownian layer-cake identity \eqref{eq:threshold-matrix} and \eqref{eq:two-threshold-domination}, for every $a\in\sS_{L-1}(\bx)$ we obtain
\begin{align}
\mathbf B(a)
&=\int_0^\infty
\bigl[
 v_t^+(a)v_t^+(a)^\top
 +v_t^-(a)v_t^-(a)^\top
\bigr]dt
\nonumber\\
&\preceq
2\int_0^\infty H_t\,dt
=N_{\mathscr D}.
\label{eq:B-dominated-ND}
\end{align}
By the Dirac-domination equivalence \eqref{eq:dirac-domination}, $N_{\mathscr D}$ is feasible for the minimum defining $\tau_L(\bx)$.  Therefore
\begin{equation}\label{eq:tau-le-trND}
\tau_L(\bx)\le\tr N_{\mathscr D}.
\end{equation}

We next compute the trace by a scalar layer-cake identity.
For each $i\in[n]$, the diagonal entry of $H_{\mathscr D}$ is
\begin{align}
(H_{\mathscr D})_{ii}
&=s\sum_{D\in\mathscr D}(\bone_D)_i^2
\nonumber\\
&=s\abs{\set{D\in\mathscr D:i\in D}}
\nonumber\\
&\le s\Delta.
\label{eq:HD-diagonal}
\end{align}
The first identity uses \eqref{eq:HD-def}; the second uses $(\bone_D)_i^2=(\bone_D)_i$; the inequality is the definition of $\Delta$.  Since $H_t=P_tH_{\mathscr D}P_t$,
\begin{align}
\tr H_t
&=\sum_{i\in I_t}(H_{\mathscr D})_{ii}
\nonumber\\
&\le s\Delta\,|I_t|.
\label{eq:Ht-trace}
\end{align}
Combining \eqref{eq:N-D-def} and \eqref{eq:Ht-trace},
\begin{align}
\tr N_{\mathscr D}
&=2\int_0^\infty\tr H_t\,dt
\nonumber\\
&\le2s\Delta\int_0^\infty|I_t|\,dt.
\label{eq:trND-prelayercake}
\end{align}
Now
\begin{align}
\int_0^\infty|I_t|\,dt
&=\int_0^\infty\sum_{i=1}^n\mathbf1_{\{p_i\ge t\}}\,dt
\nonumber\\
&=\sum_{i=1}^n\int_0^\infty\mathbf1_{\{p_i\ge t\}}\,dt
\nonumber\\
&=\sum_{i=1}^np_i
=S_L(\bx).
\label{eq:scalar-layercake-p}
\end{align}
The interchange of the finite sum and the integral is exact, and the third equality uses that the indicator is one precisely on the interval $[0,p_i]$ up to an endpoint of Lebesgue measure zero.  Substituting \eqref{eq:scalar-layercake-p} into \eqref{eq:trND-prelayercake} and then using \eqref{eq:tau-le-trND} gives
\begin{equation}\label{eq:tau-dictionary}
\tau_L(\bx)\le2s\Delta S_L(\bx).
\end{equation}
Since the dictionary $\mathscr D$ was arbitrary, minimizing the right side of \eqref{eq:tau-dictionary} over all decomposition dictionaries yields
\begin{equation}\label{eq:tau-D-final}
\tau_L(\bx)\le2\mathfrak D_{L-1}(\bx)S_L(\bx).
\end{equation}
Dividing by $S_L(\bx)>0$ and using \eqref{eq:lambda-def} proves the upper inequality in \eqref{eq:recursive-sandwich}.

We then use stable threshold sets to obtain a converse certificate.
Fix $t\ge0$ and $\delta>0$.  Take any $A\in\mathcal F_{L-1}^{+,\mathrm{st}}(t,\delta)$.  By definition, there exists $a\in\sS_{L-1}(\bx)$ such that $A_s^+(a)=A$ for every $s\in[t,t+\delta]$.  Every term in the threshold representation \eqref{eq:threshold-matrix} is positive semidefinite, so discarding all terms except the positive-threshold contribution on $[t,t+\delta]$ gives
\begin{align}
\mathbf B(a)
&\succeq
\int_t^{t+\delta}
 v_s^+(a)v_s^+(a)^\top\,ds
\nonumber\\
&=\int_t^{t+\delta}\bone_A\bone_A^\top\,ds
\nonumber\\
&=\delta\,\bone_A\bone_A^\top.
\label{eq:stable-positive-B}
\end{align}
If instead $A\in\mathcal F_{L-1}^{-,\mathrm{st}}(t,\delta)$, choose $a\in\sS_{L-1}(\bx)$ for which $A_s^-(a)=A$ on $[t,t+\delta]$.  Discarding every positive-semidefinite term in \eqref{eq:threshold-matrix} except the negative-threshold contribution on this interval gives
\begin{align}
\mathbf B(a)
&\succeq\int_t^{t+\delta}v_s^-(a)v_s^-(a)^\top\,ds\nonumber\\
&=\int_t^{t+\delta}\bone_A\bone_A^\top\,ds
=\delta\,\bone_A\bone_A^\top.
\label{eq:stable-negative-B}
\end{align}

Let $N$ be any matrix feasible in \eqref{eq:tau-def}.  By \eqref{eq:dirac-domination}, $N\succeq\mathbf B(a)$ for every $a\in\sS_{L-1}(\bx)$.  Combining this with \eqref{eq:stable-positive-B}--\eqref{eq:stable-negative-B} gives
\begin{equation}\label{eq:N-stable-dominates}
N\succeq\delta\,\bone_A\bone_A^\top
\qquad
\text{for every }A\in\mathcal F_{L-1}^{\mathrm{st}}(t,\delta).
\end{equation}
Therefore $N/\delta$ is feasible in the definition \eqref{eq:Theta-def} of $\Theta(\mathcal F_{L-1}^{\mathrm{st}}(t,\delta))$, and hence
\begin{equation}\label{eq:theta-lower-N}
\tr N
\ge
\delta\,\Theta\!\bigl(\mathcal F_{L-1}^{\mathrm{st}}(t,\delta)\bigr).
\end{equation}
Taking the infimum over all feasible $N$ in \eqref{eq:theta-lower-N} gives
\begin{equation}\label{eq:tau-stable-lower}
\tau_L(\bx)
\ge
\delta\,\Theta\!\bigl(\mathcal F_{L-1}^{\mathrm{st}}(t,\delta)\bigr).
\end{equation}
Since \eqref{eq:tau-stable-lower} holds for every $t\ge0$ and $\delta>0$,
\begin{equation}\label{eq:tau-R-lower}
\tau_L(\bx)
\ge
S_L(\bx)\mathfrak R_{L-1}(\bx).
\end{equation}
Dividing by $S_L(\bx)$ proves the lower inequality in \eqref{eq:recursive-sandwich}.

Finally, we derive the Gaussian complexity consequence.
Equation \eqref{eq:recursive-tau} is exactly \eqref{eq:tau-D-final}.  Substitute it into \eqref{eq:profile-bound-tau} to obtain
\begin{align}
\Ghat_{\bx}(B_L(r))
\le
\frac{r\sqrt{2\mathfrak D_{L-1}(\bx)}}{n}
\sqrt{S_L(\bx)},
\notag
\end{align}
which is \eqref{eq:recursive-G-sum} after inserting the definition \eqref{eq:d-and-S} of $S_L$.  Finally,
\begin{align}
S_L(\bx)
\le
n\max_i\norm{x_i}_2^{\,2^{-(L-2)}},
\notag
\end{align}
so
\begin{align}
\frac{\sqrt{S_L(\bx)}}{n}
\le
\frac1{\sqrt n}
\bigl(\max_i\norm{x_i}_2\bigr)^{2^{-(L-1)}}.
\notag
\end{align}
This proves \eqref{eq:recursive-G-radius} and completes the proof.
\end{proof}

\begin{proof}[Proof of \Cref{cor:hierarchical-threshold}]
If $S_L(\bx)=0$, then $\Lambda_L(\bx)=0$ and $\Ghat_{\bx}(B_L(r))=0$ by \Cref{prop:profile-existence,thm:profile-complexity}, so all claims hold.  Assume $S_L(\bx)>0$.  Let $\mathscr D$ be the set of all nodes of the rooted laminar partition tree.  By hypothesis, each $A\in\mathcal F_{L-1}(\bx)$ is a disjoint union of at most $q$ members of $\mathscr D$, so
\begin{equation}\label{eq:hierarchical-s-proof}
s(\mathscr D)\le q.
\end{equation}
Fix $i\in[n]$.  At each tree level there is at most one partition node containing $i$.  A tree of height $H$ has at most $H+1$ levels, including the root level.  Therefore
\begin{equation}\label{eq:hierarchical-Delta-proof}
\Delta(\mathscr D)\le H+1.
\end{equation}
Multiplying \eqref{eq:hierarchical-s-proof} and \eqref{eq:hierarchical-Delta-proof} and applying \eqref{eq:D-profile} gives
\begin{align}
\mathfrak D_{L-1}(\bx)\le q(H+1).
\notag
\end{align}
The bound for $\Lambda_L$ follows from the upper half of \eqref{eq:recursive-sandwich}, and \eqref{eq:hierarchical-G} follows from \eqref{eq:recursive-G-radius}.  If $H=O(\log n)$ and $q=O(1)$, the displayed factor is $O(\sqrt{\log n/n})$.
\end{proof}

\begin{proof}[Proof of \Cref{cor:ordered-threshold}]
If $S_L(\bx)=0$, then $\Lambda_L(\bx)=0$ and $\Ghat_{\bx}(B_L(r))=0$ by \Cref{prop:profile-existence,thm:profile-complexity}, so all claims hold.  Assume $S_L(\bx)>0$.  Relabel the fixed ordering as $1,\ldots,n$.  Put
\begin{align}
N:=2^{\lceil\log_2 n\rceil},
\notag
\end{align}
so $n\le N<2n$.  Pad the ordered ground set by the unused positions $n+1,\ldots,N$, and consider the complete binary interval tree on $[N]$.  Its nodes are the dyadic intervals
\begin{align}
\{(j-1)2^r+1,\ldots,j2^r\},
\qquad
r=0,\ldots,\log_2N,
\notag
\end{align}
with the admissible values of $j$.  Let $\mathscr D$ consist of the intersections of these nodes with $[n]$, after removing the empty intersections.

We prove the two properties of this dictionary that are used below.

\emph{Interval decomposition.}
Fix an interval $I=\{a,a+1,\ldots,b\}\subseteq[n]$.  Select all dyadic nodes that are contained in $I$ and maximal under inclusion.  These maximal nodes are pairwise disjoint and cover $I$: every singleton in $I$ is a dyadic node contained in $I$, and repeatedly replacing a selected node by its parent is possible until the parent would leave $I$.  At a fixed scale $2^r$, there are at most two maximal selected nodes.  Indeed, if three such nodes occurred from left to right, the middle node would have a dyadic sibling also contained in $I$; their parent would then be contained in $I$, contradicting maximality.  The tree has $1+\log_2N$ scales, and therefore $I$ is a disjoint union of at most
\begin{align}
2(1+\log_2N)
\le2\lceil\log_2(2n)\rceil
=2H_n
\notag
\end{align}
members of $\mathscr D$.  If a threshold set is presented as a union of at most $q$ intervals, first merge overlapping or adjacent intervals; this produces at most $q$ pairwise disjoint interval components.  Applying the preceding decomposition to each component gives
\begin{equation}\label{eq:ordered-s-proof}
s(\mathscr D)\le2qH_n.
\end{equation}

\emph{Overlap.}
Fix $i\in[n]$.  At each scale there is exactly one dyadic node of the padded tree containing $i$, and intersecting with $[n]$ does not create additional nodes containing $i$.  Hence
\begin{equation}\label{eq:ordered-Delta-proof}
\Delta(\mathscr D)
\le1+\log_2N
\le H_n.
\end{equation}

Combining \eqref{eq:ordered-s-proof} and \eqref{eq:ordered-Delta-proof},
\begin{align}
\mathfrak D_{L-1}(\bx)
\le s(\mathscr D)\Delta(\mathscr D)
\le2qH_n^2.
\notag
\end{align}
Apply \eqref{eq:recursive-sandwich} to obtain
\begin{align}
\Lambda_L(\bx)
\le\sqrt{2\mathfrak D_{L-1}(\bx)}
\le2\sqrt q\,H_n.
\notag
\end{align}
Substituting this estimate into \eqref{eq:profile-bound-radius} proves \eqref{eq:ordered-G}.
\end{proof}

\begin{proof}[Proof of \Cref{thm:graph-majorant}]
We begin with the rank-one estimate needed below.  Because $G$ is connected,
\begin{align}
\ker L_G
=\operatorname{span}\{\bone\},
\qquad
L_GL_G^\dagger
=L_G^\dagger L_G
=P_{\bone^\perp}.
\notag
\end{align}
Let $w\perp\bone$.  Then $w=L_G^\dagger L_Gw$.  For arbitrary $z\in\R^n$, the positive semidefinite square roots of $L_G$ and $L_G^\dagger$ give
\begin{align*}
z^\top w
&=z^\top L_G^\dagger L_Gw\\
&=(L_G^{\dagger/2}z)^\top(L_G^{1/2}w).
\end{align*}
Cauchy--Schwarz yields
\begin{align}
(z^\top w)^2
\le
(z^\top L_G^\dagger z)(w^\top L_Gw).
\notag
\end{align}
Since this holds for every $z$,
\begin{align}
ww^\top
\preceq
(w^\top L_Gw)L_G^\dagger.
\notag
\end{align}

For a threshold indicator $v\in\{0,1\}^n$, write
\begin{align}
v=\bar v\bone+w,
\qquad
\bar v=\frac{\bone^\top v}{n},
\qquad
w\perp\bone.
\notag
\end{align}
Because $L_G\bone=0$,
\begin{align}
w^\top L_Gw
=(v-\bar v\bone)^\top L_G(v-\bar v\bone)
=v^\top L_Gv.
\notag
\end{align}
The elementary inequality $(p+q)(p+q)^\top\preceq2pp^\top+2qq^\top$ and the rank-one estimate just proved imply
\begin{align}
vv^\top
\preceq2\bar v^2\bone\bone^\top
+2(v^\top L_Gv)L_G^\dagger.
\notag
\end{align}
Apply this to $v_t^+(a)$ and $v_t^-(a)$ and integrate.  Since $\bar v^2\le \bar v$ for an indicator vector,
\begin{align}
\int_0^\infty\bigl[(\bar v_t^+)^2+(\bar v_t^-)^2\bigr]dt
\le\frac1n\int_0^\infty\bigl(|A_t^+(a)|+|A_t^-(a)|\bigr)dt
=\frac{\norm{a}_1}{n}.
\notag
\end{align}
For completeness, the weighted graph coarea identity follows edge by edge.  For fixed real numbers $a_i,a_j$,
\begin{align}
\int_0^\infty
\bigl|
\mathbf1_{\{a_i\ge t\}}-\mathbf1_{\{a_j\ge t\}}
\bigr|dt
+
\int_0^\infty
\bigl|
\mathbf1_{\{a_i\le-t\}}-\mathbf1_{\{a_j\le-t\}}
\bigr|dt
=|a_i-a_j|.
\notag
\end{align}
If $a_i,a_j$ have the same sign, exactly one integral equals the difference of their magnitudes; if they have opposite signs, the two integrals equal $|a_i|$ and $|a_j|$.  Multiplying by $w_{ij}$ and summing over edges gives
\begin{align}
\int_0^\infty
\bigl[(v_t^+)^\top L_Gv_t^++(v_t^-)^\top L_Gv_t^-\bigr]dt
=\operatorname{TV}_G(a).
\notag
\end{align}
Together with \eqref{eq:threshold-matrix}, these identities prove \eqref{eq:graph-psd}.  Taking suprema over $a$, the resulting common majorant is
\begin{align}
N_G=\frac{2M_{L-1,1}}n\bone\bone^\top
+2V_{L-1,G}L_G^\dagger.
\notag
\end{align}
Its trace is the right side of \eqref{eq:graph-tau}.

Finally, a previous-level unit atom satisfies
\begin{align}
|a_i|\le\rho_{L-1}(x_i)=d_{L,i}
\notag
\end{align}
and
\begin{align}
|a_i-a_j|
\le\norm{x_i-x_j}_2^{\,2^{-(L-2)}}.
\notag
\end{align}
Summation proves \eqref{eq:M-bound} and \eqref{eq:V-bound}.
\end{proof}

\section{Proofs for Geometric Regimes}\label[appendix]{app:regimes}

\begin{proof}[Proof of \Cref{prop:k-distinct}]
Let $\xi_1,\ldots,\xi_k$ be the distinct sample values and let $P\in\{0,1\}^{n\times k}$ be the incidence matrix, so $P_{ij}=1$ exactly when $x_i=\xi_j$.  Every empirical BKL Gram matrix has the form
\begin{align}
K=P\widetilde K P^\top,
\notag
\end{align}
where $\widetilde K\succeq0$ is the corresponding Gram matrix on the distinct locations.  Let
\begin{align}
\widetilde D=\diag\bigl(\rho_L(\xi_1)^2,\ldots,\rho_L(\xi_k)^2\bigr).
\notag
\end{align}
As in \Cref{prop:profile-existence}, $\widetilde K\preceq k\widetilde D$.  Therefore every $K$ is dominated by
\begin{align}
N:=P(k\widetilde D)P^\top.
\notag
\end{align}
Its trace is
\begin{align}
\tr N=k\sum_{i=1}^n\rho_L(x_i)^2=kS_L(\bx).
\notag
\end{align}
Thus $\Lambda_L^2\le k$, and \eqref{eq:k-distinct-complexity} follows from \Cref{thm:profile-complexity}.
\end{proof}

\subsection{Orthogonal Growth Law}

\begin{lemma}[RKHS of the two-sided Brownian kernel]\label{lem:two-sided-brownian-rkhs}
Define $\Phi:\R\to L_2(\R)$ by
\begin{equation}\label{eq:two-sided-brownian-feature}
\Phi(t)(s)
:=
\begin{cases}
\mathbf 1_{(0,t]}(s),&t>0,\\
0,&t=0,\\
-\mathbf 1_{[t,0)}(s),&t<0.
\end{cases}
\end{equation}
Then
\begin{equation}\label{eq:two-sided-brownian-feature-inner-product}
\kB(t,u)=\ip{\Phi(t)}{\Phi(u)}_{L_2(\R)},
\qquad t,u\in\R.
\end{equation}
The RKHS $\HB$ of $\kB$ is
\begin{equation}\label{eq:two-sided-brownian-rkhs}
\HB
=
\set{g\in \mathrm{AC}_{\mathrm{loc}}(\R):g(0)=0,\ g'\in L_2(\R)},
\end{equation}
with
\begin{equation}\label{eq:two-sided-brownian-rkhs-norm}
\norm g_{\HB}^2
=
\int_{\R}|g'(s)|^2\,ds.
\end{equation}
\end{lemma}

\begin{proof}
If $t$ and $u$ have opposite signs, the supports of $\Phi(t)$ and $\Phi(u)$ are disjoint, so their inner product is zero.  The Brownian kernel is also zero because $|t-u|=|t|+|u|$.  If $t,u>0$, the inner product equals the length of $(0,t]\cap(0,u]$, namely $\min\{t,u\}$.  If $t,u<0$, the two minus signs cancel and the inner product equals the length of $[t,0)\cap[u,0)$, namely $\min\{|t|,|u|\}$.  In both same-sign cases this quantity equals $(|t|+|u|-|t-u|)/2$.  This proves \eqref{eq:two-sided-brownian-feature-inner-product}.

For $h\in L_2(\R)$, define
\begin{align}
(Th)(t)
:=\ip{h}{\Phi(t)}_{L_2(\R)}.
\notag
\end{align}
By \eqref{eq:two-sided-brownian-feature},
\begin{align}
(Th)(t)
=\int_0^t h(s)\,ds,
\qquad t\in\R,
\notag
\end{align}
where the integral is oriented when $t<0$.  Hence $Th\in \mathrm{AC}_{\mathrm{loc}}(\R)$, $(Th)(0)=0$, and $(Th)'=h$ almost everywhere.  Conversely, if $g\in \mathrm{AC}_{\mathrm{loc}}(\R)$, $g(0)=0$, and $g'\in L_2(\R)$, the fundamental theorem of calculus on each compact interval gives
\begin{align}
g(t)=\int_0^t g'(s)\,ds=(Tg')(t).
\notag
\end{align}
Thus the range of $T$ is exactly the set in \eqref{eq:two-sided-brownian-rkhs}.

It remains to identify the norm.  Differences of the features in \eqref{eq:two-sided-brownian-feature} generate indicators of bounded intervals contained in either half-line.  Finite linear combinations of such indicators are dense in $L_2(\R)$.  Therefore \eqref{eq:two-sided-brownian-feature-inner-product} is a minimal feature representation of $\kB$, and the associated RKHS norm is the minimum $L_2$ norm of a representing vector.  The map $T$ is injective: if $Th=0$, then the indefinite integral of $h$ vanishes at every $t$, so $h=0$ almost everywhere.  Hence the representing vector of $g=Th$ is unique, and
\begin{align}
\norm g_{\HB}
=\norm h_{L_2(\R)}
=\norm{g'}_{L_2(\R)}.
\notag
\end{align}
Squaring proves \eqref{eq:two-sided-brownian-rkhs-norm}.
\end{proof}

\begin{lemma}[Pullback contraction]\label{lem:pullback-contraction}
Let $k$ be a kernel on a set $Z$, let $u:\cX\to Z$, and define
\begin{align}
k_u(x,x'):=k(u(x),u(x')).
\notag
\end{align}
For every $g\in H_k$, the function $g\circ u$ belongs to $H_{k_u}$ and
\begin{align}
\norm{g\circ u}_{H_{k_u}}\le\norm g_{H_k}.
\notag
\end{align}
\end{lemma}

\begin{proof}
Let
\begin{align}
M:=\overline{\operatorname{span}}\{k(\cdot,u(x)):x\in\cX\}\subseteq H_k
\notag
\end{align}
and let $P_M$ be the orthogonal projection onto $M$.  For every $x\in\cX$,
\begin{align}
(g-P_Mg)(u(x))
=\ip{g-P_Mg}{k(\cdot,u(x))}_{H_k}
=0,
\notag
\end{align}
so $g\circ u=(P_Mg)\circ u$.

On the algebraic span of the kernel sections in $M$, define
\begin{align}
U\!\bigl(\sum_{j=1}^q c_jk(\cdot,u(x_j))\bigr)
:=\sum_{j=1}^q c_jk_u(\cdot,x_j).
\notag
\end{align}
For every such finite sum,
\begin{align*}
\norm{\sum_{j=1}^q c_jk(\cdot,u(x_j))}_{H_k}^2
&=\sum_{i,j=1}^q c_ic_jk(u(x_i),u(x_j))\\
&=\sum_{i,j=1}^q c_ic_jk_u(x_i,x_j)\\
&=\norm{\sum_{j=1}^q c_jk_u(\cdot,x_j)}_{H_{k_u}}^2.
\end{align*}
Thus $U$ is well defined and isometric on a dense subspace of $M$.  For a finite combination $h=\sum_jc_jk(\cdot,u(x_j))$ and any $x\in\cX$, the reproducing property gives
\begin{align}
(Uh)(x)=\sum_jc_jk_u(x,x_j)=\sum_jc_jk(u(x),u(x_j))=h(u(x)).
\notag
\end{align}
The map $U$ extends to an isometric isomorphism from $M$ onto $H_{k_u}$, because its range contains the dense span of the kernel sections of $H_{k_u}$ and is closed.  Passing to limits in the displayed evaluation identity shows that $P_Mg$ corresponds to $(P_Mg)\circ u=g\circ u$.  Therefore
\begin{align}
\norm{g\circ u}_{H_{k_u}}
=\norm{P_Mg}_{H_k}
\le\norm g_{H_k},
\notag
\end{align}
where the last inequality is the contractivity of an orthogonal projection.
\end{proof}

\begin{proof}[Proof of \Cref{thm:orthogonal}]
We first prove the upper bound.

For the orthogonal sample, define
\begin{equation}\label{eq:m-ell}
m_\ell:=\sup_{f\in B_\ell(1)}\sum_{i=1}^n|f(e_i)|,
\qquad \ell\ge1.
\end{equation}
The linear base gives $m_1=\sqrt n$.

Let $K$ be any top Gram matrix at depth $\ell\ge2$, and let $y\in E(K)$.  Then
\begin{align}
\norm{y}_1
=\max_{s\in\{\pm1\}^n}s^\top y
\le\max_s\sqrt{s^\top Ks}.
\notag
\end{align}
Every depth-$\ell$ top kernel is a probability mixture of Dirac matrices $\mathbf B(a)$ generated by previous-layer unit traces.  Hence, for fixed $s$,
\begin{align}
s^\top Ks
=\int s^\top\mathbf B(a(u))s\,d\mu(u)
\le\sup_{a\in\sS_{\ell-1}}s^\top\mathbf B(a)s.
\notag
\end{align}
Let $\phi_t$ be the signed-interval Brownian feature, so $\norm{\phi_t}_{L_2}^2=|t|$.  For every previous-layer unit trace $a$,
\begin{align*}
s^\top\mathbf B(a)s
&=\norm{\sum_{i=1}^n s_i\phi_{a_i}}_{L_2}^2\\
&\le\bigl(\sum_{i=1}^n\norm{\phi_{a_i}}_{L_2}\bigr)^2\\
&=\bigl(\sum_{i=1}^n\sqrt{|a_i|}\bigr)^2\\
&\le n\sum_{i=1}^n|a_i|\\
&\le nm_{\ell-1}.
\end{align*}
The first inequality is the triangle inequality in $L_2$, and the second is Cauchy--Schwarz.  Therefore
\begin{align}
m_\ell\le\sqrt{nm_{\ell-1}}.
\notag
\end{align}
Writing $m_\ell\le n^{\beta_\ell}$, the recursion starts from $\beta_1=1/2$ and satisfies $\beta_\ell=(1+\beta_{\ell-1})/2$.  Induction gives $\beta_\ell=1-2^{-\ell}$, so
\begin{align}
m_\ell\le n^{1-2^{-\ell}}.
\notag
\end{align}

For every depth-$L$ top kernel,
\begin{align}
\tr K
=\int\sum_{i=1}^n|u(e_i)|\,d\mu(u)
\le m_{L-1}.
\notag
\end{align}
Since $K\succeq0$, $K\preceq(\tr K)I\preceq m_{L-1}I$.  Thus
\begin{align}
\tau_L(e_1,\ldots,e_n)
\le n m_{L-1}
\le n^{2-2^{-(L-1)}},
\notag
\end{align}
and, because $S_L=n$,
\begin{equation}\label{eq:orthogonal-profile-upper}
\Lambda_L(e_1,\ldots,e_n)
\le n^{1/2-2^{-L}}.
\end{equation}

We now prove the matching lower bound.

For $s\in\{\pm1\}^n$, set
\begin{align}
u_1^s(x):=\ip{s/\sqrt n}{x}.
\notag
\end{align}
Then $\norm{u_1^s}_{\cH^{(1)}}=1$ and $u_1^s(e_i)=n^{-1/2}s_i$.

By \Cref{lem:two-sided-brownian-rkhs}, the RKHS of the two-sided kernel in \eqref{eq:brownian-kernel} is the anchored space in \eqref{eq:two-sided-brownian-rkhs}, with derivative norm \eqref{eq:two-sided-brownian-rkhs-norm}.  This is the signed two-sided realization, not the one-sided $\min(s,t)$ RKHS.  The function $\psi_a$ in \eqref{eq:psi-a} is locally absolutely continuous, is anchored at zero, has derivative $b/a$ on $(-a,a)$ and zero elsewhere, and therefore satisfies
\begin{align}
\norm{\psi_a}_{\HB}^2
=2a(b/a)^2=1,
\qquad
\psi_a(\pm a)=\pm\sqrt{a/2}.
\notag
\end{align}
We construct unit-sphere atoms recursively.  Suppose that $u_\ell^s$ is a unit vector in one admissible level-$\ell$ RKHS and
\begin{align}
u_\ell^s(e_i)=c_\ell^s s_i,
\qquad c_\ell^s\ge a_\ell.
\notag
\end{align}
Use the next-layer Dirac measure $\delta_{u_\ell^s}$.  The resulting kernel is the pullback
$k(x,x')=\kB(u_\ell^s(x),u_\ell^s(x'))$.  Applying \Cref{lem:pullback-contraction} with $k=\kB$, $u=u_\ell^s$, and $g=\psi_{c_\ell^s}$ shows that
$h_{\ell+1}^s:=\psi_{c_\ell^s}\circ u_\ell^s$ belongs to the pullback RKHS and satisfies
$\norm{h_{\ell+1}^s}\le\norm{\psi_{c_\ell^s}}_{\HB}=1$.  At the sample points,
\begin{align}
h_{\ell+1}^s(e_i)
=\psi_{c_\ell^s}(c_\ell^s s_i)
=\sqrt{c_\ell^s/2}\,s_i.
\notag
\end{align}
This function is nonzero.  If its pullback norm is $q\in(0,1]$, dividing by $q$ produces a next-level unit-sphere atom and multiplies, rather than decreases, the trace amplitude.  Hence, exactly and without an attainment assumption,
\begin{align}
a_{\ell+1}\ge\sqrt{a_\ell/2}.
\notag
\end{align}
Starting from $a_1=n^{-1/2}$ gives
\begin{equation}\label{eq:aL}
a_L
\ge2^{-(1-2^{-(L-1)})}n^{-2^{-L}}
\ge\frac12n^{-2^{-L}}.
\end{equation}
If normalization produces a larger amplitude, scalar contraction yields the common amplitude in \eqref{eq:aL}.  Hence the full BKL unit ball contains the sign cube
\begin{align}
\set{a_Ls:s\in\{\pm1\}^n}
\notag
\end{align}
on the sample.  Therefore
\begin{align}
\Ghat_{(e_i)}(B_L(1))
&\ge\frac{a_L}{n}\E_G\max_{s\in\{\pm1\}^n}\sum_{i=1}^nG_is_i\nonumber\\
&=a_L\E|G_1|
\ge c n^{-2^{-L}}.
\label{eq:orthogonal-width-lower}
\end{align}
On the other hand, \Cref{thm:profile-complexity} gives
\begin{align}
\Ghat_{(e_i)}(B_L(1))\le\frac{\Lambda_L(e_1,\ldots,e_n)}{\sqrt n}.
\notag
\end{align}
Combining with \eqref{eq:orthogonal-width-lower} yields
\begin{align}
\Lambda_L(e_1,\ldots,e_n)
\ge c n^{1/2-2^{-L}}.
\notag
\end{align}
Together with \eqref{eq:orthogonal-profile-upper}, this proves \Cref{thm:orthogonal}.
\end{proof}

\section{Proofs for Gaussian Richness, Random Designs, and Stability}\label[appendix]{app:richness}

\subsection{Proof of the Stable-Threshold Gaussian Reverse}

\begin{proof}[Proof of \Cref{thm:threshold-gaussian-reverse}]
We first establish the threshold-domination estimate used below.  Let $\mathcal F\subseteq2^{[n]}$ be finite.  If $\mathcal F=\varnothing$, then $\omega_{\mathrm G}(\mathcal F)=\Theta(\mathcal F)=0$.  Assume $\mathcal F\ne\varnothing$, and let $H\succeq0$ satisfy
\begin{align}
H\succeq\bone_A\bone_A^\top
\qquad
\text{for every }A\in\mathcal F.
\notag
\end{align}
For every $g\in\R^n$ and every $A\in\mathcal F$,
\begin{align}
\abs{\sum_{i\in A}g_i}^2
=g^\top\bone_A\bone_A^\top g
\le g^\top Hg.
\notag
\end{align}
Taking the maximum over $A$, then taking square roots and Gaussian expectation, gives
\begin{align*}
\omega_{\mathrm G}(\mathcal F)
&\le\E\sqrt{G^\top HG}\\
&\le\sqrt{\E[G^\top HG]}\\
&=\sqrt{\tr H}.
\end{align*}
After squaring and minimizing over feasible $H$, we obtain
\begin{equation}\label{eq:set-width-theta}
\omega_{\mathrm G}(\mathcal F)^2
\le\Theta(\mathcal F).
\end{equation}

If $S_L(\bx)=0$, then \Cref{prop:profile-existence} gives $\tau_L(\bx)=\Lambda_L(\bx)=0$, and the conventions in the statement make every displayed inequality immediate.  Assume $S_L(\bx)>0$.

We first compare $\mathfrak W_{L-1}$ and $\mathfrak R_{L-1}$.
For every $t\ge0$ and $\delta>0$, the threshold-domination estimate \eqref{eq:set-width-theta} gives
\begin{align}
\omega_{\mathrm G}
\left(\mathcal F_{L-1}^{\mathrm{st}}(t,\delta)\right)^2
\le
\Theta
\left(\mathcal F_{L-1}^{\mathrm{st}}(t,\delta)\right).
\notag
\end{align}
Multiply by $\delta/S_L(\bx)$ and take the supremum over $t,\delta$.  By \eqref{eq:W-profile} and \eqref{eq:R-profile},
\begin{align}
\mathfrak W_{L-1}(\bx)
\le
\mathfrak R_{L-1}(\bx).
\notag
\end{align}
The remaining inequalities in \eqref{eq:W-R-D-chain} are \Cref{thm:recursive-threshold}.

We next show that a stable threshold set gives a pointwise lower bound for the BKL support function.
Fix $t\ge0$, $\delta>0$, and
\begin{align}
A\in\mathcal F_{L-1}^{\mathrm{st}}(t,\delta).
\notag
\end{align}
Suppose first that $A\in\mathcal F_{L-1}^{+,\mathrm{st}}(t,\delta)$.  By definition, there exists $a_A\in\sS_{L-1}(\bx)$ such that
\begin{align}
\set{i:(a_A)_i\ge s}=A
\qquad
\text{for every }s\in[t,t+\delta].
\notag
\end{align}
The signed layer-cake identity \eqref{eq:threshold-matrix} gives
\begin{align*}
\mathbf B(a_A)
&=
\int_0^\infty
\left(
\bv_s^+(a_A)\bv_s^+(a_A)^\top
+
\bv_s^-(a_A)\bv_s^-(a_A)^\top
\right)ds\\
&\succeq
\int_t^{t+\delta}
\bv_s^+(a_A)\bv_s^+(a_A)^\top\,ds\\
&=
\int_t^{t+\delta}
\bone_A\bone_A^\top\,ds\\
&=
\delta\bone_A\bone_A^\top.
\end{align*}
If $A\in\mathcal F_{L-1}^{-,\mathrm{st}}(t,\delta)$, the same argument uses the negative threshold vectors and reaches the same conclusion.  Therefore, for every stable set $A$, there is an admissible trace $a_A$ satisfying
\begin{equation}\label{eq:stable-set-B-lower}
\mathbf B(a_A)
\succeq
\delta\bone_A\bone_A^\top.
\end{equation}
For arbitrary $g\in\R^n$, \Cref{thm:dirac-reduction} and \eqref{eq:stable-set-B-lower} give
\begin{align*}
h_{L,\bx}(g)^2
&=
\sup_{a\in\sS_{L-1}(\bx)}g^\top\mathbf B(a)g\\
&\ge
\max_{A\in\mathcal F_{L-1}^{\mathrm{st}}(t,\delta)}
g^\top\mathbf B(a_A)g\\
&\ge
\delta
\max_{A\in\mathcal F_{L-1}^{\mathrm{st}}(t,\delta)}
\left(\sum_{i\in A}g_i\right)^2.
\end{align*}
Both sides are nonnegative.  Taking square roots yields
\begin{equation}\label{eq:stable-set-support-lower}
h_{L,\bx}(g)
\ge
\sqrt\delta
\max_{A\in\mathcal F_{L-1}^{\mathrm{st}}(t,\delta)}
\abs{\sum_{i\in A}g_i}.
\end{equation}

We then take the Gaussian expectation and optimize the threshold interval.
Set $g=G$ in \eqref{eq:stable-set-support-lower} and take expectation.  By \eqref{eq:set-gaussian-width},
\begin{align}
W_{G,1}(L,\bx)
\ge
\sqrt\delta\,
\omega_{\mathrm G}
\left(\mathcal F_{L-1}^{\mathrm{st}}(t,\delta)\right).
\notag
\end{align}
Squaring gives
\begin{align}
W_{G,1}(L,\bx)^2
\ge
\delta\,
\omega_{\mathrm G}
\left(\mathcal F_{L-1}^{\mathrm{st}}(t,\delta)\right)^2.
\notag
\end{align}
Take the supremum over $t\ge0$ and $\delta>0$, and use \eqref{eq:W-profile}.  This proves \eqref{eq:W-lower-first-moment}.

Finally, we translate the result to Gaussian complexity and the common-covariance certificate.
Multiply the square root of \eqref{eq:W-lower-first-moment} by $r/n$ and use \eqref{eq:complexity-WG1}; this gives the lower bound in \eqref{eq:threshold-two-sided-complexity}.  Its upper bound is \Cref{thm:profile-complexity}.

Assume now that $\mathfrak D_{L-1}(\bx)>0$.  By \eqref{eq:recursive-tau},
\begin{align}
\tau_L(\bx)
\le
2\mathfrak D_{L-1}(\bx)S_L(\bx).
\notag
\end{align}
Rearranging gives
\begin{align}
S_L(\bx)
\ge
\frac{\tau_L(\bx)}{2\mathfrak D_{L-1}(\bx)}.
\notag
\end{align}
Substitution into the lower bound in \eqref{eq:threshold-two-sided-complexity} proves \eqref{eq:threshold-reverse-tau}.  If $\mathfrak D_{L-1}(\bx)=0$, \eqref{eq:recursive-tau} gives $\tau_L(\bx)=0$, as stated.
\end{proof}

\begin{proof}[Proof of \Cref{cor:minimax-orthogonal}]
Let $\bx=(e_1,\ldots,e_n)$.  The construction in the proof of \Cref{thm:orthogonal} gives, at depth $L-1$, a number
\begin{align}
a_{L-1}\ge c_L n^{-2^{-(L-1)}}
\notag
\end{align}
and, for every sign vector $s\in\set{-1,+1}^n$, an admissible trace with coordinates $a_{L-1}s_i$.  Fix a subset $A\subseteq[n]$ and choose $s_i=+1$ on $A$ and $s_i=-1$ on $A^c$.  For every threshold level $u\in[0,a_{L-1}]$,
\begin{align}
\set{i:a_{L-1}s_i\ge u}=A.
\notag
\end{align}
Therefore
\begin{align}
2^{[n]}
\subseteq
\mathcal F_{L-1}^{+,\mathrm{st}}(0,a_{L-1}).
\notag
\end{align}

For every realization of $G$,
\begin{align}
\max_{A\subseteq[n]}\abs{\sum_{i\in A}G_i}
\ge
\sum_{i:G_i>0}G_i.
\notag
\end{align}
Taking expectation and using $\E[G_i\bone_{\{G_i>0\}}]=1/\sqrt{2\pi}$ gives
\begin{align}
\omega_{\mathrm G}(2^{[n]})
\ge
\frac{n}{\sqrt{2\pi}}.
\notag
\end{align}
Conversely,
\begin{align}
\max_{A\subseteq[n]}\abs{\sum_{i\in A}G_i}
\le
\sum_{i=1}^n|G_i|,
\notag
\end{align}
so
\begin{align}
\omega_{\mathrm G}(2^{[n]})
\le
n\sqrt{\frac2\pi}.
\notag
\end{align}
Thus
\begin{align}
\omega_{\mathrm G}(2^{[n]})\asymp n.
\notag
\end{align}
Since $S_L(\bx)=n$, the definition \eqref{eq:W-profile} gives
\begin{align}
\mathfrak W_{L-1}(\bx)
\ge
\frac{a_{L-1}}n
\omega_{\mathrm G}(2^{[n]})^2
\gtrsim
n^{1-2^{-(L-1)}}.
\notag
\end{align}
On the other hand, \eqref{eq:W-R-D-chain} and \Cref{thm:orthogonal} give
\begin{align}
\mathfrak W_{L-1}(\bx)
\le
\Lambda_L(\bx)^2
\lesssim
n^{1-2^{-(L-1)}}.
\notag
\end{align}
This proves \eqref{eq:orthogonal-W-equivalence}.  Substituting this equivalence into \eqref{eq:threshold-reverse-tau} shows that the lower and upper common-covariance Gaussian bounds differ by at most universal constants.
\end{proof}

\begin{lemma}[A macroscopic stable threshold gives nonvanishing Gaussian richness]\label{lem:macroscopic-W}
Assume $S_L(\bx)>0$ and suppose that \eqref{eq:macroscopic-stable-threshold} holds for some $A,t,\delta$ and $\kappa>0$.  Then
\begin{equation}\label{eq:macroscopic-W-lower}
\mathfrak W_{L-1}(\bx)
\ge
\frac{2\kappa}{\pi}.
\end{equation}
\end{lemma}

\begin{proof}
The family $\mathcal F_{L-1}^{\mathrm{st}}(t,\delta)$ contains $A$.  Therefore
\begin{align*}
\omega_{\mathrm G}
\left(\mathcal F_{L-1}^{\mathrm{st}}(t,\delta)\right)
&\ge
\E\abs{\sum_{i\in A}G_i}\\
&=
\sqrt{\frac{2|A|}{\pi}},
\end{align*}
because $\sum_{i\in A}G_i$ is centered Gaussian with variance $|A|$.  Hence
\begin{align*}
\mathfrak W_{L-1}(\bx)
&\ge
\frac\delta{S_L(\bx)}
\omega_{\mathrm G}
\left(\mathcal F_{L-1}^{\mathrm{st}}(t,\delta)\right)^2\\
&\ge
\frac\delta{S_L(\bx)}\frac{2|A|}{\pi}\\
&\ge
\frac{2\kappa}{\pi}.
\end{align*}
\end{proof}

\begin{proof}[Proof of \Cref{cor:minimax-hierarchical}]
The macroscopic-threshold assumption implies $S_L(\bx)>0$, so \Cref{lem:macroscopic-W} applies.  By that lemma,
\begin{align}
\mathfrak W_{L-1}(\bx)
\ge
\frac{2\kappa}{\pi}.
\notag
\end{align}
By \Cref{cor:hierarchical-threshold},
\begin{align}
\mathfrak D_{L-1}(\bx)
\le
q(H+1).
\notag
\end{align}
Insert these bounds into \eqref{eq:threshold-reverse-tau}:
\begin{align*}
\Ghat_{\bx}(B_L(r))
&\ge
\frac{r\sqrt{\tau_L(\bx)}}n
\sqrt{
\frac{2\kappa/\pi}
{2q(H+1)}
}\\
&=
\sqrt{\frac{\kappa}{\pi q(H+1)}}
\frac{r\sqrt{\tau_L(\bx)}}n.
\end{align*}
This is \eqref{eq:hierarchical-minimax-reverse}.  Comparing it with the upper bound in \eqref{eq:profile-bound-tau} gives the stated factor.
\end{proof}

\begin{proof}[Proof of \Cref{cor:minimax-ordered}]
The macroscopic-threshold assumption implies $S_L(\bx)>0$, so \Cref{lem:macroscopic-W} applies and gives
\begin{align}
\mathfrak W_{L-1}(\bx)
\ge
\frac{2\kappa}{\pi}.
\notag
\end{align}
By \Cref{cor:ordered-threshold},
\begin{align}
\mathfrak D_{L-1}(\bx)
\le
2qH_n^2.
\notag
\end{align}
Substitution into \eqref{eq:threshold-reverse-tau} gives
\begin{align*}
\Ghat_{\bx}(B_L(r))
&\ge
\frac{r\sqrt{\tau_L(\bx)}}n
\sqrt{
\frac{2\kappa/\pi}
{4qH_n^2}
}\\
&=
\sqrt{\frac{\kappa}{2\pi q}}
\frac1{H_n}
\frac{r\sqrt{\tau_L(\bx)}}n.
\end{align*}
This is \eqref{eq:ordered-minimax-reverse}.
\end{proof}

\subsection{Proofs for Random-Design Gaussian Richness}\label{app:random-design}

\begin{proof}[Proof of \Cref{prop:brownian-distortion}]
Fix an actual level-$\ell$ kernel $K_x$ from the family on $\bx$.  By the hypothesis in \eqref{eq:family-domination}, there is an actual level-$\ell$ kernel $K_y$ from the family on $\by$ such that
\begin{equation}\label{eq:distortion-base-pair}
K_x\preceq cK_y.
\end{equation}
If $K_x=0$, every unit-sphere element of the associated RKHS has zero trace on $\bx$, so every empirical kernel generated by the next Brownian layer is zero.  The required domination is then immediate.  Assume henceforth that $K_x\ne0$.  Equation \eqref{eq:distortion-base-pair} implies $K_y\ne0$.

Let $H_x$ and $H_y$ be the two actual level-$\ell$ RKHSs, and let $\mu$ be the probability measure on $S_1(H_x)$ that generates an actual level-$(\ell+1)$ kernel on $\bx$.  For $u\in S_1(H_x)$, write
\begin{align}
a(u):=\left(u(x_1),\ldots,u(x_n)\right).
\notag
\end{align}
The RKHS trace-ellipsoid identity gives $a(u)\in E(K_x)$.  Since \eqref{eq:distortion-base-pair} implies
\begin{align}
E(K_x)\subseteq E(cK_y)=\sqrt c\,E(K_y),
\notag
\end{align}
the vector
\begin{align}
b(u):=\frac{a(u)}{\sqrt c}
\notag
\end{align}
belongs to $E(K_y)$.  Define
\begin{align}
q(u):=\left(b(u)^\top K_y^\dagger b(u)\right)^{1/2}\in[0,1].
\notag
\end{align}

We construct a Borel unit-sphere element of $H_y$ for every $u$.  If $q(u)>0$, put
\begin{align}
s(u):=\frac{b(u)}{q(u)}
\notag
\end{align}
and define the minimum-norm interpolant
\begin{equation}\label{eq:distortion-interpolant}
v_u
:=
\sum_{i=1}^n
\left[K_y^\dagger s(u)\right]_i
k_y(\cdot,y_i).
\end{equation}
Because $b(u)\in\operatorname{ran}K_y$, also $s(u)\in\operatorname{ran}K_y$, and therefore the trace of $v_u$ on $\by$ is
\begin{align}
K_yK_y^\dagger s(u)=s(u).
\notag
\end{align}
The standard finite-sample RKHS norm formula gives
\begin{align}
\norm{v_u}_{H_y}^2
=s(u)^\top K_y^\dagger s(u)
=1.
\notag
\end{align}
If $q(u)=0$, then $b(u)=0$ because $b(u)\in\operatorname{ran}K_y$.  Choose once and for all an arbitrary $v_0\in S_1(H_y)$, set $v_u:=v_0$, and let $s(u):=(v_0(y_i))_{i=1}^n$.  The factor $q(u)$ multiplying $\mathbf B(s(u))$ below is then zero.  The map $u\mapsto a(u)$ is continuous, $b\mapsto q$ is continuous, and the formula in \eqref{eq:distortion-interpolant} is continuous on the set where $q>0$.  Hence $u\mapsto v_u$ is Borel measurable after the preceding definition on the Borel set $\set{q=0}$.

For $q(u)>0$, Brownian nonnegative homogeneity gives
\begin{align*}
\mathbf B(a(u))
&=\mathbf B\left(\sqrt c\,q(u)s(u)\right)\\
&=\sqrt c\,q(u)\mathbf B(s(u)).
\end{align*}
For $q(u)=0$, both sides of this equality are zero if the right side is interpreted with its factor $q(u)$.  Therefore the empirical next-layer kernel on $\bx$ satisfies
\begin{align*}
K_x^+
&=\int_{S_1(H_x)}\mathbf B(a(u))\,d\mu(u)\\
&=\sqrt c
\int_{S_1(H_x)}q(u)\mathbf B(s(u))\,d\mu(u)\\
&\preceq
\sqrt c
\int_{S_1(H_x)}\mathbf B(s(u))\,d\mu(u),
\end{align*}
because $(1-q(u))\mathbf B(s(u))\succeq0$ pointwise.  Let $\nu$ be the pushforward of $\mu$ under $u\mapsto v_u$.  It is a Borel probability measure on $S_1(H_y)$, and the final integral equals
\begin{align}
K_y^+
:=
\int_{S_1(H_y)}
\mathbf B\left(\left(v(y_i)\right)_{i=1}^n\right)
\,d\nu(v).
\notag
\end{align}
Thus $K_y^+$ is an actual level-$(\ell+1)$ kernel on $\by$ and
\begin{align}
K_x^+\preceq\sqrt c\,K_y^+.
\notag
\end{align}
This proves the domination of the actual level-$(\ell+1)$ families.

To pass to closures, let $A_m\in\sK_{\ell+1}^\circ(\bx)$ converge to $A\in\sK_{\ell+1}(\bx)$.  For every $m$, choose $B_m\in\sK_{\ell+1}^\circ(\by)$ with $A_m\preceq\sqrt c\,B_m$.  The closed family $\sK_{\ell+1}(\by)$ is compact, so a subsequence of $(B_m)$ converges to some $B\in\sK_{\ell+1}(\by)$.  Closedness of the positive-semidefinite cone gives $A\preceq\sqrt c\,B$.  This proves \eqref{eq:brownian-distortion-propagation}.
\end{proof}

\begin{proof}[Proof of \Cref{thm:well-conditioned-richness}]
At level one, \eqref{eq:well-conditioned-gram} gives
\begin{align}
\sK_1^\circ(\bx)&\preceq_{\lambda_+}\sK_1^\circ(\be_n),
\notag\\
\sK_1^\circ(\be_n)&\preceq_{\lambda_-^{-1}}\sK_1^\circ(\bx).
\notag
\end{align}
There are $L-1$ Brownian lifts from the linear Gram matrix to the top covariance family.  Repeated application of \Cref{prop:brownian-distortion} therefore yields
\begin{align}
\sK_L(\bx)
&\preceq_{\lambda_+^{\vartheta_L}}
\sK_L(\be_n),
\label{eq:well-conditioned-family-upper}\\
\sK_L(\be_n)
&\preceq_{\lambda_-^{-\vartheta_L}}
\sK_L(\bx).
\label{eq:well-conditioned-family-lower}
\end{align}

Let $N_e$ be an optimal common covariance for the orthogonal design.  Equation \eqref{eq:well-conditioned-family-upper} implies that $\lambda_+^{\vartheta_L}N_e$ dominates every kernel in $\sK_L(\bx)$, so
\begin{align}
\tau_L(\bx)
\le
\lambda_+^{\vartheta_L}\tau_L(\be_n).
\notag
\end{align}
Applying the same argument to \eqref{eq:well-conditioned-family-lower} gives
\begin{align}
\tau_L(\be_n)
\le
\lambda_-^{-\vartheta_L}\tau_L(\bx),
\notag
\end{align}
which is the lower inequality in \eqref{eq:well-conditioned-tau}.

For every $z\in\R^n$, \eqref{eq:support-covariance} and \eqref{eq:well-conditioned-family-upper} give
\begin{align*}
h_{L,\bx}(z)^2
&=\sup_{K\in\sK_L(\bx)}z^\top Kz\\
&\le
\lambda_+^{\vartheta_L}
\sup_{K\in\sK_L(\be_n)}z^\top Kz\\
&=\lambda_+^{\vartheta_L}h_{L,\be_n}(z)^2.
\end{align*}
The lower comparison follows analogously from \eqref{eq:well-conditioned-family-lower}.  Taking square roots proves \eqref{eq:well-conditioned-support}.

The diagonal entries of $G_{\bx}$ satisfy
\begin{align}
\lambda_-
\le
\norm{x_i}_2^2
\le
\lambda_+.
\notag
\end{align}
Since $d_{L,i}=\left(\norm{x_i}_2^2\right)^{\vartheta_L}$,
\begin{equation}\label{eq:well-conditioned-S}
n\lambda_-^{\vartheta_L}
\le
S_L(\bx)
\le
n\lambda_+^{\vartheta_L}.
\end{equation}
By \eqref{eq:well-conditioned-tau}, \eqref{eq:well-conditioned-S}, and \Cref{thm:orthogonal},
\begin{align*}
\Lambda_L(\bx)^2
&=\frac{\tau_L(\bx)}{S_L(\bx)}\\
&\le
\frac{\lambda_+^{\vartheta_L}\tau_L(\be_n)}{n\lambda_-^{\vartheta_L}}\\
&=\kappa_{\bx}^{\vartheta_L}\Lambda_L(\be_n)^2\\
&\le
C\kappa_{\bx}^{\vartheta_L}n^{1-\vartheta_L}.
\end{align*}

It remains to prove the lower bound for $\mathfrak W_{L-1}(\bx)$.  Fix $s\in\set{-1,+1}^n$ and define
\begin{align}
y_s:=\sqrt{\frac{\lambda_-}{n}}\,s.
\notag
\end{align}
The Gram matrix is positive definite, and $G_{\bx}^{-1}\preceq\lambda_-^{-1}I_n$.  Therefore
\begin{align*}
y_s^\top G_{\bx}^{-1}y_s
&\le
\lambda_-^{-1}\norm{y_s}_2^2\\
&=\lambda_-^{-1}\frac{\lambda_-}{n}\norm{s}_2^2\\
&=1.
\end{align*}
Thus $y_s\in E(G_{\bx})$.  Let $q_s:=\left(y_s^\top G_{\bx}^{-1}y_s\right)^{1/2}$.  Since $y_s\ne0$, one has $q_s>0$.  The vector $y_s/q_s$ is the sample trace of a unit linear function and has common sign amplitude at least
\begin{align}
a_1:=\sqrt{\frac{\lambda_-}{n}}.
\notag
\end{align}

Use the Brownian profile $\psi_a$ from \eqref{eq:psi-a}.  By \Cref{lem:pullback-contraction}, if a unit-sphere trace has common signed amplitude at least $a_\ell$, then the next-level unit sphere contains the same sign pattern with common amplitude at least
\begin{align}
a_{\ell+1}:=\sqrt{\frac{a_\ell}{2}}.
\notag
\end{align}
An induction gives, for every $m\ge1$,
\begin{equation}\label{eq:amplitude-recursion-solved}
a_m
\ge
2^{-\left(1-2^{-(m-1)}\right)}
a_1^{2^{-(m-1)}}.
\end{equation}
Indeed, the formula is exact at $m=1$.  If it holds at $m$, then
\begin{align*}
a_{m+1}
&\ge
2^{-1/2}a_m^{1/2}\\
&\ge
2^{-1/2}
2^{-\left(1-2^{-(m-1)}\right)/2}
a_1^{2^{-m}}\\
&=
2^{-\left(1-2^{-m}\right)}a_1^{2^{-m}}.
\end{align*}
Taking $m=L-1$ and using $2^{-\left(1-2^{-(L-2)}\right)}\ge1/2$ gives
\begin{equation}\label{eq:well-conditioned-sign-amplitude}
a_{L-1}
\ge
\frac12\lambda_-^{\vartheta_L}n^{-\vartheta_L}.
\end{equation}

For every subset $A\subseteq[n]$, choose the sign pattern that is positive on $A$ and negative on $A^c$.  Every threshold $u\in[a_{L-1}/2,a_{L-1}]$ then has positive superlevel set exactly $A$.  Hence
\begin{align}
2^{[n]}
\subseteq
\mathcal F_{L-1}^{+,\mathrm{st}}
\left(\frac{a_{L-1}}2,\frac{a_{L-1}}2\right).
\notag
\end{align}
For every Gaussian vector $G$,
\begin{align}
\max_{A\subseteq[n]}\abs{\sum_{i\in A}G_i}
\ge
\frac12\sum_{i=1}^n|G_i|.
\notag
\end{align}
Taking expectation gives
\begin{align}
\omega_{\mathrm G}(2^{[n]})
\ge
\frac n2\sqrt{\frac2\pi}
=
\frac n{\sqrt{2\pi}}.
\notag
\end{align}
Using \eqref{eq:W-profile}, \eqref{eq:well-conditioned-S}, and \eqref{eq:well-conditioned-sign-amplitude},
\begin{align*}
\mathfrak W_{L-1}(\bx)
&\ge
\frac{a_{L-1}/2}{S_L(\bx)}
\omega_{\mathrm G}(2^{[n]})^2\\
&\ge
\frac{
\frac14\lambda_-^{\vartheta_L}n^{-\vartheta_L}
}{n\lambda_+^{\vartheta_L}}
\frac{n^2}{2\pi}\\
&=
\frac1{8\pi}
\kappa_{\bx}^{-\vartheta_L}
n^{1-\vartheta_L}.
\end{align*}
Together with \eqref{eq:W-R-D-chain} and the profile upper bound proved above, this establishes \eqref{eq:well-conditioned-W-Lambda}.

By \eqref{eq:well-conditioned-support},
\begin{align}
W_{G,2}(L,\bx)
\ge
\lambda_-^{\vartheta_L/2}W_{G,2}(L,\be_n).
\notag
\end{align}
Equation \eqref{eq:well-conditioned-tau} gives
\begin{align}
\sqrt{\tau_L(\bx)}
\le
\lambda_+^{\vartheta_L/2}\sqrt{\tau_L(\be_n)}.
\notag
\end{align}
Therefore
\begin{align}
\chi_L(\bx)
\le
\kappa_{\bx}^{\vartheta_L/2}\chi_L(\be_n).
\notag
\end{align}
For the orthogonal design, $S_L(\be_n)=n$.  Since $W_{G,2}\ge W_{G,1}$, \eqref{eq:W-lower-first-moment} and \eqref{eq:orthogonal-W-equivalence} give
\begin{align*}
\chi_L(\be_n)^2
&=\frac{\tau_L(\be_n)}{W_{G,2}(L,\be_n)^2}\\
&\le\frac{\tau_L(\be_n)}{W_{G,1}(L,\be_n)^2}\\
&\le\frac{n\Lambda_L(\be_n)^2}{n\mathfrak W_{L-1}(\be_n)}\\
&\le C.
\end{align*}
Therefore $\chi_L(\be_n)\le C$ universally.  This proves \eqref{eq:well-conditioned-defect}.  Finally, take expectations in \eqref{eq:well-conditioned-support}, use \eqref{eq:empirical-width-support}, and apply \Cref{thm:orthogonal}; this gives \eqref{eq:well-conditioned-complexity}.
\end{proof}

\begin{proof}[Proof of \Cref{cor:spherical-random-richness}]
Define the $d\times n$ random matrix
\begin{align}
A:=\sqrt d\left[X_1\ \cdots\ X_n\right].
\notag
\end{align}
Its columns are independent isotropic subgaussian random vectors in $\R^d$, each with Euclidean norm $\sqrt d$ almost surely.  The nonasymptotic singular-value theorem for independent subgaussian columns~\citep[Theorem~5.58]{vershynin2012nonasymptotic} gives universal constants $c_0,C_0>0$ such that, for every $t\ge0$, with probability at least $1-2e^{-c_0t^2}$,
\begin{equation}\label{eq:spherical-singular-values}
\sqrt d-C_0\sqrt n-t
\le
s_{\min}(A)
\le
s_{\max}(A)
\le
\sqrt d+C_0\sqrt n+t.
\end{equation}
Take
\begin{align}
t:=\sqrt{\frac1{c_0}\log\frac2\delta}.
\notag
\end{align}
After increasing the universal constant in \eqref{eq:spherical-dimension-condition}, one has
\begin{align}
C_0\sqrt n+t
\le
\frac12\sqrt d.
\notag
\end{align}
On the event in \eqref{eq:spherical-singular-values},
\begin{align}
\frac14 I_n
\preceq
\frac1dA^\top A
=G_{X_{1:n}}
\preceq
\frac94 I_n.
\notag
\end{align}
Thus the condition number is at most nine.  Apply \Cref{thm:well-conditioned-richness}.  The same Gram event is independent of $L$, so all conclusions hold simultaneously for every depth $L\ge2$.  Since every $X_i$ has norm one, $S_L(X_{1:n})=n$.  This proves \eqref{eq:spherical-W-Lambda}--\eqref{eq:spherical-complexity}.
\end{proof}

\begin{proof}[Proof of \Cref{thm:random-repetition-transfer}]
Let $R\in\set{0,1}^{n\times K}$ be the replication matrix
\begin{align}
R_{ij}:=\bone_{\set{X_i=\xi_j}},
\notag
\end{align}
and put
\begin{align}
D:=R^\top R=\diag(N_1,\ldots,N_K).
\notag
\end{align}
For every canonical ladder, evaluating its base-support Gram matrix $K$ on the repeated sample gives $RKR^\top$.  Conversely, every repeated-sample Gram matrix is obtained in this way from the restriction of the same ladder to the base support.  Taking closures yields
\begin{equation}\label{eq:replicated-kernel-family}
\sK_L(X_{1:n})
=
\set{RKR^\top:K\in\sK_L(\boldsymbol\xi)}.
\end{equation}

Let $N_\xi$ be an optimal common covariance for the base design.  By \eqref{eq:replicated-kernel-family}, the matrix $RN_\xi R^\top$ dominates the entire repeated covariance family.  Moreover,
\begin{align*}
\tr\left(RN_\xi R^\top\right)
&=\tr\left(R^\top RN_\xi\right)\\
&=\tr(DN_\xi)\\
&=\sum_{j=1}^K N_j(N_\xi)_{jj}\\
&\le
N_+\tr N_\xi,
\end{align*}
because the diagonal entries of a positive-semidefinite matrix are nonnegative.  Hence
\begin{align}
\tau_L(X_{1:n})
\le
N_+\tau_L(\boldsymbol\xi).
\notag
\end{align}

For the reverse inequality, let $N$ be any common covariance for the repeated sample and define
\begin{align}
P:=D^{-1}R^\top.
\notag
\end{align}
The event $N_->0$ makes $D$ invertible, and $PR=I_K$.  Congruence of $RKR^\top\preceq N$ by $P$ gives
\begin{align}
K\preceq PNP^\top
\qquad
\text{for every }K\in\sK_L(\boldsymbol\xi).
\notag
\end{align}
Thus $PNP^\top$ is feasible for the base common-covariance problem.  Also,
\begin{align}
P^\top P
=RD^{-2}R^\top.
\notag
\end{align}
On the block corresponding to category $j$, this matrix equals $N_j^{-2}\bone\bone^\top$ and has largest eigenvalue $N_j^{-1}$.  Therefore
\begin{align}
\norm{P^\top P}_{\mathrm{op}}
=\frac1{N_-}.
\notag
\end{align}
Since $P^\top P\succeq0$ and $N\succeq0$,
\begin{align*}
\tr(PNP^\top)
&=\tr(P^\top PN)\\
&\le
\norm{P^\top P}_{\mathrm{op}}\tr N\\
&=\frac1{N_-}\tr N.
\end{align*}
Taking the infimum over repeated-sample feasible $N$ gives
\begin{align}
\tau_L(\boldsymbol\xi)
\le
\frac1{N_-}\tau_L(X_{1:n}),
\notag
\end{align}
which completes \eqref{eq:replication-tau}.

Assume now that $S_L(\boldsymbol\xi)>0$.  Since $N_->0$, the repeated normalizing scale is also positive.  It repeats exactly:
\begin{equation}\label{eq:replicated-S}
S_L(X_{1:n})
=
\sum_{j=1}^K N_j\rho_L(\xi_j)^2.
\end{equation}
Combining the upper inequality in \eqref{eq:replication-tau} with \eqref{eq:replicated-S} gives \eqref{eq:replication-Lambda}.

We next transfer stable-threshold Gaussian richness.  Let $\mathcal F$ be a stable threshold family on the base support for some interval $[t,t+\delta]$.  For $A\subseteq[K]$, define its lifted subset of observations by
\begin{align}
A^\uparrow
:=
\set{i\in[n]:X_i=\xi_j\text{ for some }j\in A}.
\notag
\end{align}
Repeating the base trace at every observation shows that $A^\uparrow$ is stable over the same interval.  Conditional on the observed sample, define the Gaussian block sums
\begin{align}
Z_j:=\sum_{i:X_i=\xi_j}G_i.
\notag
\end{align}
They are independent and satisfy $Z_j\sim\mathcal N(0,N_j)$.  For a finite family $\mathcal F\subseteq2^{[K]}$, define
\begin{align}
H_{\mathcal F}(z)
:=
\max_{A\in\mathcal F}
\abs{\sum_{j\in A}z_j}.
\notag
\end{align}
This function is convex because it is the maximum of finitely many absolute linear functionals.  In distribution,
\begin{align}
Z
=\sqrt{N_-}\,G'+H,
\notag
\end{align}
where $G'\sim\mathcal N(0,I_K)$ and $H$ is an independent centered Gaussian vector with covariance $\diag(N_j-N_-)$.  Conditional Jensen gives, for every realization of $G'$,
\begin{align}
\E_H H_{\mathcal F}\left(\sqrt{N_-}G'+H\right)
\ge
H_{\mathcal F}\left(\sqrt{N_-}G'\right).
\notag
\end{align}
Taking expectation in $G'$ yields
\begin{equation}\label{eq:replicated-gaussian-width}
\omega_{\mathrm G}(\mathcal F^\uparrow)
\ge
\sqrt{N_-}\,\omega_{\mathrm G}(\mathcal F).
\end{equation}
The threshold interval length is unchanged.  Multiply the square of \eqref{eq:replicated-gaussian-width} by $\delta/S_L(X_{1:n})$ and optimize over all base stable intervals.  This proves \eqref{eq:replication-W}.  Dividing \eqref{eq:replication-W} by \eqref{eq:replication-Lambda} proves \eqref{eq:replication-ratio}.  Finally, apply \eqref{eq:defect-from-W} to obtain \eqref{eq:replication-defect}.
\end{proof}

\begin{proof}[Proof of \Cref{cor:finite-support-random-richness}]
For every $j$, the count $N_j$ has the binomial distribution with mean $np_j$.  The multiplicative Chernoff inequalities with relative deviation $1/2$ give
\begin{align}
\Pp\left(N_j<\frac12np_j\right)
&\le
\exp\left(-\frac{np_j}{8}\right),
\notag\\
\Pp\left(N_j>\frac32np_j\right)
&\le
\exp\left(-\frac{np_j}{12}\right).
\notag
\end{align}
A union bound over $j\in[K]$ gives
\begin{align}
\Pp\left(
\exists j:\ N_j\notin\left[\frac12np_j,\frac32np_j\right]
\right)
\le
2K\exp\left(-\frac{np_{\min}}{12}\right).
\notag
\end{align}
Under \eqref{eq:occupancy-sample-size}, the right-hand side is at most $\delta$.  On the complementary event,
\begin{align*}
N_-
&\ge
\frac12np_{\min},\\
N_+
&\le
\frac32np_{\max},
\end{align*}
so \eqref{eq:occupancy-ratio} follows.  Assume now that $S_L(\boldsymbol\xi)>0$ and that \eqref{eq:base-richness-assumption} holds.  Combine \eqref{eq:occupancy-ratio} with \eqref{eq:replication-ratio} and \eqref{eq:base-richness-assumption} to obtain \eqref{eq:finite-support-W-Lambda}.  Equation \eqref{eq:finite-support-defect} follows from \eqref{eq:defect-from-W}.
\end{proof}

\begin{proof}[Proof of \Cref{cor:categorical-orthogonal}]
For the $K$-point orthogonal base design, \Cref{cor:minimax-orthogonal} gives, simultaneously for every depth,
\begin{align}
\mathfrak W_{L-1}(\boldsymbol\xi)
\asymp
\Lambda_L(\boldsymbol\xi)^2
\asymp
K^{1-\vartheta_L}.
\notag
\end{align}
The probability assumption \eqref{eq:balanced-categorical-probabilities} gives
\begin{align}
p_{\min}\ge\frac{c_-}{K},
\qquad
p_{\max}\le\frac{c_+}{K}.
\notag
\end{align}
After increasing $C$, condition \eqref{eq:categorical-sample-size} implies \eqref{eq:occupancy-sample-size}.  On the resulting occupancy event,
\begin{align}
N_-
\ge
\frac{c_-n}{2K},
\qquad
N_+
\le
\frac{3c_+n}{2K}.
\notag
\end{align}
Every support point has norm one, so
\begin{align}
S_L(\boldsymbol\xi)=K,
\qquad
S_L(X_{1:n})=n.
\notag
\end{align}
Equations \eqref{eq:replication-W} and \eqref{eq:replication-Lambda} therefore give
\begin{align*}
\mathfrak W_{L-1}(X_{1:n})
&\ge
\frac{c_-}{2}\mathfrak W_{L-1}(\boldsymbol\xi),\\
\Lambda_L(X_{1:n})^2
&\le
\frac{3c_+}{2}\Lambda_L(\boldsymbol\xi)^2.
\end{align*}
Together with $\mathfrak W_{L-1}\le\Lambda_L^2$ and the orthogonal base equivalence, these inequalities prove \eqref{eq:categorical-W-Lambda}.  Equation \eqref{eq:categorical-defect} follows either from \Cref{cor:finite-support-random-richness} or directly from \eqref{eq:defect-from-W}.

The profile upper bound and \eqref{eq:categorical-W-Lambda} give
\begin{align}
\Ghat_{X_{1:n}}(B_L(r))
\le
C\frac r n
\sqrt{nK^{1-\vartheta_L}}
=
C\frac{rK^{1/2-2^{-L}}}{\sqrt n}.
\notag
\end{align}
The stable-threshold lower bound \eqref{eq:threshold-two-sided-complexity} gives the matching lower estimate.  This proves \eqref{eq:categorical-complexity}.
\end{proof}

\begin{proof}[Proof of \Cref{cor:random-hierarchical-support}]
The macroscopic-threshold assumption implies $S_L(\boldsymbol\xi)>0$, so \Cref{lem:macroscopic-W} applies and the base support satisfies
\begin{align}
\mathfrak W_{L-1}(\boldsymbol\xi)
\ge
\frac{2\kappa_0}{\pi}.
\notag
\end{align}
By \Cref{cor:hierarchical-threshold},
\begin{align}
\Lambda_L(\boldsymbol\xi)^2
\le
2q(H+1).
\notag
\end{align}
Therefore
\begin{equation}\label{eq:base-hierarchical-ratio}
\frac{\mathfrak W_{L-1}(\boldsymbol\xi)}{\Lambda_L(\boldsymbol\xi)^2}
\ge
\frac{\kappa_0}{\pi q(H+1)}.
\end{equation}
On the high-probability occupancy event from \Cref{cor:finite-support-random-richness}, combine \eqref{eq:replication-ratio}, \eqref{eq:occupancy-ratio}, and \eqref{eq:base-hierarchical-ratio}.  This proves \eqref{eq:random-hierarchical-W}.  Apply \eqref{eq:defect-from-W} to obtain \eqref{eq:random-hierarchical-defect}.
\end{proof}

\subsection{Proofs for Perturbation Stability}

\begin{proof}[Proof of \Cref{thm:matched-sample-stability}]
Put $\varepsilon:=\eps(\bx,\by)$ and $\vartheta:=\vartheta_L$.  Let $u$ be a unit-sphere element of any admissible level-$(L-1)$ RKHS, and define
\begin{align}
a_i:=u(x_i),
\qquad
b_i:=u(y_i).
\notag
\end{align}
Since $\Chat_{L-1}(u)\le1$, the recursive H\"older estimate \eqref{eq:holder-scale} gives
\begin{equation}\label{eq:proof-trace-perturbation}
|a_i-b_i|
\le
\norm{x_i-y_i}_2^{\,2\vartheta}
\le
\varepsilon^{2\vartheta}.
\end{equation}

For $s\in\R$, let $\phi_s$ be the signed-interval Brownian feature
\begin{align}
\phi_s
:=
\begin{cases}
\bone_{[0,s]},&s\ge0,\\
-\bone_{[s,0]},&s<0.
\end{cases}
\notag
\end{align}
It satisfies
\begin{align}
\left\langle\phi_s,\phi_t\right\rangle_{L_2(\R)}
=\kB(s,t).
\notag
\end{align}
Consequently,
\begin{align}
\norm{\phi_s-\phi_t}_{L_2(\R)}^2
&=\kB(s,s)+\kB(t,t)-2\kB(s,t)\notag\\
&=|s-t|.
\notag
\end{align}
Equation \eqref{eq:proof-trace-perturbation} therefore implies
\begin{equation}\label{eq:proof-feature-perturbation}
\norm{\phi_{a_i}-\phi_{b_i}}_{L_2(\R)}
\le
\varepsilon^{\vartheta}.
\end{equation}

Fix $z\in\R^n$.  The feature representation of the Brownian covariance gives
\begin{align*}
\sqrt{z^\top\mathbf B(a)z}
&=
\norm{\sum_{i=1}^n z_i\phi_{a_i}}_{L_2(\R)}\\
&\le
\norm{\sum_{i=1}^n z_i\phi_{b_i}}_{L_2(\R)}
+
\norm{\sum_{i=1}^n z_i\left(\phi_{a_i}-\phi_{b_i}\right)}_{L_2(\R)}\\
&\le
\sqrt{z^\top\mathbf B(b)z}
+
\sum_{i=1}^n|z_i|\norm{\phi_{a_i}-\phi_{b_i}}_{L_2(\R)}\\
&\le
\sqrt{z^\top\mathbf B(b)z}
+
\varepsilon^\vartheta\norm z_1\\
&\le
\sqrt{z^\top\mathbf B(b)z}
+
\sqrt n\,\varepsilon^\vartheta\norm z_2.
\end{align*}
The same admissible function $u$ is evaluated on both samples.  Taking the supremum over all actual previous-layer unit functions, and then using continuity to pass to the trace closures in \Cref{thm:dirac-reduction}, gives
\begin{align}
h_{L,\bx}(z)
\le
h_{L,\by}(z)
+
\sqrt n\,\varepsilon^\vartheta\norm z_2.
\notag
\end{align}
Interchanging $\bx$ and $\by$ proves \eqref{eq:matched-support-stability}.

For $q=1$, take expectation in \eqref{eq:matched-support-stability} and use
\begin{align}
\E\norm G_2
\le
\left(\E\norm G_2^2\right)^{1/2}
=\sqrt n.
\notag
\end{align}
For $q=2$, take $L_2$-norms and apply Minkowski's inequality:
\begin{align*}
W_{G,2}(L,\bx)
&\le
W_{G,2}(L,\by)
+
\sqrt n\,\varepsilon^\vartheta
\left(\E\norm G_2^2\right)^{1/2}\\
&=
W_{G,2}(L,\by)
+n\varepsilon^\vartheta.
\end{align*}
The reverse inequalities follow by symmetry.  This proves \eqref{eq:matched-gaussian-stability}.

We next prove \eqref{eq:matched-tau-stability}.  Let $\nu\in\mathfrak P_{\preceq I_n}$.  Equation \eqref{eq:matched-support-stability} and Minkowski's inequality give
\begin{align*}
\left(\int h_{L,\bx}(z)^2\,d\nu(z)\right)^{1/2}
&\le
\left(\int h_{L,\by}(z)^2\,d\nu(z)\right)^{1/2}\\
&\quad+
\sqrt n\,\varepsilon^\vartheta
\left(\int\norm z_2^2\,d\nu(z)\right)^{1/2}.
\end{align*}
Because the covariance of $\nu$ is dominated by $I_n$,
\begin{align}
\int\norm z_2^2\,d\nu(z)
=
\tr\left(\int zz^\top\,d\nu(z)\right)
\le n.
\notag
\end{align}
Hence
\begin{align}
\left(\int h_{L,\bx}(z)^2\,d\nu(z)\right)^{1/2}
\le
\left(\int h_{L,\by}(z)^2\,d\nu(z)\right)^{1/2}
+n\varepsilon^\vartheta.
\notag
\end{align}
Take the supremum over $\nu$ and use \Cref{thm:robust-minimax}.  This gives
\begin{align}
\sqrt{\tau_L(\bx)}
\le
\sqrt{\tau_L(\by)}+n\varepsilon^\vartheta.
\notag
\end{align}
Symmetry proves \eqref{eq:matched-tau-stability}.

Put $r_i:=\norm{x_i}_2$, $s_i:=\norm{y_i}_2$, and $\beta:=2\vartheta\in(0,1]$.  For nonnegative $r\ge s$, subadditivity of the map $t\mapsto t^\beta$ gives
\begin{align}
r^\beta
=\left(s+(r-s)\right)^\beta
\le
s^\beta+(r-s)^\beta.
\notag
\end{align}
Thus
\begin{align}
|r^\beta-s^\beta|
\le
|r-s|^\beta.
\notag
\end{align}
The reverse triangle inequality yields $|r_i-s_i|\le\norm{x_i-y_i}_2\le\varepsilon$, so
\begin{align}
\abs{\norm{x_i}_2^{2\vartheta}-\norm{y_i}_2^{2\vartheta}}
\le
\varepsilon^{2\vartheta}.
\notag
\end{align}
Summing over $i$ proves \eqref{eq:matched-S-stability}.

If $S_L(\by)>n\varepsilon^{2\vartheta}$, then \eqref{eq:matched-tau-stability} and \eqref{eq:matched-S-stability} imply
\begin{align*}
\Lambda_L(\bx)^2
&=\frac{\tau_L(\bx)}{S_L(\bx)}\\
&\le
\frac{\left(\sqrt{\tau_L(\by)}+n\varepsilon^\vartheta\right)^2}
{S_L(\by)-n\varepsilon^{2\vartheta}},
\end{align*}
which is \eqref{eq:matched-profile-stability}.

Finally, assume \eqref{eq:matched-defect-condition}.  Equation \eqref{eq:matched-gaussian-stability} gives
\begin{align}
W_{G,2}(L,\bx)
\ge
(1-\rho)W_{G,2}(L,\by).
\notag
\end{align}
Equation \eqref{eq:matched-tau-stability} gives
\begin{align*}
\sqrt{\tau_L(\bx)}
&\le
\sqrt{\tau_L(\by)}+\rho W_{G,2}(L,\by)\\
&=
\left(\chi_L(\by)+\rho\right)W_{G,2}(L,\by).
\end{align*}
Divide the last inequality by the preceding lower bound on $W_{G,2}(L,\bx)$.  This proves \eqref{eq:matched-defect-stability}.
\end{proof}

\begin{proof}[Proof of \Cref{thm:threshold-perturbation-stability}]
Let
\begin{align}
\mathcal F_y
:=
\mathcal F_{L-1,\by}^{\mathrm{st}}(t,\delta).
\notag
\end{align}
We show that every set in $\mathcal F_y$ remains stable on $\bx$ after trimming both ends of the interval by $\eta$.

Take first $A\in\mathcal F_{L-1,\by}^{+,\mathrm{st}}(t,\delta)$.  By definition, there is a trace $a^y\in\sS_{L-1}(\by)$ such that
\begin{align}
\set{i:a_i^y\ge s}=A
\qquad
\text{for every }s\in[t,t+\delta].
\notag
\end{align}
Choose actual previous-layer unit functions $u_m$ whose traces on $\by$ converge to $a^y$.  Their traces on $\bx$ are bounded coordinatewise by the pointwise scale.  After passage to a subsequence, they converge to some $a^x\in\sS_{L-1}(\bx)$.  The recursive H\"older estimate gives, for every $m$ and every $i$,
\begin{align}
|u_m(x_i)-u_m(y_i)|
\le
\eta.
\notag
\end{align}
Passing to the limit yields
\begin{equation}\label{eq:proof-closure-trace-perturbation}
|a_i^x-a_i^y|
\le
\eta,
\qquad i\in[n].
\end{equation}
For $i\in A$, membership at the upper endpoint gives $a_i^y\ge t+\delta$.  For $i\notin A$, nonmembership at the lower endpoint gives $a_i^y<t$.  Hence, for every
\begin{align}
s\in[t+\eta,t+\delta-\eta],
\notag
\end{align}
equation \eqref{eq:proof-closure-trace-perturbation} gives $a_i^x\ge s$ when $i\in A$ and $a_i^x<s$ when $i\notin A$.  Therefore
\begin{align}
A\in
\mathcal F_{L-1,\bx}^{+,\mathrm{st}}
\left(t+\eta,\delta-2\eta\right)
\notag
\end{align}
whenever $\delta>2\eta$.  The negative-threshold case follows by applying the same argument to the traces $-a^y$ and $-a^x$.  Thus
\begin{equation}\label{eq:proof-stable-family-inclusion}
\mathcal F_y
\subseteq
\mathcal F_{L-1,\bx}^{\mathrm{st}}
\left(t+\eta,\delta-2\eta\right)
\end{equation}
when $\delta>2\eta$.  If $\delta\le2\eta$, the lower bound in \eqref{eq:W-perturbation-stability} is zero and there is nothing to prove.

Assume $\delta>2\eta$.  Gaussian set width is monotone under inclusion, so \eqref{eq:proof-stable-family-inclusion} gives
\begin{align*}
\mathfrak W_{L-1}(\bx)
&\ge
\frac{\delta-2\eta}{S_L(\bx)}
\omega_{\mathrm G}(\mathcal F_y)^2\\
&=
\frac{\delta-2\eta}{\delta}
\frac{S_L(\by)}{S_L(\bx)}
\frac{\delta\,\omega_{\mathrm G}(\mathcal F_y)^2}{S_L(\by)}\\
&\ge
(1-\zeta)
\frac{\delta-2\eta}{\delta}
\frac{S_L(\by)}{S_L(\bx)}
\mathfrak W_{L-1}(\by).
\end{align*}
This proves \eqref{eq:W-perturbation-stability}.

We now impose the additional relative-error assumptions.  Equations \eqref{eq:matched-tau-stability} and \eqref{eq:matched-S-stability} imply
\begin{align}
\tau_L(\bx)
&\le
(1+\rho)^2\tau_L(\by),
\notag\\
(1-\rho)S_L(\by)
&\le
S_L(\bx)
\le
(1+\rho)S_L(\by).
\notag
\end{align}
Consequently,
\begin{equation}\label{eq:proof-Lambda-perturbation-relative}
\Lambda_L(\bx)^2
\le
\frac{(1+\rho)^2}{1-\rho}
\Lambda_L(\by)^2.
\end{equation}
The threshold-margin condition gives
\begin{align}
\frac{\delta-2\eta}{\delta}
\ge
1-\rho.
\notag
\end{align}
Using also $S_L(\by)/S_L(\bx)\ge1/(1+\rho)$ in \eqref{eq:W-perturbation-stability}, we obtain
\begin{align*}
\mathfrak W_{L-1}(\bx)
&\ge
(1-\zeta)
\frac{1-\rho}{1+\rho}
\mathfrak W_{L-1}(\by)\\
&\ge
c(1-\zeta)
\frac{1-\rho}{1+\rho}
\Lambda_L(\by)^2.
\end{align*}
Rearranging \eqref{eq:proof-Lambda-perturbation-relative} gives
\begin{align}
\Lambda_L(\by)^2
\ge
\frac{1-\rho}{(1+\rho)^2}
\Lambda_L(\bx)^2.
\notag
\end{align}
Substitution proves \eqref{eq:W-Lambda-perturbation-stability}.
\end{proof}

\begin{proof}[Proof of \Cref{cor:noisy-repetition-transfer}]
By \eqref{eq:replication-defect},
\begin{align}
\chi_L(Y_{1:n})
\le
C_{\mathrm{rep}}.
\notag
\end{align}
Apply \Cref{thm:matched-sample-stability} to the matched samples $X_{1:n}$ and $Y_{1:n}$.  Condition \eqref{eq:noisy-repetition-perturbation-condition} is exactly \eqref{eq:matched-defect-condition}.  Equation \eqref{eq:matched-defect-stability} therefore gives \eqref{eq:noisy-repetition-defect}.  The final statement follows directly from \Cref{thm:threshold-perturbation-stability}.
\end{proof}

\subsection{Proofs for Integrated Threshold Profiles}

\begin{proof}[Proof of \Cref{thm:integrated-threshold-sandwich}]
If $S_L(\bx)=0$, then every previous-layer trace and every top covariance vanish by the pointwise scale, so all quantities in \eqref{eq:integrated-threshold-chain} are zero.  Assume $S_L(\bx)>0$.

We first prove the lower integrated bound.  Fix $a\in\sS_{L-1}(\bx)$.  The support identity gives, pointwise in $G$,
\begin{align}
h_{L,\bx}(G)^2
\ge
G^\top\mathbf B(a)G.
\notag
\end{align}
Taking expectation and using $\E GG^\top=I_n$,
\begin{align*}
W_{G,2}(L,\bx)^2
&=\E h_{L,\bx}(G)^2\\
&\ge
\E G^\top\mathbf B(a)G\\
&=\tr\mathbf B(a)\\
&=\sum_{i=1}^n\kB(a_i,a_i)\\
&=\sum_{i=1}^n|a_i|\\
&=\norm a_1.
\end{align*}
Taking the supremum over $a$ and dividing by $S_L(\bx)$ proves
\begin{align}
\mathfrak A_{L-1}(\bx)
\le
\frac{W_{G,2}(L,\bx)^2}{S_L(\bx)}.
\notag
\end{align}
The inequality $W_{G,2}(L,\bx)^2\le\tau_L(\bx)$ follows by applying an optimal common covariance $N_*$:
\begin{align*}
W_{G,2}(L,\bx)^2
&=\E h_{L,\bx}(G)^2\\
&\le\E G^\top N_*G\\
&=\tr N_*\\
&=\tau_L(\bx).
\end{align*}
After division by $S_L(\bx)$, this gives the middle inequality in \eqref{eq:integrated-threshold-chain}.

We next prove the integrated upper bound.  Let $(H_s)$ be an admissible threshold covariance field.  For every $a\in\sS_{L-1}(\bx)$, the signed Brownian layer-cake identity gives
\begin{align*}
\mathbf B(a)
&=\int_0^\infty
\left(
v_s^+(a)v_s^+(a)^\top
+
v_s^-(a)v_s^-(a)^\top
\right)ds\\
&\preceq
2\int_0^\infty H_s\,ds.
\end{align*}
The integral is an entrywise finite-dimensional integral of positive-semidefinite matrices.  Thus
\begin{align}
N_H:=2\int_0^\infty H_s\,ds
\notag
\end{align}
is positive semidefinite and dominates every Dirac Brownian matrix.  By \Cref{thm:dirac-reduction}, it is feasible for the full common-covariance problem.  Hence
\begin{align}
\tau_L(\bx)
\le
2\int_0^\infty\tr H_s\,ds.
\notag
\end{align}
Take the infimum over admissible fields and divide by $S_L(\bx)$ to obtain
\begin{align}
\Lambda_L(\bx)^2
\le
2\overline{\mathfrak T}_{L-1}(\bx).
\notag
\end{align}

It remains to compare the integrated profile with the global decomposition profile.  Let $\mathscr D$ attain $\mathfrak D_{L-1}(\bx)$, and abbreviate
\begin{align}
s:=s(\mathscr D),
\qquad
\Delta:=\Delta(\mathscr D).
\notag
\end{align}
Define
\begin{align}
H_{\mathscr D}
:=
s\sum_{D\in\mathscr D}\bone_D\bone_D^\top.
\notag
\end{align}
The proof of \Cref{thm:recursive-threshold} shows that
\begin{align}
\bone_A\bone_A^\top
\preceq
H_{\mathscr D}
\qquad
\text{for every }A\in\mathcal F_{L-1}(\bx),
\notag
\end{align}
and that every diagonal entry of $H_{\mathscr D}$ is at most $s\Delta$.
For $r\ge0$, put
\begin{align}
I_r:=\set{i:d_{L,i}\ge r},
\qquad
P_r:=\diag(\bone_{I_r}),
\qquad
H_r:=P_rH_{\mathscr D}P_r.
\notag
\end{align}
Every signed threshold vector at level $r$ is supported in $I_r$.  Congruence by $P_r$ therefore shows that $H_r$ dominates its rank-one outer product.  Thus $(H_r)$ is an admissible threshold covariance field.  Its trace satisfies
\begin{align}
\tr H_r
\le
s\Delta|I_r|.
\notag
\end{align}
The scalar layer-cake identity gives
\begin{align*}
\int_0^\infty|I_r|\,dr
&=\sum_{i=1}^n\int_0^\infty\bone_{\{r\le d_{L,i}\}}\,dr\\
&=\sum_{i=1}^n d_{L,i}\\
&=S_L(\bx).
\end{align*}
Therefore
\begin{align*}
\overline{\mathfrak T}_{L-1}(\bx)
&\le
\frac1{S_L(\bx)}
\int_0^\infty\tr H_r\,dr\\
&\le
s\Delta\\
&=\mathfrak D_{L-1}(\bx).
\end{align*}
This completes \eqref{eq:integrated-threshold-chain}.  Finally,
\begin{align*}
\chi_L(\bx)^2
&=\frac{\tau_L(\bx)}{W_{G,2}(L,\bx)^2}\\
&=\frac{\Lambda_L(\bx)^2}{W_{G,2}(L,\bx)^2/S_L(\bx)}\\
&\le
\frac{2\overline{\mathfrak T}_{L-1}(\bx)}{\mathfrak A_{L-1}(\bx)},
\end{align*}
which proves \eqref{eq:integrated-defect-bound}.
\end{proof}

\begin{proof}[Proof of \Cref{thm:integrated-ordered-laminar}]
We first treat ordered interval traces.  Pad the ordered index set to the next power of two and use the complete dyadic interval tree, exactly as in the proof of \Cref{cor:ordered-threshold}.  Every interval of sample indices is a disjoint union of at most $2H_n$ dyadic intervals, and every sample index belongs to at most $H_n$ dyadic intervals.

Fix $a\in\sS_{L-1}(\bx)$.  Choose actual previous-layer unit traces converging to $a$.  For every threshold $s$ that is different from all coordinates of $a$, the positive and negative threshold sets of the approximating traces agree with those of $a$ for all sufficiently large indices.  The ordered interval hypothesis therefore passes to $a$ at every such $s$.  Since a vector in $\R^n$ has only finitely many coordinate values, the exceptional thresholds have Lebesgue measure zero.  Thus, for every $a$, at almost every threshold $s$, each signed threshold set is a union of at most $q(s)$ ordinary intervals, and hence a disjoint union of at most
\begin{align}
2q(s)H_n
\notag
\end{align}
dyadic intervals.  Let $\mathscr D_n$ be the full dyadic dictionary and define
\begin{align}
M_s
:=
2q(s)H_n
\sum_{D\in\mathscr D_n}\bone_D\bone_D^\top.
\notag
\end{align}
The same Cauchy--Schwarz argument used in \eqref{eq:dict-rank-one-domination} shows that $M_s$ dominates the outer product of every signed threshold vector at level $s$.

Put
\begin{align}
I_s:=\set{i:p_i\ge s},
\qquad
P_s:=\diag(\bone_{I_s}),
\qquad
H_s:=P_sM_sP_s.
\notag
\end{align}
Every threshold vector at level $s$ is supported in $I_s$, so $(H_s)$ is an admissible threshold covariance field.  Since every active coordinate belongs to at most $H_n$ dyadic intervals,
\begin{align}
\tr H_s
\le
2q(s)H_n^2|I_s|.
\notag
\end{align}
Integrating and applying Tonelli's theorem,
\begin{align*}
\int_0^\infty q(s)|I_s|\,ds
&=\sum_{i=1}^n\int_0^{p_i}q(s)\,ds\\
&=\sum_{i=1}^nQ(p_i).
\end{align*}
Definition \eqref{eq:integrated-threshold-domination} therefore gives
\begin{align}
\overline{\mathfrak T}_{L-1}(\bx)
\le
2H_n^2
\frac{\sum_iQ(p_i)}{\sum_ip_i},
\notag
\end{align}
which is \eqref{eq:ordered-integrated-profile}.

For a fixed admissible witness $u_\star$, the trace
\begin{align}
a_\star:=\left(u_\star(x_i)\right)_{i=1}^n
\notag
\end{align}
belongs to $\sS_{L-1}(\bx)$.  Hence
\begin{align}
\mathfrak A_{L-1}(\bx)
\ge
\frac{\sum_i|u_\star(x_i)|}{\sum_ip_i}.
\notag
\end{align}
Substitute this estimate and \eqref{eq:ordered-integrated-profile} into \eqref{eq:integrated-defect-bound}.  The factors $\sum_ip_i$ cancel, giving \eqref{eq:ordered-integrated-defect}.

For the laminar model, let $\mathscr D$ be the collection of all tree nodes.  The same closure argument shows that, for every $a\in\sS_{L-1}(\bx)$ and almost every threshold $s$, every signed threshold set is a disjoint union of at most $q(s)$ dictionary nodes.  Each sample index belongs to at most $H+1$ nodes.  Repeating the preceding construction with
\begin{align}
M_s
:=
q(s)
\sum_{D\in\mathscr D}\bone_D\bone_D^\top
\notag
\end{align}
gives
\begin{align}
\tr(P_sM_sP_s)
\le
q(s)(H+1)|I_s|.
\notag
\end{align}
Integration proves \eqref{eq:laminar-integrated-profile}, and the same witness lower bound for $\mathfrak A_{L-1}$ proves \eqref{eq:laminar-integrated-defect}.
\end{proof}

\begin{proof}[Proof of \Cref{cor:non-atomic-integrated-richness}]
Non-atomicity of the law of $T$ implies that the sample parameters are pairwise distinct with probability one, so they admit the order required by \Cref{thm:integrated-ordered-laminar}.  The deterministic theorem gives
\begin{align}
\chi_L(X_{1:n})^2
\le
2C_{\mathrm{geom}}
\frac{n^{-1}\sum_{i=1}^nV_i}{n^{-1}\sum_{i=1}^nU_i},
\notag
\end{align}
where
\begin{align}
U_i:=|u_\star(X_i)|,
\qquad
V_i:=Q\left(\norm{X_i}_2^{2\vartheta_L}\right).
\notag
\end{align}

The bounded-variable form of Hoeffding's inequality~\citep{hoeffding1963probability} gives
\begin{align}
\Pp\left(
\frac1n\sum_{i=1}^nU_i<m_U-\epsilon_U
\right)
&\le\frac\delta2,
\notag\\
\Pp\left(
\frac1n\sum_{i=1}^nV_i>m_V+\epsilon_V
\right)
&\le\frac\delta2.
\notag
\end{align}
A union bound shows that both inequalities fail with probability at most $\delta$.  On their common complement, if $m_U>\epsilon_U$,
\begin{align}
\chi_L(X_{1:n})^2
\le
2C_{\mathrm{geom}}
\frac{m_V+\epsilon_V}{m_U-\epsilon_U},
\notag
\end{align}
which proves \eqref{eq:non-atomic-general-defect}.  If $\epsilon_U\le m_U/2$ and $\epsilon_V\le m_V/2$, then
\begin{align}
\frac{m_V+\epsilon_V}{m_U-\epsilon_U}
\le
\frac{3m_V/2}{m_U/2}
=
3\frac{m_V}{m_U}.
\notag
\end{align}
Substitution proves \eqref{eq:non-atomic-simplified-defect}.
\end{proof}

\section{Proofs for Approximation Widths}\label[appendix]{app:approximation}

\begin{proof}[Proof of \Cref{thm:kolmogorov-width}]
We separate the exact width identity from the common-covariance upper bounds.

We first reduce the distance to an orthogonal residual.
Fix a linear subspace $V\subseteq\R^n$ and let $P_V$ be the Euclidean orthogonal projection.  For every $a\in\R^n$, the closest point in $V$ is $P_Va$.  Therefore
\begin{equation}\label{eq:app-dist-proj}
\operatorname{dist}_n(a,V)
=\frac{1}{\sqrt n}\norm{P_{V^\perp}a}_2.
\end{equation}
Indeed, $a=P_Va+P_{V^\perp}a$ is an orthogonal decomposition.  The claim follows from the Pythagorean identity.

We next compute the worst residual on one RKHS ellipsoid.
Fix $K\succeq0$.  By \eqref{eq:ellipsoid}, every $a\in E(K)$ has the form $a=K^{1/2}u$ with $\norm{u}_2\le1$.  Hence
\begin{align}
\sup_{a\in E(K)}\norm{P_{V^\perp}a}_2^2
&=
\sup_{\norm{u}_2\le1}
\norm{P_{V^\perp}K^{1/2}u}_2^2
\nonumber\\
&=
\norm{P_{V^\perp}K^{1/2}}_{\op}^2
\nonumber\\
&=
\lambda_{\max}\!\bigl(
P_{V^\perp}K P_{V^\perp}
\bigr).
\label{eq:app-one-ellipsoid}
\end{align}
The second equality is the definition of the operator norm.  For the last equality, set $A=P_{V^\perp}K^{1/2}$.  Then $AA^\top=P_{V^\perp}KP_{V^\perp}$, and the squared operator norm of $A$ is the largest eigenvalue of $AA^\top$.

We then pass from one ellipsoid to the full adaptive trace body.
By \Cref{thm:trace-body},
\begin{align}
\cA_L(\bx)=\cl\bigcup_{K\in\sK_L(\bx)}E(K).
\notag
\end{align}
The map $a\mapsto\norm{P_{V^\perp}a}_2$ is continuous, so it has the same supremum on a set and on its closure.  By \eqref{eq:ball-scaling}, $B_L(r)=rB_L(1)$; hence the closed trace body of $B_L(r)$ is exactly $r\cA_L(\bx)$.  Therefore \eqref{eq:app-dist-proj} and \eqref{eq:app-one-ellipsoid} imply
\begin{align}
&\sup_{f\in B_L(r)}
\operatorname{dist}_n\!\bigl((f(x_i))_{i=1}^n,V\bigr)^2
\nonumber\\
&\qquad=
\frac{r^2}{n}
\sup_{K\in\sK_L(\bx)}
\lambda_{\max}(P_{V^\perp}KP_{V^\perp}).
\label{eq:app-fixed-V-width}
\end{align}
Taking the infimum over all $V$ with $\dim V\le m$ proves \eqref{eq:width-exact}.

We next use one common covariance.
Let $N_*$ minimize \eqref{eq:tau-def}.  Then $K\preceq N_*$ for every $K\in\sK_L(\bx)$.  Congruence by the orthogonal projection $P_{V^\perp}$ preserves Loewner order, so
\begin{equation}\label{eq:app-compression-order}
P_{V^\perp}K P_{V^\perp}
\preceq
P_{V^\perp}N_*P_{V^\perp}.
\end{equation}
Choose $V=V_m^*$ to be the span of eigenvectors of $N_*$ corresponding to its $m$ largest eigenvalues.  In an orthonormal eigenbasis of $N_*$, the compression to $(V_m^*)^\perp$ has eigenvalues
\begin{align}
\lambda_{m+1}(N_*),\ldots,\lambda_n(N_*).
\notag
\end{align}
Hence
\begin{equation}\label{eq:app-courant}
\lambda_{\max}\!\bigl(
P_{(V_m^*)^\perp}N_*P_{(V_m^*)^\perp}
\bigr)
=
\lambda_{m+1}(N_*).
\end{equation}
Combining \eqref{eq:width-exact}, \eqref{eq:app-compression-order}, and \eqref{eq:app-courant} proves \eqref{eq:width-spectral-envelope}.

Finally, we convert spectral decay into the trace and profile bounds.
Because the eigenvalues are nonincreasing and nonnegative,
\begin{align}
(m+1)\lambda_{m+1}(N_*)
&\le\sum_{j=1}^{m+1}\lambda_j(N_*)
\nonumber\\
&\le\sum_{j=1}^{n}\lambda_j(N_*)
=\tr N_*
=\tau_L(\bx).
\label{eq:app-eigen-trace}
\end{align}
Insert \eqref{eq:app-eigen-trace} into \eqref{eq:width-spectral-envelope} to obtain \eqref{eq:width-trace-envelope}.  The identity $\tau_L=\Lambda_L^2S_L$ gives \eqref{eq:width-profile-envelope}.  Finally,
\begin{align}
S_L(\bx)
\le n\bigl(\max_i\norm{x_i}_2\bigr)^{2^{-(L-2)}},
\notag
\end{align}
so taking square roots gives \eqref{eq:width-radius-envelope}.  If $m=n$, choose $V=\R^n$ in \eqref{eq:kolmogorov-width-def}; then every residual is zero and $d_{n,L}(\bx;r)=0$.
\end{proof}

\begin{proof}[Proof of \Cref{cor:recursive-spectral-width}]
Fix a decomposition dictionary $\mathscr D$.  The construction in the proof of \Cref{thm:recursive-threshold} produces the matrix $N_{\mathscr D}$ in \eqref{eq:N-D-def} and shows that
\begin{align}
N_{\mathscr D}\succeq\mathbf B(a)
\qquad
\text{for every }a\in\sS_{L-1}(\bx).
\notag
\end{align}
By \Cref{thm:dirac-reduction}, this is equivalent to
\begin{align}
N_{\mathscr D}\succeq K
\qquad
\text{for every }K\in\sK_L(\bx).
\notag
\end{align}
Thus the common-covariance argument in the proof of \Cref{thm:kolmogorov-width} applies with $N_{\mathscr D}$ in place of the optimizer $N_*$.  Choosing the first $m$ eigenvectors of $N_{\mathscr D}$ yields
\begin{align}
d_{m,L}(\bx;r)
\le\frac{r}{\sqrt n}
\sqrt{\lambda_{m+1}(N_{\mathscr D})},
\notag
\end{align}
which is \eqref{eq:recursive-spectral-width}.

For the second claim, \eqref{eq:app-eigen-trace} gives
\begin{align}
\lambda_{m+1}(N_{\mathscr D})
\le\frac{\tr N_{\mathscr D}}{m+1}.
\notag
\end{align}
The trace calculation in the recursive-threshold proof gives
\begin{align}
\tr N_{\mathscr D}
\le2s(\mathscr D)\Delta(\mathscr D)S_L(\bx).
\notag
\end{align}
Taking the best dictionary gives
\begin{align}
\lambda_{m+1}(N_{\mathscr D})
\le
\frac{2\mathfrak D_{L-1}(\bx)S_L(\bx)}{m+1}
\notag
\end{align}
for a minimizing dictionary.  Substitute this estimate into \eqref{eq:recursive-spectral-width} and use
\begin{align}
\frac{S_L(\bx)}{n}
\le
\bigl(\max_i\norm{x_i}_2\bigr)^{2^{-(L-2)}}
\notag
\end{align}
to obtain \eqref{eq:recursive-width-D}.
\end{proof}

\begin{lemma}[Spectrum of the balanced-tree incidence covariance]\label{lem:haar-spectrum}
Let $n=2^H$ and let $\mathscr T$ be the set of all nodes of a complete balanced binary partition tree on $[n]$.  Define
\begin{align}
C_{\mathscr T}:=\sum_{D\in\mathscr T}\bone_D\bone_D^\top.
\notag
\end{align}
Then the discrete Haar basis associated with $\mathscr T$ diagonalizes $C_{\mathscr T}$.  The constant vector has eigenvalue $2n-1$.  For every internal tree node $D$ of cardinality $s$, the Haar wavelet supported on $D$ and constant with opposite signs on its two children has eigenvalue $s-1$.  Consequently, if the eigenvalues are arranged in nonincreasing order, then
\begin{equation}\label{eq:haar-spectrum-bound}
\lambda_{m+1}(C_{\mathscr T})
\le\frac{2n}{m+1},
\qquad
m=0,\ldots,n-1.
\end{equation}
\end{lemma}

\begin{proof}
We prove the eigenvalue formulas directly.

We first consider the constant vector.
Let $\bone$ be the all-ones vector.  For a node $D$,
\begin{align}
\bone_D\bone_D^\top\bone=|D|\bone_D.
\notag
\end{align}
Fix a leaf index $i$.  Exactly one node of each size $n,n/2,\ldots,1$ contains $i$.  Therefore the $i$-th coordinate of $C_{\mathscr T}\bone$ is
\begin{align}
n+\frac n2+\frac n4+\cdots+1=2n-1.
\notag
\end{align}
This holds for every $i$, so
\begin{align}
C_{\mathscr T}\bone=(2n-1)\bone.
\notag
\end{align}

We next consider one Haar wavelet.
Let $D$ be an internal node of size $s$, with children $D_+$ and $D_-$ of size $s/2$.  Let $h_D$ be the normalized vector supported on $D$, constant and positive on $D_+$, constant and negative on $D_-$, and satisfying $\norm{h_D}_2=1$.  Every ancestor $A\supseteq D$ and the node $D$ itself satisfy
\begin{align}
\bone_A^\top h_D=0,
\notag
\end{align}
because the entries of $h_D$ sum to zero on $D$.  Every node disjoint from $D$ also gives zero.  It remains to sum over the proper descendants of $D$.

For an integer $\ell\in\{1,\ldots,\log_2 s\}$, let $\mathscr P_\ell(D)$ be the $2^\ell$ descendants of $D$ at depth $\ell$ below $D$.  They partition $D$, every $A\in\mathscr P_\ell(D)$ has cardinality $s/2^\ell$, and $h_D$ is constant on $A$.  If that constant is $c_A$, then
\begin{align}
\bone_A^\top h_D
=\sum_{i\in A}(h_D)_i
=|A|c_A
=\frac{s}{2^\ell}c_A.
\notag
\end{align}
Therefore
\begin{align*}
\sum_{A\in\mathscr P_\ell(D)}
\bone_A\bone_A^\top h_D
&=
\sum_{A\in\mathscr P_\ell(D)}
\bigl(\frac{s}{2^\ell}c_A\bigr)\bone_A\\
&=\frac{s}{2^\ell}
\sum_{A\in\mathscr P_\ell(D)}c_A\bone_A\\
&=\frac{s}{2^\ell}h_D.
\end{align*}
The first equality uses the preceding inner-product calculation, and the last equality uses that the sets in $\mathscr P_\ell(D)$ partition the support of $h_D$ and $h_D=c_A$ on each such set.  Summing these contributions over all proper descendant levels $\ell=1,\ldots,\log_2 s$ gives
\begin{align}
C_{\mathscr T}h_D
=
\bigl(\frac s2+\frac s4+\cdots+1\bigr)h_D
=(s-1)h_D.
\notag
\end{align}
Thus every Haar wavelet is an eigenvector with the claimed eigenvalue.  The constant vector together with the $n-1$ Haar wavelets forms an orthonormal basis, so these are all eigenvalues.

Finally, we order the eigenvalues.
For $j=0,\ldots,H-1$, there are $2^j$ internal nodes of size $n/2^j$, hence $2^j$ Haar eigenvalues equal to $n/2^j-1$.  Together with the constant eigenvalue, the level-$j$ wavelets occupy indices from $2^j+1$ through $2^{j+1}$ in the nonincreasing ordering.  If $m+1$ lies in this range, then
\begin{align}
\lambda_{m+1}(C_{\mathscr T})
=\frac{n}{2^j}-1
\le\frac{n}{2^j}.
\notag
\end{align}
Because $m+1\le2^{j+1}$,
\begin{align}
\frac{n}{2^j}
\le\frac{2n}{m+1}.
\notag
\end{align}
For $m=0$, the bound reads $2n-1\le2n$.  This proves \eqref{eq:haar-spectrum-bound} in every case.
\end{proof}

\begin{proof}[Proof of \Cref{thm:haar-width}]
Let $\mathscr D=\mathscr T$ be the set of all tree nodes.  By assumption every realizable signed threshold set is a disjoint union of at most $q$ nodes, hence
\begin{align}
s(\mathscr T)\le q.
\notag
\end{align}
Since all input norms equal $R$,
\begin{align}
p_i=d_{L,i}=p:=R^{\,2^{-(L-2)}}
\qquad\text{for every }i.
\notag
\end{align}
For $0\le t\le p$, the active set $I_t$ is all of $[n]$ up to the endpoint $t=p$, which does not affect the integral.  Thus $P_t=I_n$ almost everywhere on $[0,p]$ and $P_t=0$ for $t>p$.  From \eqref{eq:approx-HD}--\eqref{eq:approx-ND},
\begin{align}
N_{\mathscr T}
&=2\int_0^p H_{\mathscr T}\,dt
\nonumber\\
&=2p\,s(\mathscr T)
\sum_{D\in\mathscr T}\bone_D\bone_D^\top
\nonumber\\
&\preceq2pq\,C_{\mathscr T}.
\label{eq:haar-N-bound}
\end{align}
The last relation is scalar comparison, since $s(\mathscr T)\le q$.

By \Cref{cor:recursive-spectral-width,lem:haar-spectrum},
\begin{align}
d_{m,L}(\bx;r)
&\le\frac{r}{\sqrt n}
\sqrt{\lambda_{m+1}(N_{\mathscr T})}
\nonumber\\
&\le\frac{r}{\sqrt n}
\sqrt{2pq\,\lambda_{m+1}(C_{\mathscr T})}
\nonumber\\
&\le\frac{r}{\sqrt n}
\sqrt{2pq\,\frac{2n}{m+1}}
\nonumber\\
&=2r\sqrt p\sqrt{\frac{q}{m+1}}.
\end{align}
Since
\begin{align}
\sqrt p=R^{\,2^{-(L-1)}},
\notag
\end{align}
this is exactly \eqref{eq:haar-width}.  The eigenspaces of $C_{\mathscr T}$ are the constant mode and Haar wavelets ordered from coarse to fine by \Cref{lem:haar-spectrum}; therefore their leading $m$ modes give the stated sample-dependent but target-independent trace subspace.
\end{proof}

\begin{proof}[Proof of \Cref{prop:stable-width-lower}]
We first establish directly the rank-one implication used below.  Let $K\succeq0$ and suppose $yy^\top\preceq K$.  For every $\varepsilon>0$,
\begin{align}
yy^\top\preceq K+\varepsilon I.
\notag
\end{align}
Congruence by $(K+\varepsilon I)^{-1/2}$ gives
\begin{align}
(K+\varepsilon I)^{-1/2}yy^\top(K+\varepsilon I)^{-1/2}\preceq I.
\notag
\end{align}
The left side is rank one, with only nonzero eigenvalue $y^\top(K+\varepsilon I)^{-1}y$.  Hence
\begin{align}
y^\top(K+\varepsilon I)^{-1}y\le1
\qquad\text{for every }\varepsilon>0.
\notag
\end{align}
If $y$ had a nonzero component in $\ker K$, this quantity would diverge as $\varepsilon\downarrow0$.  Thus $y\in\ran K$, and passage to the limit on $\ran K$ gives
\begin{align}
y^\top K^\dagger y\le1.
\notag
\end{align}
By \eqref{eq:ellipsoid}, $y\in E(K)$.

For each $j$, choose the trace $a^{(j)}$ and the stable sign specified in the hypothesis.  We write the positive-threshold case; the negative-threshold case is identical.

We first show that each stable block creates one vector in the full depth-$L$ trace body.
The Brownian threshold identity gives
\begin{align}
\mathbf B(a^{(j)})
&\succeq
\int_{t_j}^{t_j+\delta}
 v_u^+(a^{(j)})v_u^+(a^{(j)})^\top\,du
\nonumber\\
&=
\int_{t_j}^{t_j+\delta}
\bone_{A_j}\bone_{A_j}^\top\,du
\nonumber\\
&=
\delta\bone_{A_j}\bone_{A_j}^\top.
\label{eq:stable-block-rankone}
\end{align}
Set
\begin{align}
y_j:=\sqrt\delta\,\bone_{A_j}.
\notag
\end{align}
Then $y_jy_j^\top\preceq\mathbf B(a^{(j)})$, so the rank-one implication just proved gives
\begin{align}
y_j\in E(\mathbf B(a^{(j)})).
\notag
\end{align}
Because $a^{(j)}\in\sS_{L-1}(\bx)$ may be a limit of actual unit traces, \Cref{thm:dirac-reduction} places $\mathbf B(a^{(j)})$ in the compact covariance closure $\sK_L(\bx)$; it need not be the Gram matrix of one realized ladder.  By \Cref{thm:trace-body}, however, $E(\mathbf B(a^{(j)}))\subseteq\cA_L(\bx)$, so $y_j\in\cA_L(\bx)$.  The distance to a fixed subspace is continuous, and therefore its supremum over the actual radius-$r$ traces equals its supremum over their closure $r\cA_L(\bx)$.  Thus the vectors $ry_j$ may be used in the width lower bound without an attainment assumption.

We then show that no $m$-dimensional space can approximate all disjoint block vectors too well.
Because the sets $A_1,\ldots,A_M$ are pairwise disjoint and all have cardinality $s$, the normalized vectors
\begin{align}
u_j:=\frac{1}{\sqrt s}\bone_{A_j},
\qquad j=1,\ldots,M,
\notag
\end{align}
form an orthonormal family.  Let $V\subseteq\R^n$ satisfy $\dim V\le m$.  Then
\begin{align}
\sum_{j=1}^M\norm{P_Vu_j}_2^2
&=\sum_{j=1}^M\ip{u_j}{P_Vu_j}
\nonumber\\
&\le\tr P_V
\le m.
\label{eq:stable-proj-budget}
\end{align}
To justify the first inequality, extend $u_1,\ldots,u_M$ to an orthonormal basis and write the trace of the positive semidefinite projection $P_V$ in that basis; the omitted terms are nonnegative.

Since $\norm{u_j}_2=1$,
\begin{align}
\norm{P_{V^\perp}u_j}_2^2
=1-\norm{P_Vu_j}_2^2.
\notag
\end{align}
Summing and using \eqref{eq:stable-proj-budget},
\begin{align}
\sum_{j=1}^M\norm{P_{V^\perp}u_j}_2^2
\ge M-m.
\notag
\end{align}
Therefore at least one index $j$ satisfies
\begin{align}
\norm{P_{V^\perp}u_j}_2^2
\ge1-\frac{m}{M}.
\notag
\end{align}
For that $j$,
\begin{align}
\operatorname{dist}_n(ry_j,V)^2
&=\frac{r^2}{n}\norm{P_{V^\perp}y_j}_2^2
\nonumber\\
&=\frac{r^2\delta s}{n}
\norm{P_{V^\perp}u_j}_2^2
\nonumber\\
&\ge
\frac{r^2\delta s}{n}
\bigl(1-\frac{m}{M}\bigr).
\end{align}
Since $V$ was arbitrary, taking the infimum over $\dim V\le m$ proves \eqref{eq:stable-width-lower}.

If the $A_j$ form an equal partition and $M=2m$, then $s=n/(2m)$ and $1-m/M=1/2$.  Substitution gives
\begin{align}
d_{m,L}(\bx;r)^2
\ge\frac{r^2\delta}{4m},
\notag
\end{align}
which is \eqref{eq:stable-width-partition}.
\end{proof}

\section{Proof of the Learning Consequence}\label[appendix]{app:learning}

\begin{proof}[Proof of \Cref{cor:generalization}]
Fix the realized sample $S=((X_i,Y_i))_{i=1}^n$.  Apply the one-sided empirical Rademacher uniform-deviation theorem for functions with range length $M_\ell$~\cite{bartlett2002rademacher} first to the loss class $\ell\circ B_L(r)$ with failure probability $\delta/2$, and then to its negative with the same failure probability.  The Rademacher complexities of a class and its negative are equal, because replacing every Rademacher sign by its negative preserves the joint law.  A union bound therefore gives, with probability at least $1-\delta$,
\begin{equation}\label{eq:app-uniform-rad}
\sup_{f\in B_L(r)}|R(f)-R_n(f)|
\le2\Rhat_S(\ell\circ B_L(r))
+3M_\ell\sqrt{\frac{\log(4/\delta)}{2n}}.
\end{equation}
Here
\begin{align}
\Rhat_S(\ell\circ B_L(r))
:=\E_\sigma\sup_{f\in B_L(r)}
\frac1n\sum_{i=1}^n\sigma_i\ell(f(X_i),Y_i).
\notag
\end{align}
For each $i$, define
$\varphi_i(t):=\ell(t,Y_i)-\ell(0,Y_i)$.  Then $\varphi_i(0)=0$ and $\varphi_i$ is $L_\ell$-Lipschitz.  The term
$n^{-1}\sum_i\sigma_i\ell(0,Y_i)$ does not depend on $f$ and has zero expectation with respect to $\sigma$.  The contraction inequality therefore yields
\begin{equation}\label{eq:app-contraction}
\Rhat_S(\ell\circ B_L(r))
\le L_\ell\Rhat_{X_{1:n}}(B_L(r)).
\end{equation}

We next compare Rademacher and Gaussian complexities explicitly.  Let $A$ be the sample trace set of $B_L(r)$ and let $h_A$ be its support function.  Write a standard Gaussian vector as
$G_i=\sigma_i R_i$, where the signs $\sigma_i$ are independent Rademacher variables, the magnitudes $R_i=|G_i|$ are independent of the signs, and
$\E R_i=\sqrt{2/\pi}$.  For fixed $\sigma$, convexity of the support function and Jensen's inequality in $R$ give
\begin{align}
\E_R h_A(\sigma\odot R)
\ge h_A(\sigma\odot\E R)
=\sqrt{\frac{2}{\pi}}\,h_A(\sigma).
\notag
\end{align}
Averaging over $\sigma$ and dividing by $n$ gives
\begin{equation}\label{eq:app-rad-gauss}
\Rhat_{X_{1:n}}(B_L(r))
\le\sqrt{\frac{\pi}{2}}\,
\Ghat_{X_{1:n}}(B_L(r)).
\end{equation}
Combine \eqref{eq:app-uniform-rad}--\eqref{eq:app-rad-gauss} with \Cref{thm:profile-complexity}.  Absorbing the numerical factor $2\sqrt{\pi/2}$ into a universal constant $C$ proves \eqref{eq:generalization}.  Equation \eqref{eq:generalization-lambda} follows from \eqref{eq:profile-bound-radius}.
\end{proof}

\section{Proofs for Covariance Certification}\label[appendix]{app:optimization}

\subsection{Sparse Contact Certificates}

\begin{proof}[Proof of \Cref{thm:sparse-contact}]
Fix an optimizer $N_*$ of the full problem.  We construct a finite isotropic certificate.

We first approximate the $2$-summing optimum by a finite sequence.
By \Cref{thm:pi2}, for every integer $k\ge1$ there are vectors
$z_{k,1},\ldots,z_{k,r_k}\in\R^n$ satisfying
\begin{equation}\label{eq:contact-approx-moment}
\sum_{r=1}^{r_k}z_{k,r}z_{k,r}^\top\preceq I_n
\end{equation}
and
\begin{equation}\label{eq:contact-approx-value}
\sum_{r=1}^{r_k}h_{L,\bx}(z_{k,r})^2
\ge\tau_L(\bx)-\frac1k.
\end{equation}
The compactness of $\sS_{L-1}(\bx)$ and continuity of
$a\mapsto z^\top\mathbf B(a)z$ imply that the supremum in \eqref{eq:dirac-support} is attained.  Hence, for every $r$, choose
$a_{k,r}\in\sS_{L-1}(\bx)$ such that
\begin{equation}\label{eq:contact-maximizer}
h_{L,\bx}(z_{k,r})^2
=z_{k,r}^\top\mathbf B(a_{k,r})z_{k,r}.
\end{equation}

We next fill the isotropic slack.
Set
\begin{align}
R_k:=I_n-\sum_{r=1}^{r_k}z_{k,r}z_{k,r}^\top\succeq0.
\notag
\end{align}
By the spectral theorem, there are vectors $y_{k,1},\ldots,y_{k,s_k}$ such that
\begin{align}
R_k=\sum_{s=1}^{s_k}y_{k,s}y_{k,s}^\top.
\notag
\end{align}
For every nonzero $y_{k,s}$, choose
$b_{k,s}\in\sS_{L-1}(\bx)$ attaining
$h_{L,\bx}(y_{k,s})^2$.  Append these vectors and traces to the finite family selected above.  Denote the combined vectors by $w_{k,j}$ and the corresponding traces by $c_{k,j}$.  Then
\begin{equation}\label{eq:contact-exact-moment-k}
\sum_jw_{k,j}w_{k,j}^\top=I_n.
\end{equation}
The combined objective
\begin{align}
J_k:=\sum_jw_{k,j}^\top\mathbf B(c_{k,j})w_{k,j}
\notag
\end{align}
is at least the left side of \eqref{eq:contact-approx-value}.  On the other hand, the combined sequence is admissible in \eqref{eq:pi2-def}, so \Cref{thm:pi2} gives $J_k\le\tau_L(\bx)$.  Consequently,
\begin{equation}\label{eq:contact-J-limit}
\tau_L(\bx)-\frac1k\le J_k\le\tau_L(\bx).
\end{equation}

We then rewrite the certificate in a fixed finite-dimensional cone.
Discard zero vectors.  Write
\begin{align}
w_{k,j}=\sqrt{\lambda_{k,j}}u_{k,j},
\qquad
\lambda_{k,j}:=\norm{w_{k,j}}_2^2>0,
\qquad
u_{k,j}:=\frac{w_{k,j}}{\norm{w_{k,j}}_2}\in S^{n-1}.
\notag
\end{align}
Let $\operatorname{vec}_{\mathrm{sym}}$ be any fixed linear identification of $\mathbb S^n$ with $\R^{d_n}$ and define
\begin{align}
\mathscr C(u,a)
:=\bigl(
\operatorname{vec}_{\mathrm{sym}}(uu^\top),
\ u^\top\mathbf B(a)u
\bigr)
\in\R^{d_n+1}.
\notag
\end{align}
Equations \eqref{eq:contact-exact-moment-k} and the definition of $J_k$ give
\begin{equation}\label{eq:contact-cone-representation}
\bigl(
\operatorname{vec}_{\mathrm{sym}}(I_n),J_k
\bigr)
=\sum_j\lambda_{k,j}\mathscr C(u_{k,j},c_{k,j}).
\end{equation}

We next reduce to at most $d_n+1$ terms.  We use the following finite-dimensional conic argument directly.  Suppose
\begin{align}
x=\sum_{j=1}^m\lambda_ja_j,
\qquad
\lambda_j>0,
\qquad
a_j\in A\subset\R^p,
\notag
\end{align}
with $m>p$.  The vectors $a_1,\ldots,a_m$ are linearly dependent, so choose a nonzero $c=(c_1,\ldots,c_m)$ with $\sum_jc_ja_j=0$.  Reverse its sign if necessary so that some $c_j>0$, and put
\begin{align}
t_*:=\min_{j:c_j>0}\frac{\lambda_j}{c_j}.
\notag
\end{align}
The coefficients $\lambda_j':=\lambda_j-t_*c_j$ are nonnegative, at least one is zero, and $\sum_j\lambda_j'a_j=x$.  Deleting zero coefficients and repeating gives a conic representation with at most $p$ terms.

Apply this argument in dimension $p=d_n+1$ to \eqref{eq:contact-cone-representation}.  For every $k$, we obtain a representation with at most $d_n+1$ nonzero coefficients.  Pad shorter representations with zero coefficients so that every representation has exactly $d_n+1$ indexed positions; at a zero coefficient, choose the accompanying sphere point and trace arbitrarily from the nonempty compact index sets.  Taking traces in the matrix part of \eqref{eq:contact-cone-representation} gives
\begin{align}
\sum_{j=1}^{d_n+1}\lambda_{k,j}=\tr I_n=n.
\notag
\end{align}
Thus every coefficient belongs to $[0,n]$.

We then pass to a limiting finite certificate.
The sphere $S^{n-1}$ is compact.  The trace set $\sS_{L-1}(\bx)$ is closed by definition and bounded coordinatewise by the preceding-layer evaluation scale, hence it is compact in $\R^n$.  The coefficient vectors belong to the compact simplex
$\set{\lambda\ge0:\sum_j\lambda_j=n}$.  Hence, after passing to a subsequence, for every $j$,
\begin{align}
\lambda_{k,j}\longrightarrow\lambda_j,
\qquad
u_{k,j}\longrightarrow u_j,
\qquad
c_{k,j}\longrightarrow a_j.
\notag
\end{align}
Continuity of $\mathscr C$ and \eqref{eq:contact-J-limit} imply
\begin{align}
I_n&=\sum_{j=1}^{d_n+1}\lambda_ju_ju_j^\top,\label{eq:contact-limit-moment}\\
\tau_L(\bx)&=\sum_{j=1}^{d_n+1}\lambda_j u_j^\top\mathbf B(a_j)u_j.\label{eq:contact-limit-value}
\end{align}
Delete the zero coefficients and define $z_j:=\sqrt{\lambda_j}u_j$.  This proves \eqref{eq:contact-isotropy} and \eqref{eq:contact-value}, with $m\le d_n+1$.

We next prove the contact equations.
For every $j$, full feasibility of $N_*$ and \Cref{thm:dirac-reduction} give
$N_*-\mathbf B(a_j)\succeq0$.  Therefore
\begin{align}
z_j^\top\bigl(N_*-\mathbf B(a_j)\bigr)z_j\ge0.
\notag
\end{align}
Summing these nonnegative terms and using \eqref{eq:contact-isotropy}, \eqref{eq:contact-value}, and $\tr N_*=\tau_L(\bx)$ yields
\begin{align*}
\sum_{j=1}^m z_j^\top\bigl(N_*-\mathbf B(a_j)\bigr)z_j
&=\tr\!\bigl(N_*\sum_{j=1}^m z_jz_j^\top\bigr)
-\sum_{j=1}^m z_j^\top\mathbf B(a_j)z_j\\
&=\tr N_*-\tau_L(\bx)\\
&=0.
\end{align*}
Every summand is nonnegative, so each one is zero.  If $A\succeq0$ and $z^\top Az=0$, then
$\norm{A^{1/2}z}_2^2=0$, hence $Az=0$.  Apply this with
$A=N_*-\mathbf B(a_j)$ to obtain \eqref{eq:contact-equations}.

Finally, we prove exactness of the finite family.
Set $Z_j:=z_jz_j^\top$.  Equation \eqref{eq:contact-isotropy} makes $(Z_j)_{j=1}^m$ feasible for the finite dual \eqref{eq:finite-dual} associated with $\mathcal F_*$.  Its objective value is $\tau_L(\bx)$ by \eqref{eq:contact-value}.  Therefore finite strong duality gives
\begin{align}
\tau(\mathcal F_*)\ge\tau_L(\bx).
\notag
\end{align}
The reverse inequality holds because $\mathcal F_*\subset\sK_L(\bx)$ and the finite primal has fewer constraints than the full primal.  Hence
$\tau(\mathcal F_*)=\tau_L(\bx)$, proving \eqref{eq:finite-exact-profile}.
\end{proof}

\begin{proof}[Proof of \Cref{cor:finite-multiplier-attainment}]
Let $m\le n(n+1)/2+1$, $a_1,\ldots,a_m$, and $z_1,\ldots,z_m$ be the contact certificate from \Cref{thm:sparse-contact}.  The isotropy relation gives
\begin{align}
\sum_{j=1}^m z_jz_j^\top=I_n.
\notag
\end{align}
For every $j$, feasibility of $N_*$ and the support identity imply
\begin{align}
z_j^\top\mathbf B(a_j)z_j
\le
h_{L,\bx}(z_j)^2
\le
z_j^\top N_*z_j.
\notag
\end{align}
The contact equation $\bigl(N_*-\mathbf B(a_j)\bigr)z_j=0$ makes both inequalities equalities.  Summing and using \eqref{eq:contact-value} gives
\begin{align}
\sum_{j=1}^m h_{L,\bx}(z_j)^2=\tau_L(\bx).
\notag
\end{align}
Define
\begin{align}
\nu_*:=\frac1{2m}\sum_{j=1}^m
\bigl(\delta_{\sqrt m z_j}+\delta_{-\sqrt m z_j}\bigr).
\notag
\end{align}
This law is symmetric and centered.  Its covariance is
\begin{align}
\int zz^\top\,d\nu_*(z)
=\sum_{j=1}^m z_jz_j^\top
=I_n.
\notag
\end{align}
Absolute homogeneity of the support function yields
\begin{align}
\int h_{L,\bx}(z)^2\,d\nu_*(z)
=\sum_{j=1}^m h_{L,\bx}(z_j)^2
=\tau_L(\bx).
\notag
\end{align}
Thus $\nu_*$ attains the supremum in \eqref{eq:robust-minimax} and has at most $2m\le2\bigl(n(n+1)/2+1\bigr)$ support points.
\end{proof}

\subsection{Certified Active-Set Brackets}

\begin{proof}[Proof of \Cref{prop:certified-computation}]
We first prove \eqref{eq:floor-bias}.  Every matrix feasible for the floored problem is feasible for the original problem, so
$\tau_L(\bx)\le\tau_{L,\varepsilon}(\bx)$.  Conversely, if $N_*$ solves the original problem, then
$N_*+\varepsilon I_n\succeq\varepsilon I_n$ and
$N_*+\varepsilon I_n\succeq K$ for every $K\in\sK_L(\bx)$.  Hence
\begin{align}
\tau_{L,\varepsilon}(\bx)
\le\tr(N_*+\varepsilon I_n)
=\tau_L(\bx)+n\varepsilon.
\notag
\end{align}

Now let $\mathcal F$ and $N_{\mathcal F}$ be as in the proposition.  Since the finite problem imposes only a subset of the full constraints,
\begin{align}
\tr N_{\mathcal F}=\tau(\mathcal F)\le\tau_L(\bx).
\notag
\end{align}
Put $\widetilde N:=N_{\mathcal F}+\varepsilon I_n\succ0$ and
$\alpha:=\operatorname{sep}_L(\widetilde N;\bx)$.  For every $a\in\sS_{L-1}(\bx)$,
\begin{align}
\lambda_{\max}(\widetilde N^{-1/2}\mathbf B(a)\widetilde N^{-1/2})\le\alpha
\notag
\end{align}
is equivalent to
\begin{align}
\widetilde N^{-1/2}\mathbf B(a)\widetilde N^{-1/2}\preceq\alpha I_n.
\notag
\end{align}
Congruence by $\widetilde N^{1/2}$ gives
$\mathbf B(a)\preceq\alpha\widetilde N$.  By \Cref{thm:dirac-reduction},
$\alpha\widetilde N$ dominates every matrix in $\sK_L(\bx)$ and is full-family feasible.  Therefore
\begin{align}
\tau_L(\bx)
\le\alpha\tr\widetilde N
=\operatorname{sep}_L(N_{\mathcal F}+\varepsilon I_n;\bx)
(\tr N_{\mathcal F}+n\varepsilon),
\notag
\end{align}
which is \eqref{eq:certified-bracket}.  Equation \eqref{eq:eta-bracket} follows by substituting the verified bound $\alpha\le1+\eta$.  Finally, if
$\mathcal F_t\subset\mathcal F_{t+1}$, then the feasible set of the $(t+1)$-st finite primal is contained in that of the $t$-th, so
$\tau(\mathcal F_t)\le\tau(\mathcal F_{t+1})$.
\end{proof}

\subsection{Convex Resistance Design}

\begin{proof}[Proof of \Cref{thm:resistance-design}]
We first characterize the finite-objective domain.
For every $w\in\mathcal W_{\mathcal E}$, $L(w)\succeq0$ and $L(w)\bone=0$.  The matrix
$L(w)+J_n$ is positive definite if and only if the positive-weight support graph is connected.  Indeed, on $\operatorname{span}\{\bone\}$, $J_n$ acts as the identity and $L(w)$ vanishes.  On $\bone^\perp$, $J_n$ vanishes, and $L(w)$ is positive definite exactly when its nullspace is $\operatorname{span}\{\bone\}$, which is equivalent to connectedness.

If the support is connected, the orthogonal decomposition
$\R^n=\operatorname{span}\{\bone\}\oplus\bone^\perp$ gives
\begin{equation}\label{eq:inverse-lap-j}
\bigl(L(w)+J_n\bigr)^{-1}=L(w)^\dagger+J_n.
\end{equation}
Taking traces and using $\tr J_n=1$ proves
$\Phi_{\mathcal E}(w)=\tr L(w)^\dagger$.

We next prove convexity and existence.
On the positive-definite cone, let $A(t)=A+tH$ with $A\succ0$ and $H=H^\top$.  Differentiating the identity
$A(t)A(t)^{-1}=I$ gives
\begin{align}
\frac{d}{dt}A(t)^{-1}=-A(t)^{-1}HA(t)^{-1}.
\notag
\end{align}
Differentiating once more and taking traces yields
\begin{align}
\frac{d^2}{dt^2}\tr A(t)^{-1}
=2\tr\bigl(A(t)^{-1}HA(t)^{-1}HA(t)^{-1}\bigr).
\notag
\end{align}
Set $C=A(t)^{-1/2}HA(t)^{-1/2}$.  Direct multiplication gives
\begin{align}
A(t)^{-1}HA(t)^{-1}HA(t)^{-1}
=A(t)^{-1/2}C^2A(t)^{-1/2}\succeq0.
\notag
\end{align}
Therefore the second derivative is nonnegative.  Since $w\mapsto L(w)+J_n$ is affine, $\Phi_{\mathcal E}$ is convex on its finite domain; assigning $+\infty$ to the singular boundary gives an extended convex function on the simplex.

The simplex $\mathcal W_{\mathcal E}$ is compact and contains a point with all edge weights positive, hence a point with finite objective.  Let $(w_k)$ be a minimizing sequence and pass to a convergent subsequence $w_k\to w_*$.  If $L(w_*)+J_n$ were singular, its smallest eigenvalue would be zero.  Continuity of eigenvalues would give
$\lambda_{\min}(L(w_k)+J_n)\to0$, and hence
\begin{align}
\Phi_{\mathcal E}(w_k)+1
=\tr(L(w_k)+J_n)^{-1}
\ge\frac{1}{\lambda_{\min}(L(w_k)+J_n)}\longrightarrow\infty,
\notag
\end{align}
contradicting the existence of a finite feasible value.  Thus $w_*$ has connected support and continuity on the positive-definite cone gives attainment.

We then derive the full-family covariance certificate.
For a feasible connected $w$, apply \Cref{thm:graph-majorant} to the weighted graph with Laplacian $L(w)$.  Equations \eqref{eq:M-bound} and \eqref{eq:V-bound} give
\begin{align}
M_{L-1,1}(\bx)\le S_L(\bx),
\qquad
V_{L-1,G}(\bx)
\le\sum_{e\in\mathcal E}c_ew_e
=1.
\notag
\end{align}
Therefore \eqref{eq:graph-tau} gives
\begin{align}
\tau_L(\bx)
\le2S_L(\bx)+2\tr L(w)^\dagger
=2S_L(\bx)+2\Phi_{\mathcal E}(w).
\notag
\end{align}
Minimize over $w$ to obtain \eqref{eq:resistance-tau}.  Divide by $S_L(\bx)>0$ to obtain \eqref{eq:resistance-lambda}.

We next compute the gradient.
Fix a connected feasible $w$ and set $A=L(w)+J_n$.  Since
$\partial A/\partial w_e=b_eb_e^\top$, the first-derivative calculation above gives
\begin{align*}
\frac{\partial\Phi_{\mathcal E}}{\partial w_e}(w)
&=-\tr\bigl(A^{-1}b_eb_e^\top A^{-1}\bigr)\\
&=-b_e^\top A^{-2}b_e.
\end{align*}
Because $b_e\perp\bone$, equation \eqref{eq:inverse-lap-j} gives
$A^{-1}b_e=L(w)^\dagger b_e$.  Thus
\begin{align}
b_e^\top A^{-2}b_e
=\norm{A^{-1}b_e}_2^2
=b_e^\top(L(w)^\dagger)^2b_e
=q_e(w),
\notag
\end{align}
which proves \eqref{eq:resistance-gradient}.

We then prove the Euler identity.
Using $L(w)=\sum_e w_eb_eb_e^\top$ and cyclicity of the trace,
\begin{align*}
\sum_e w_eq_e(w)
&=\sum_e w_e\tr\bigl((L(w)^\dagger)^2b_eb_e^\top\bigr)\\
&=\tr\bigl((L(w)^\dagger)^2L(w)\bigr).
\end{align*}
In an eigenbasis of $L(w)$, a positive eigenvalue $\lambda$ contributes
$\lambda^{-2}\lambda=\lambda^{-1}$, while the zero eigenvalue contributes zero.  Hence
\begin{align}
\tr\bigl((L(w)^\dagger)^2L(w)\bigr)=\tr L(w)^\dagger=\Phi_{\mathcal E}(w),
\notag
\end{align}
which proves \eqref{eq:resistance-euler}.

We next derive the optimality conditions.
The Lagrangian for the finite-objective convex program is
\begin{align}
\mathcal L(w,\lambda,\mu)
=\Phi_{\mathcal E}(w)
+\lambda\bigl(\sum_ec_ew_e-1\bigr)
-\sum_e\mu_ew_e,
\qquad \mu_e\ge0.
\notag
\end{align}
Stationarity and \eqref{eq:resistance-gradient} give
\begin{align}
-q_e(w)+\lambda c_e-\mu_e=0.
\notag
\end{align}
If $w_e>0$, complementary slackness gives $\mu_e=0$, so
$q_e(w)/c_e=\lambda$.  If $w_e=0$, then $\mu_e\ge0$ gives
$q_e(w)/c_e\le\lambda$.  Multiply the equality on active edges by $c_ew_e$ and sum.  Since inactive edges have zero weight and $\sum_ec_ew_e=1$,
\begin{align}
\lambda
=\sum_e w_eq_e(w)
=\Phi_{\mathcal E}(w)
\notag
\end{align}
by \eqref{eq:resistance-euler}.  This proves the necessity of \eqref{eq:resistance-kkt}.  Conversely, if \eqref{eq:resistance-kkt} holds, set
$\lambda=\Phi_{\mathcal E}(w)$ and
$\mu_e=\lambda c_e-q_e(w)\ge0$.  Stationarity, primal feasibility, dual feasibility, and complementary slackness all hold.  Convexity therefore makes $w$ optimal.

Finally, we prove the computable gap.
Let $w_*$ be an optimizer.  Convexity gives
\begin{align}
\Phi_{\mathcal E}(w_*)
\ge\Phi_{\mathcal E}(w)
+\ip{\nabla\Phi_{\mathcal E}(w)}{w_*-w}.
\notag
\end{align}
Rearranging,
\begin{align}
\Phi_{\mathcal E}(w)-\Psi_{L,\mathcal E}(\bx)
\le
\ip{\nabla\Phi_{\mathcal E}(w)}{w-w_*}
\le
\max_{u\in\mathcal W_{\mathcal E}}
\ip{\nabla\Phi_{\mathcal E}(w)}{w-u}.
\notag
\end{align}
By \eqref{eq:resistance-gradient} and \eqref{eq:resistance-euler},
\begin{align}
\ip{\nabla\Phi_{\mathcal E}(w)}{w}
=-\sum_e w_eq_e(w)
=-\Phi_{\mathcal E}(w).
\notag
\end{align}
The minimum of the linear functional
$u\mapsto\ip{\nabla\Phi_{\mathcal E}(w)}{u}$ over the cost simplex is attained at an edge vertex
$u=c_e^{-1}\mathbf1_{\{e\}}$, and its minimum is
$-\max_e q_e(w)/c_e$.  Therefore the last maximum equals
$\max_e q_e(w)/c_e-\Phi_{\mathcal E}(w)=g_{\mathcal E}(w)$.  This proves \eqref{eq:resistance-gap-bound}; nonnegativity follows from the left side or directly from the weighted average identity \eqref{eq:resistance-euler}.
\end{proof}

\begin{proof}[Proof of \Cref{cor:tree-design}]
For vertices $i,j$, define the effective resistance
\begin{align}
R_{ij}:=(e_i-e_j)^\top L_T(w)^\dagger(e_i-e_j).
\notag
\end{align}
First we prove
\begin{equation}\label{eq:kirchhoff-trace}
\sum_{1\le i<j\le n}R_{ij}=n\tr L_T(w)^\dagger.
\end{equation}
Indeed,
\begin{align}
\sum_{i<j}(e_i-e_j)(e_i-e_j)^\top=nI_n-\bone\bone^\top.
\notag
\end{align}
Therefore, using $L_T(w)^\dagger\bone=0$,
\begin{align*}
\sum_{i<j}R_{ij}
&=\tr\!\bigl(L_T(w)^\dagger
\sum_{i<j}(e_i-e_j)(e_i-e_j)^\top\bigr)\\
&=\tr\bigl(L_T(w)^\dagger(nI_n-\bone\bone^\top)\bigr)\\
&=n\tr L_T(w)^\dagger.
\end{align*}

We now verify the tree resistance formula directly.  Choose an arbitrary orientation of the tree and let $B_T\in\R^{n\times(n-1)}$ be its signed incidence matrix, so
$L_T(w)=B_TWB_T^\top$ with $W=\diag((w_e)_{e\in T})$.  For a fixed pair $i,j$, let $\chi_{ij}\in\R^{n-1}$ be the signed indicator of the unique oriented path from $i$ to $j$.  Then
$B_T\chi_{ij}=e_i-e_j$.  The columns of a tree incidence matrix are linearly independent, so this is the unique edge vector $y$ satisfying $B_Ty=e_i-e_j$.

Let $x=L_T(w)^\dagger(e_i-e_j)$.  Since $e_i-e_j\perp\bone$, one has
$L_T(w)x=e_i-e_j$.  Set $y:=WB_T^\top x$.  Then
$B_Ty=B_TWB_T^\top x=e_i-e_j$, so uniqueness gives $y=\chi_{ij}$.  Therefore
\begin{align*}
R_{ij}
&=(e_i-e_j)^\top x\\
&=x^\top L_T(w)x\\
&=(B_T^\top x)^\top W(B_T^\top x)\\
&=y^\top W^{-1}y\\
&=\sum_{e\text{ on the }i\text{--}j\text{ path}}\frac1{w_e}.
\end{align*}
An edge $e$ lies on exactly $s_e(n-s_e)$ unordered vertex-pair paths.  Hence
\begin{align}
\sum_{i<j}R_{ij}
=\sum_{e\in T}\frac{s_e(n-s_e)}{w_e}.
\notag
\end{align}
Combine this identity with \eqref{eq:kirchhoff-trace} to obtain
\begin{equation}\label{eq:tree-trace-resistance}
\tr L_T(w)^\dagger
=\sum_{e\in T}\frac{a_e}{w_e}.
\end{equation}

Under $\sum_ec_ew_e=1$, Cauchy--Schwarz gives
\begin{align*}
\bigl(\sum_{e\in T}\sqrt{a_ec_e}\bigr)^2
&=\bigl(\sum_{e\in T}\sqrt{\frac{a_e}{w_e}}\sqrt{c_ew_e}\bigr)^2\\
&\le\bigl(\sum_{e\in T}\frac{a_e}{w_e}\bigr)
\bigl(\sum_{e\in T}c_ew_e\bigr)\\
&=\tr L_T(w)^\dagger.
\end{align*}
Equality holds exactly when the two Cauchy--Schwarz vectors are proportional, equivalently
$\sqrt{a_e/w_e}=C\sqrt{c_ew_e}$ for every edge.  Thus
$w_e=C^{-1}\sqrt{a_e/c_e}$.  The cost normalization determines
$C=\sum_f\sqrt{a_fc_f}$ and proves \eqref{eq:tree-optimal-weights}.  Substitution into \eqref{eq:tree-trace-resistance} proves \eqref{eq:tree-optimal-value}.  Finally, apply \eqref{eq:resistance-tau} to the fixed-tree optimum to obtain \eqref{eq:tree-profile-certificate}.
\end{proof}

\subsection{Brownian Separation Computations}

\begin{proof}[Proof of \Cref{prop:brownian-oracle}]
For fixed $a$, the Rayleigh--Ritz formula gives
\begin{align}
\lambda_{\max}(N^{-1/2}\mathbf B(a)N^{-1/2})
=\sup_{\norm{v}_2=1}v^\top N^{-1/2}\mathbf B(a)N^{-1/2}v.
\notag
\end{align}
Setting $q=N^{-1/2}v$ and then taking the supremum over $a$ proves \eqref{eq:oracle-minmax}.

We next prove the sorting formula.  If $a_i$ and $a_j$ have opposite signs, then
$\mathbf B(a)_{ij}=0$.  If they have the same positive sign, then
$\mathbf B(a)_{ij}=\min\{|a_i|,|a_j|\}$; the same identity holds within the negative-sign block.  Therefore it is enough to consider one sign block.  Let its magnitudes be ordered as
$0<t_1\le\cdots\le t_m$, with corresponding coefficients $q_1,\ldots,q_m$.  The scalar layer-cake identity gives
\begin{align*}
\sum_{i,j=1}^m q_iq_j\min\{t_i,t_j\}
&=\int_0^\infty\bigl(\sum_{i:t_i\ge u}q_i\bigr)^2du\\
&=\sum_{k=1}^m\int_{t_{k-1}}^{t_k}
\bigl(\sum_{\ell=k}^m q_\ell\bigr)^2du\\
&=\sum_{k=1}^m(t_k-t_{k-1})
\bigl(\sum_{\ell=k}^m q_\ell\bigr)^2.
\end{align*}
Apply this identity to the positive and negative blocks and add them to obtain \eqref{eq:oracle-sorting}.  Sorting dominates the computational cost, and the suffix sums and final summation are linear-time operations.

For the derivative, expand the Brownian quadratic form as
\begin{align}\label{eq:oracle-expanded}
F_q(a)
&=\frac12\sum_{i,j}q_iq_j
\bigl(|a_i|+|a_j|-|a_i-a_j|\bigr)\\
&=Q\sum_{i=1}^n q_i|a_i|
-\frac12\sum_{i,j=1}^nq_iq_j|a_i-a_j|,
\nonumber
\end{align}
where $Q=\sum_jq_j$.  Under the stated nonzero and no-tie assumptions, every absolute-value term is differentiable.  Differentiating the first term in \eqref{eq:oracle-expanded} with respect to $a_k$ gives
$Qq_k\operatorname{sign}(a_k)$.  In the double sum, the terms with first index $k$ contribute
\begin{align}
-\frac12\sum_jq_kq_j\operatorname{sign}(a_k-a_j),
\notag
\end{align}
while the terms with second index $k$ contribute the same quantity.  The $i=j=k$ term is identically zero and contributes zero.  Adding the two contributions proves \eqref{eq:oracle-gradient-a}.

\enlargethispage{2\baselineskip}
Finally, let
$M(\theta)=N^{-1/2}\mathbf B(a(\theta))N^{-1/2}$.  For a simple largest eigenvalue with normalized eigenvector $v$, differentiating
$M(\theta)v(\theta)=\lambda(\theta)v(\theta)$ and using $v^\top v=1$ gives the standard identity
\begin{align}
d\lambda=v^\top(dM)v.
\notag
\end{align}
At the point under consideration, put $q=N^{-1/2}v$.  Then
$d\lambda=q^\top(d\mathbf B(a))q=\ip{\nabla_aF_q(a)}{da}$.  Since $da=J_a(\theta)d\theta$, the chain rule yields \eqref{eq:oracle-chain-rule}.
\end{proof}

\section{Finite Covariance-Path Proofs}
\label{app:cove-finite-path}

\subsection{Proof of the exact spherical Brownian mixture}

\begin{proof}
Let $u\in\R^d$.  Rotational invariance of the uniform measure on
$\mathbb S^{d-1}$ implies that the distribution of
$\ip{\Theta}{u}$ depends on $u$ only through $\norm{u}_2$.  If $u\ne0$, choose
an orthogonal matrix $Q$ satisfying $Qu=\norm{u}_2e_1$.  Since $Q\Theta$ has the
same distribution as $\Theta$,
\begin{align}
 \E\abs{\ip{\Theta}{u}}
 &=
 \E\abs{\ip{Q\Theta}{Qu}}
 =
 \norm{u}_2\E\abs{\Theta_1}
 =
 c_d\norm{u}_2.
 \label{cove:eq:spherical-absolute-moment}
\end{align}
The same identity is immediate for $u=0$.

Apply \Cref{cove:eq:spherical-absolute-moment} to
$\Sigma^{1/2}z$, $\Sigma^{1/2}z'$, and
$\Sigma^{1/2}\bigl(z-z'\bigr)$.  By linearity of expectation,
\begin{align}
 &\frac{1}{c_d}
 \E_{\Theta}
 \kB\bigl(
 \ip{\Theta}{\Sigma^{1/2}z},
 \ip{\Theta}{\Sigma^{1/2}z'}
 \bigr)
 \\
 &\quad=
 \frac{1}{2c_d}
 \E_{\Theta}
 \bigl(
 \abs{\ip{\Theta}{\Sigma^{1/2}z}}
 +
 \abs{\ip{\Theta}{\Sigma^{1/2}z'}}
 -
 \abs{\ip{\Theta}{\Sigma^{1/2}\bigl(z-z'\bigr)}}
 \bigr)
 \\
 &\quad=
 \frac{1}{2}
 \bigl(
 \norm{\Sigma^{1/2}z}_2
 +
 \norm{\Sigma^{1/2}z'}_2
 -
 \norm{\Sigma^{1/2}\bigl(z-z'\bigr)}_2
 \bigr)
 \\
 &\quad=
 k_{\Sigma}(z,z').
\end{align}
Each integrand is positive semidefinite by the scalar layer-cake identity
underlying \Cref{eq:threshold-matrix}.  An expectation of positive semidefinite kernels
is positive semidefinite, which proves the final assertion.
\end{proof}

\subsection{Proof of exact two-basis path evaluation}

\begin{proof}
By \Cref{cove:eq:covariance-path},
\begin{align}
 u^{\top}\Sigma_su
 &=
 u^{\top}\bigl(\bigl(1-s\bigr)I+s\Sigma_{\mathrm F}\bigr)u
 \\
 &=
 \bigl(1-s\bigr)u^{\top}u
 +
 su^{\top}\Sigma_{\mathrm F}u.
\end{align}
This is the first identity in \Cref{cove:eq:path-affine-squares}.  Replacing $u$ by
$u-v$ gives the second.  The kernel in
\Cref{cove:eq:covariance-brownian-kernel} is obtained by taking the nonnegative
square roots of these three affine squared quantities and forming their stated
linear combination.  Therefore the identity and Fisher squared components determine
every path member exactly.
\end{proof}

\subsection{Proof of training-only selection}

\begin{proof}
For fixed split seeds, every fold index is a deterministic function of the
training labels and seeds.  Inside a fold, the scaler, class means,
\Cref{cove:eq:fisher-diagonal}, kernel matrices, and fitted classifier are functions
only of fold-training features and labels.  The validation correctness values
$q_{i,\vartheta}$ are functions only of the labelled training sample.  The finite
maximization of \Cref{cove:eq:crossfit-accuracy}, including deterministic tie
breaking, is therefore measurable with respect to the labelled training sample
and split seeds.  No operation in this construction reads a test label.
\end{proof}

\subsection{Proof of explicit identity repair}

\begin{proof}
For every $t\in\mathcal S$,
\begin{align}
 \lambda_{\min}\bigl(\widehat N-K_t\bigr)
 &=
 \lambda_{\min}\bigl(
 \widehat N_{\mathrm{raw}}-K_t+\rho I
 \bigr)
 \\
 &=
 \lambda_{\min}\bigl(
 \widehat N_{\mathrm{raw}}-K_t
 \bigr)+\rho
 \\
 &\ge
 \widehat\lambda_{\min}
 +
 \max\bigl\{0,\mu-\widehat\lambda_{\min}\bigr\}
 \\
 &\ge
 \mu.
\end{align}
Thus $\widehat N-K_t\succeq\mu I$ for every path member.
\end{proof}

\subsection{Scope of the finite certificate}
The active family consists only of the five path kernels on one anchor sample.  The numerical repair proves finite-matrix domination at the reported tolerance.  It neither proves optimality for the unrestricted empirical BKL covariance family in \Cref{sec:setup,sec:profile} nor supplies a global recursive separation oracle.

\footnotesize
\setlength{\bibsep}{0pt}
\setlength{\baselineskip}{12pt}

\end{document}